%% file: main.tex
\documentclass[english]{article}
\usepackage[T1]{fontenc}
\usepackage[latin9]{inputenc}
\usepackage{geometry}
\usepackage{babel}
\usepackage{mathtools}
\usepackage{multirow}
\usepackage{amsmath}
\usepackage{amssymb}
\usepackage[bookmarks=false,
 breaklinks=false,pdfborder={0 0 1},backref=false,colorlinks=false]
 {hyperref}
\hypersetup{
 colorlinks,linkcolor=red,anchorcolor=blue,citecolor=blue}

\makeatletter

\providecommand{\tabularnewline}{\\}

\usepackage{babel}
\usepackage{babel}
\usepackage{babel}

\usepackage{xcolor,colortbl}
\definecolor{Gray}{gray}{0.85}
\usepackage{tcolorbox}

\usepackage{mathtools}

\usepackage{cite}\usepackage{amsthm}\usepackage{dsfont}\usepackage{array}\usepackage{mathrsfs}\usepackage{comment}\onecolumn
\usepackage{natbib}
\usepackage{color}\usepackage{babel}

\newcommand{\mymid}{\,|\,} 

\allowdisplaybreaks

\usepackage{enumitem}
\setlist[itemize]{leftmargin=1.5em}
\setlist[enumerate]{leftmargin=1.5em}

\usepackage{babel}
\usepackage{algorithm}% http://ctan.org/pkg/algorithms
\usepackage{algorithmic}% http://ctan.org/pkg/algorithmicx
\usepackage{arydshln}

\DeclareMathOperator{\ind}{\mathds{1}}  % Indicator

\numberwithin{equation}{section}

\definecolor{yxc}{RGB}{255,0,0}
\definecolor{yjc}{RGB}{125,0,0}
\definecolor{cm}{RGB}{0,0,200}
\definecolor{yly}{RGB}{0,150,0}

\makeatother

\begin{document}
\theoremstyle{plain} \newtheorem{lemma}{\textbf{Lemma}} \newtheorem{prop}{\textbf{Proposition}}\newtheorem{theorem}{\textbf{Theorem}}\setcounter{theorem}{0}
\newtheorem{corollary}{\textbf{Corollary}} \newtheorem{assumption}{\textbf{Assumption}}
\newtheorem{example}{\textbf{Example}} \newtheorem{definition}{\textbf{Definition}}
\newtheorem{fact}{\textbf{Fact}} \newtheorem{condition}{\textbf{Condition}}\theoremstyle{definition}

\theoremstyle{remark}\newtheorem{remark}{\textbf{Remark}}\newtheorem{claim}{\textbf{Claim}}\newtheorem{conjecture}{\textbf{Conjecture}}
\title{$O(d/T)$ Convergence Theory for Diffusion Probabilistic Models under
Minimal Assumptions}
\author{Gen Li\thanks{Department of Statistics, The Chinese University of Hong Kong, Hong
Kong; Email: \texttt{genli@cuhk.edu.hk}.}\and Yuling Yan\thanks{Department of Statistics, University of Wisconsin-Madison, WI 53706,
USA; Email: \texttt{yuling.yan@wisc.edu}.}}

\maketitle
\input{abstract.tex}


\setcounter{tocdepth}{2}

\tableofcontents{}

\input{intro.tex}

\input{main_results.tex}

\input{analysis.tex}\input{discussion.tex}

\section*{Acknowledgements}

Gen Li is supported in part by the Chinese University of Hong Kong
Direct Grant for Research.

\appendix
\input{appendix_sde.tex}\input{appendix_low_d.tex}\input{appendix_auxiliary.tex}\bibliographystyle{apalike}
\bibliography{reference-diffusion}

\end{document}

%% file: abstract.tex
\begin{abstract}
%\vspace{1ex}
Score-based diffusion models, which generate new data by learning to reverse a diffusion process that perturbs data from the target distribution into noise, have achieved remarkable success across various generative tasks. Despite their superior empirical performance, existing theoretical guarantees are often constrained by stringent assumptions or suboptimal convergence rates. In this paper, we establish a fast convergence theory for the denoising diffusion probabilistic model (DDPM), a widely used SDE-based sampler, under minimal assumptions. Our analysis shows that, provided $\ell_{2}$-accurate estimates of the score functions, the total variation distance between the target and generated distributions is upper bounded by $O(d/T)$ (ignoring logarithmic factors), where $d$ is the data dimensionality and $T$ is the number of steps. This result holds for any target distribution with finite first-order moment. Moreover, we show that with careful coefficient design, the convergence rate improves to $O(k/T)$, where $k$ is the intrinsic dimension of the target data distribution. This highlights the ability of DDPM to automatically adapt to unknown low-dimensional structures, a common feature of natural image distributions. These results are achieved through a novel set of analytical tools that provides a fine-grained characterization of how the error propagates at each step of the reverse process. 
\end{abstract}

%% file: intro.tex
\section{Introduction}

Score-based generative models (SGMs) have emerged as a powerful class
of generative frameworks, capable of learning and sampling from complex
data distributions \citep{sohl2015deep,ho2020denoising,song2020score,song2019generative,dhariwal2021diffusion}.
These models, including Denoising Diffusion Probabilistic Models (DDPM)
\citep{ho2020denoising} and Denoising Diffusion Implicit Models (DDIM)
\citep{song2020denoising}, operate by gradually transforming a simple
noise-like distribution (e.g., standard Gaussian) into a target data
distribution through a series of diffusion steps. This transformation
is achieved by learning a sequence of denoising processes governed
by score functions, which estimate the gradient of the log-density
of the data at each step. SGMs have demonstrated remarkable success
in various generative tasks, including image generation \citep{rombach2022high,ramesh2022hierarchical,saharia2022photorealistic},
audio generation \citep{kong2021diffwave}, video generation \citep{villegas2022phenaki},
and molecular design \citep{hoogeboom2022equivariant}. See e.g., \citet{yang2022diffusion,croitoru2023diffusion} for overviews of recent
development.

At the core of SGMs are two stochastic processes: a forward process,
which progressively adds noise to the data, 
\[
X_{0}\rightarrow X_{1}\rightarrow\cdots\rightarrow X_{T},
\]
where $X_{0}$ is drawn from the target data distribution $p_{\mathsf{data}}$
and is gradually transformed into $X_{T}$ that resembles standard
Gaussian noise; and a reverse process, 
\[
Y_{T}\rightarrow Y_{T-1}\rightarrow\cdots\rightarrow Y_{0},
\]
which starts from pure Gaussian noise $Y_{T}$ and sequentially converts
it into $Y_{0}$ that closely mimics the target data distribution
$p_{\mathsf{data}}$. At each step, the distributions of $Y_{t}$
and $X_{t}$ are kept close. The key challenge lies in constructing
this reverse process effectively to ensure accurate sampling from
the target distribution.

Motivated by classical results on the time-reversal of stochastic
differential equations (SDEs) \citep{anderson1982reverse,haussmann1986time},
SGMs construct the reverse process using the gradients of the log
marginal density of the forward process, known as score functions.
At each step, $Y_{t-1}$ is generated from $Y_{t}$ with the help
of the score function $\nabla \log p_{X_{t}}(\cdot)$, where $p_{X_{t}}$
denotes the density of $X_{t}$. Both the DDPM sampler \citep{ho2020denoising} and the DDIM sampler \citep{song2020denoising}
follow this denoising framework, with the key distinction being whether
additional random noise is injected when generating each $Y_{t-1}$.
Although the score functions are not known explicitly, they are pre-trained
using large neural networks through score-matching techniques \citep{hyvarinen2005estimation,hyvarinen2007some,vincent2011connection,song2019generative}.

Despite their impressive empirical success, the theoretical foundations of diffusion models remain relatively underdeveloped. Early studies on the convergence of SGMs \citep{de2021diffusion,de2022convergence,liu2022let,pidstrigach2022score,block2020generative} did not provide polynomial convergence guarantees. In recent years, a line of works have explored the convergence of the generated
distribution to the target distribution, treating the score-matching step as a black box and focusing on the effects of the number of steps
$T$ and the score estimation error on the convergence of the sampling
phase \citep{chen2022sampling,chen2022improved,chen2023probability,benton2023linear,lee2022convergence,lee2023convergence,li2023towards,li2024sharp,li2024adapting,gao2024convergence,huang2024convergence,tang2024score,liang2024non,chen2023restoration}. Recent studies have investigated the performance guarantees of SGMs in the presence of low-dimensional structures (e.g., \citet{li2024adapting,azangulov2024convergence,huang2024denoising,potaptchik2024linear,wang2024diffusion,chen2023score,tang2024adaptivity,tang2024conditional,chen2024exploring}) and the acceleration of SGMs (e.g., \citet{li2024accelerating,liang2024non,gupta2024faster,li2024provable,li2024improved}).
Following this general avenue, the goal of this paper is to establish
a sharp convergence theory for diffusion models with minimal assumptions.

\paragraph*{Prior convergence guarantees.}

In recent years, a flurry of work has emerged on the convergence guarantees
for the DDPM and DDIM type samplers. However, these prior studies
fall short of providing a fully satisfactory convergence theory due
to at least one of the following three obstacles:
\begin{itemize}
\item \textit{Stringent data assumptions.} Earlier works, such as \citet{lee2022convergence},
required the target data distribution to satisfy the log-Sobolev inequality.
Similarly, \citet{chen2022sampling,lee2023convergence,chen2023probability,chen2023restoration}
assumed that the score functions along the forward process must satisfy
a Lipschitz smoothness condition. More recent work \citet{gao2024convergence}
relied on the strong log-concavity assumption of the target distribution
to establish convergence guarantees in Wasserstein distance. These
assumptions are often impractical to verify and may not hold for complex
distributions commonly seen in image data. Some newer studies on the DDPM
sampler (e.g., \citet{chen2022improved,benton2023linear}) and the DDIM
sampler (e.g., \citet{li2024sharp}) have relaxed these stringent
assumptions, and their results applied to more general target distributions
with bounded second-order moments or sufficiently large support.
\item \textit{Slow convergence rate.} We follow most existing works and
focus on the total variation (TV) distance between the target and
the generated distributions.\footnote{Convergence rates in Kullback-Leibler (KL) divergence in \citet{chen2022improved,benton2023linear} are transferred to TV distance using Pinsker's inequality, because the KL divergence is not a distance.}
Let $T$ be the number of steps and $d$ be the dimensionality of the
data. For the DDPM sampler, \citet{chen2022sampling} established
a convergence rate of $O(L\sqrt{(d+M_{2})/T})$, where $L$ is the
Lipschitz constant of the score functions along the forward process,
and $M_{2}$ is the second-order moment of the target distribution.
Later, \citet{chen2022improved} lifted the Lipschitz condition, but
this came at the cost of a rate $O(d/\sqrt{T})$ with worse dimension
dependence. The state-of-the-art result for the DDPM samplers is
due to \citet{benton2023linear}, achieving a convergence rate of $O\sqrt{d/T})$.
However, this is still slower than the convergence rate for the DDIM
sampler, achieved in \citet{li2024sharp}, which attains $O(d/T)$
in the regime $T\gg d^{2}$.
\item \textit{Additional score estimation requirements.} Convergence theory
for diffusion models must also account for the impact of imperfect
score estimation on performance. While recent results for the DDPM
sampler \citep{chen2022sampling,chen2022improved,benton2023linear}
require only $\ell_{2}$-accurate score function estimates, another
line of work on the DDIM sampler \citep{li2023towards,li2024sharp,huang2024convergence}
achieves faster convergence rates, albeit under stricter requirements
for score estimates. Specifically, \citet{li2023towards,li2024sharp}
require not only that the score estimates be close to the true score
functions, but also that the Jacobian of the score estimates be close
to the Jacobian of the true score functions, which is a significantly
stronger condition. Additionally, \citet{huang2024convergence} assumes
higher-order smoothness in the score estimates.
\end{itemize}
From this discussion, it is evident that while the state-of-the-art
convergence rates for the DDIM sampler surpass those for the DDPM
sampler, they also rely on more restrictive assumptions. This motivates
us to think whether it is possible to achieve the best of both worlds,
namely, {
\setlist{rightmargin=\leftmargin}
\begin{itemize}
	\item[]	{\em Can we establish a convergence theory for diffusion models that achieves a fast convergence rate under minimal data and score estimation assumptions? }
\end{itemize}
}As noted in \citet{li2024sharp}, a counterexample demonstrates that
$\ell_{2}$-accurate score estimation alone is insufficient for convergence
of the DDIM sampler under TV distance. The current paper answers
this question affirmatively by focusing on the DDPM sampler.

\begin{table}
	\centering\label{table:comparison}%
	\begin{tabular}{|c|c|c|c|}
		\hline 
		\multirow{2}{*}{Sampler} & Convergence rate & Data assumption & Requirements on score\tabularnewline
		& (in TV distance) & {\scriptsize ($X_{0}\sim p_{\mathsf{data}}$, $s_{t}^{\star}=\nabla \log p_{X_{t}}$)} & estimates $s_{t}$ \tabularnewline
		\hline 
		DDPM & \multirow{2}{*}{$L\sqrt{d/T}$} & $L$-Lipschitz $s_{t}^{\star}$; & \multirow{2}{*}{$s_{t}\approx s_{t}^{\star}$ in $L^{2}(p_{X_{t}})$}\tabularnewline
		\citep{chen2022sampling} &  & $\mathbb{E}[\Vert X_{0}\Vert_{2}^{2}]<\infty$ & \tabularnewline
		\hline 
		DDPM & \multirow{2}{*}{$\sqrt{d^{2}/T}$} & \multirow{2}{*}{$\mathbb{E}[\Vert X_{0}\Vert_{2}^{2}]<\infty$} & \multirow{2}{*}{$s_{t}\approx s_{t}^{\star}$ in $L^{2}(p_{X_{t}})$}\tabularnewline
		\citep{chen2022improved} &  &  & \tabularnewline
		\hline 
		DDPM & \multirow{2}{*}{$\sqrt{d/T}$} & \multirow{2}{*}{$\mathbb{E}[\Vert X_{0}\Vert_{2}^{2}]<\infty$} & \multirow{2}{*}{$s_{t}\approx s_{t}^{\star}$ in $L^{2}(p_{X_{t}})$}\tabularnewline
		\citep{benton2023linear} &  &  & \tabularnewline
		\hline 
		DDIM & \multirow{2}{*}{$L^{2}\sqrt{d}/T$} & $L$-Lipschitz $s_{t}^{\star}$; & $L$-Lipschitz $s_{t}$;\tabularnewline
		\citep{chen2023probability} &  & \multirow{1}{*}{$\mathbb{E}[\Vert X_{0}\Vert_{2}^{2}]<\infty$} & $s_{t}\approx s_{t}^{\star}$ in $L^{2}(p_{X_{t}})$\tabularnewline
		\hline 
		DDIM & \multirow{2}{*}{$d^{2}/T+d^{6}/T^{2}$} & \multirow{2}{*}{bounded support} & $s_{t}\approx s_{t}^{\star}$ in $L^{2}(p_{X_{t}})$;\tabularnewline
		\citep{li2023towards} &  &  & $J_{s_{t}}\approx J_{s_{t}^{\star}}$ in $L^{2}(p_{X_{t}})$\tabularnewline
		\hline 
%		ODE-based & \multirow{2}{*}{$d^{2}/T$} & \multirow{2}{*}{bounded support} & $C^2$-smooth $s_t$;\tabularnewline
%		\citep{huang2024convergence} &  &  & $s_{t}\approx s_{t}^{\star}$ in $L^{2}(p_{X_{t}})$ \tabularnewline
%		\hline 
		\multirow{1}{*}{DDIM} & \multirow{2}{*}{$d/T$ when $T\gtrsim d^2$} & \multirow{2}{*}{bounded support} & $s_{t}\approx s_{t}^{\star}$ in $L^{2}(p_{X_{t}})$;\tabularnewline
		\citep{li2024sharp} &  &  & $J_{s_{t}}\approx J_{s_{t}^{\star}}$ in $L^{2}(p_{X_{t}})$\tabularnewline
		\hline 
		{\cellcolor{Gray}DDPM} & {\cellcolor{Gray}} & {\cellcolor{Gray}}  & {\cellcolor{Gray}}\tabularnewline
		{\cellcolor{Gray}\textbf{(this paper)}} & {\cellcolor{Gray}\multirow{-2}{*}{$d/T$}} & {\cellcolor{Gray}\multirow{-2}{*}{$\mathbb{E}[\Vert X_{0}\Vert_{2}]<\infty$}} & {\cellcolor{Gray}\multirow{-2}{*}{$s_{t}\approx s_{t}^{\star}$ in $L^{2}(p_{X_{t}})$}} \tabularnewline
		\hline 
	\end{tabular}\caption{Comparison with prior convergence guarantees for diffusion
		models (ignoring log factors). Convergence rates in KL divergence are transferred to TV distance using Pinsker's inequality. Here $J_{f}: \mathbb{R}^d \to \mathbb{R}^{d\times d}$ denotes the Jacobian matrix of a function $f:\mathbb{R}^d\to\mathbb{R}^d$.}
\end{table}

\paragraph*{Our contributions.}

This paper develops a fast convergence theory for the DDPM sampler
under minimal assumptions. We show that the TV distance between the
generated and target distributions is bounded by:
\[
\frac{d}{T}+\sqrt{\frac{1}{T}\sum_{t=1}^{T}\mathbb{E}\big[\big\Vert s_{t}(X_{t})-s_{t}^{\star}(X_{t})\big\Vert_{2}^{2}\big]},
\]
up to logarithmic factors. The first term reflects the discretization
error, while the second term accounts for score estimation error.
Compared to the two most relevant works \citep{benton2023linear,li2024sharp}
, which provide state-of-the-art results for the DDPM and DDIM samplers,
our main contributions are as follows:
\begin{itemize}
\item \textit{$O(d/T)$ convergence rate.} Under perfect score function
estimation, we establish an $O(d/T)$ convergence rate for the DDPM
sampler in TV distance, improving on the previous best rate of $O(\sqrt{d/T})$
from \citet{benton2023linear}. Our result also matches the convergence
rate of the DDIM sampler achieved in \citet{li2024sharp}, but is
more general, as their result only holds when $T\gg d^{2}$, while
ours applies for arbitrary $T$ and $d$.
\item \textit{Minimal assumptions.} Our theory requires only that the target
distribution has finite first-order moment, which, to the best of
our knowledge, is the weakest data assumption in the current literature. Additionally, we require only $\ell_{2}$-accurate score
estimates, which is a significantly weaker condition than the Jacobian
accuracy required by \citet{li2023towards,li2024sharp}.
\item \textit{Adaptivity to unknown low-dimensional structures.}  Extending our theory, we demonstrate that with carefully designed coefficients \citep{li2024adapting}, the DDPM sampler achieves a convergence rate of $O(k/T)$ in TV distance, where $k$ is the intrinsic dimension of the target data distribution. This improves upon the previous best bound of $O(\sqrt{k/T})$ established in \citet{potaptchik2024linear,huang2024denoising}.
\end{itemize}
In summary, our results achieve the fastest convergence rate in the
literature for both DDPM and DDIM samplers while requiring minimal
assumptions. A comparative summary with prior work is presented in
Table~\ref{table:comparison}.

%% file: main_results.tex
\section{Problem set-up}

In this section, we provide an overview of the diffusion model and
the DDPM sampler.

\paragraph{Forward process.}

We consider a Markov process in $\mathbb{R}^{d}$ starting from $X_{0}\sim p_{\mathsf{data}}$,
evolving according to the recursion: 
\begin{equation}
X_{t}=\sqrt{1-\beta_{t}}X_{t-1}+\sqrt{\beta_{t}}W_{t}\quad(t=1,\ldots,T),\label{eq:forward-update}
\end{equation}
where $W_{1},\ldots,W_{T}$ are independent draws from $\mathcal{N}(0,I_{d})$,
and $\beta_{1},\ldots,\beta_{t}\in(0,1)$ are the learning rates.
For each $1\leq t\leq T$, define $\alpha_{t}\coloneqq1-\beta_{t}$
and $\overline{\alpha}_{t}\coloneqq\prod_{i=1}^{t}\alpha_{i}$. This
allows us to express $X_{t}$ in closed form as: 
\begin{equation}
X_{t}=\sqrt{\overline{\alpha}_{t}}X_{0}+\sqrt{1-\overline{\alpha}_{t}}\,\overline{W}_{t}\quad\text{where}\quad\overline{W}_{t}\sim\mathcal{N}(0,I_{d}).\label{eq:forward-formula}
\end{equation}
We select the learning rates such that (i) $\beta_{t}$ is small for
every $1\leq t\leq T$; and (ii) $\overline{\alpha}_{T}$ is vanishingly
small, ensuring that the distribution of $X_{T}$ is exceedingly close
to $\mathcal{N}(0,I_{d}$). In this paper, we adopt the following
learning rate schedule 
\begin{equation}
\beta_{1}=\frac{1}{T^{c_{0}}},\qquad\beta_{t+1}=\frac{c_{1}\log T}{T}\min\left\{ \beta_{1}\Big(1+\frac{c_{1}\log T}{T}\Big)^{t},1\right\} \quad(t=1,\ldots,T-1),\label{eq:learning-rate}
\end{equation}
for sufficiently large constants $c_{0},c_{1}>0$. This schedule is
commonly used in the diffusion model literature (see, e.g., \citet{li2023towards,li2024sharp}),
although the results in this paper hold for any learning rate schedule
satisfying the conditions in Lemma~\ref{lemma:step-size}.

\paragraph{Reverse process.}

The crucial elements in constructing the reverse process are the score
functions associated with the marginal distributions of the forward
diffusion process (\ref{eq:forward-update}). For each $t=1,\ldots,T$,
we define the score function as: 
\[
s_{t}^{\star}(x)\coloneqq\nabla\log p_{X_{t}}(x)\quad(t=1,\ldots,T),
\]
where $p_{X_{t}}(\cdot)$ represents the smooth probability density
of $X_{t}$. Since the true score functions are typically unknown,
we assume access to estimates $s_{t}(\cdot)$ for each $s_{t}^{\star}(\cdot)$.
To quantify the error in these estimates, we define the averaged $\ell_{2}$
score estimation error as: 
\[
\varepsilon_{\mathsf{score}}^{2}\coloneqq\frac{1}{T}\sum_{t=1}^{T}\mathbb{E}\big[\Vert s_{t}(X_{t})-s_{t}^{\star}(X_{t})\Vert_{2}^{2}\big].
\]
This error term quantifies the effect of imperfect score approximation
in our theoretical analysis. Using these score estimates, we can construct
the reverse process, which starts from $Y_{T}\sim\mathcal{N}(0,I_{d})$
and proceeds as 
\begin{equation}
Y_{t-1}=\frac{1}{\sqrt{\alpha_{t}}}\big(Y_{t}+\eta_{t}s_{t}(Y_{t})+\sigma_{t}Z_{t}\big)\quad(t=T,\ldots,1),\label{eq:DDPM}
\end{equation}
where $Z_{1},\ldots,Z_{T}$ are independent draws from $\mathcal{N}(0,I_{d})$.
This is the popular DDPM sampler \citep{ho2020denoising}. 

\paragraph*{Notation.}

The total variation (TV) distance between two probability measures
$P$ and $Q$ on a probability space $(\Omega,\mathcal{F})$ is define
as 
\[
\mathsf{TV}(P,Q)\coloneqq\sup_{A\in\mathcal{F}}\vert P(A)-Q(A)\vert=\frac{1}{2}\int_{\Omega}\vert p(x)-q(x)\vert\mathrm{d}x,
\]
where the last relation holds if $P$ and $Q$ have probability density
functions $p(x)$ and $q(x)$. Let $\mathsf{KL}(P\,\Vert\,Q)$ denote
the Kullback-Leibler (KL) divergence of $P$ from $Q$, then Pinsker's
inequality states that 
\[
\mathsf{TV}(P,Q)\leq\sqrt{\frac{1}{2}\mathsf{KL}(P\,\Vert\,Q)}.
\]
For any matrix $A$, we use $\Vert A\Vert$ and $\Vert A\Vert_{\mathrm{F}}$
to denote its spectral norm and Frobenius norm. Let $\mathcal{X}\subseteq\mathbb{R}^{d}$
be the support set of $p_{\mathsf{data}}$, i.e., the smallest closed
set $C\subseteq\mathbb{R}^{d}$ such that $p_{\mathsf{data}}(C)=1$. 

\section{Main results}

\subsection{General theory: an $O(d/T)$ convergence bound}

In this section, we will establish a fast convergence theory for the
DDPM sampler under minimal assumptions. Before proceeding, we introduce
the only data assumption that our theory requires.

\begin{assumption}\label{assumption:moment}The target distribution
$p_{\mathsf{data}}$ has finite first-order moment. Furthermore, we
assume that there exists some constant $c_{M}>0$ such that 
\[
M_{1}\coloneqq\mathbb{E}[\Vert X_{0}\Vert_{2}]\leq T^{c_{M}}.
\]
\end{assumption}

Here we require the first-order moment $M_{1}$ to be at most polynomially
large in $T$, which allows cleaner and more concise result that avoids
unnecessary technical complicacy. Since $c_{M}>0$ can be arbitrarily
large, we allow the target data distribution to have exceedingly large
first-order moment, which is a mild assumption. In comparison, Assumption~\ref{assumption:moment}
is weaker than the finite second-order moment condition in e.g., \citet{chen2022sampling,chen2022improved,benton2023linear}
and bounded support condition in e.g., \citet{li2023towards,li2024sharp}.

Now we are positioned to present our convergence theory for the DDPM
sampler.

\begin{theorem}\label{thm:main} Suppose that Assumption~\ref{assumption:moment}
holds, and take the coefficients of the DDPM sampler (\ref{eq:DDPM})
to be $\eta_{t}=\sigma_{t}^{2}=1-\alpha_{t}$. Then there exists some
universal constant $c>0$ such that 
\begin{equation}
\mathsf{TV}(p_{X_{1}},p_{Y_{1}})\leq c\frac{d\log^{3}T}{T}+c\varepsilon_{\mathsf{score}}\sqrt{\log T}.\label{eq:error-bound}
\end{equation}
\end{theorem}

The two terms in the error bound (\ref{eq:error-bound}) correspond
to discretization error and score matching error, respectively. A
few remarks are in order. 
\begin{itemize}
\item \textit{Sharp convergence guarantees.} Consider the setting with perfect
score estimation (i.e., $\varepsilon_{\mathsf{score}}=0$) and ignore
any log factor. Theorem~\ref{thm:main} reveals that the DDPM
sampler converges at the order of $O(d/T)$ in total variation distance.
This significantly improves the state-of-the-art convergence rate
$O(\sqrt{d/T})$ for the DDPM sampler \citep{benton2023linear}.
Turning to the DDIM sampler 
\begin{equation}
Y_{t-1}=\frac{1}{\sqrt{\alpha_{t}}}\big(Y_{t}+\frac{1-\alpha_{t}}{2}s_{t}(Y_{t})\big)\quad(t=T,\ldots,1),\qquad Y_{T}\sim\mathcal{N}(0,I_{d}),\label{eq:DDIM}
\end{equation}
\citet{li2024sharp} achieved the same convergence rate $O(d/T)$
in the regime $T\gg d^{2}$. Our result holds for general $T$ and
$d$, including the regime $T\asymp d$, hence is more general. 
\item \textit{Stability vis-à-vis imperfect score estimation.} The score
estimation error in (\ref{eq:error-bound}) is linear in $\varepsilon_{\mathsf{score}}$,
which suggests that the performance of the DDPM sampler degrades
gracefully when the score estimates become less accurate. In other
words, our theory holds with $\ell_{2}$-accurate score estimates,
which aligns with recent line of work on the DDPM sampler \citep{chen2022sampling,chen2022improved,benton2023linear}.
In comparison, the convergence bound in \citet{li2024sharp} for the
DDIM sampler reads 
\[
\mathsf{TV}(p_{X_{1}},p_{Y_{1}})\lesssim\frac{d}{T}+\sqrt{d}\varepsilon_{\mathsf{score}}+d\varepsilon_{\mathsf{Jacobi}}\quad\text{where}\quad\varepsilon_{\mathsf{Jacobi}}\coloneqq\frac{1}{T}\sum_{t=1}^{T}\mathbb{E}\Big[\Big\Vert\frac{\partial s_{t}^{\star}}{\partial x}(X_{t})-\frac{\partial s_{t}}{\partial x}(X_{t})\Big\Vert\Big],
\]
which exhibits worse stability against imperfect score estimation.
First, the dependency of their bound on $\varepsilon_{\mathsf{score}}$
is $\sqrt{d}$ times worse than ours. Second, their error bound involves
an additional term that is proportional to $\varepsilon_{\mathsf{Jacobi}}$,
which means that their theory requires the Jacobian of $s_{t}$ need
to be accurate in estimating the Jacobian of $s_{t}^{\star}$, which
is a stringent requirement. 
\item \textit{Minimal data assumption}. The only assumption required in
Theorem~\ref{thm:main} is that $M_{1}$ is at most polynomially
large in $T$. In fact, by slightly modifying the proof, we can further
relax Assumption~\ref{assumption:moment} to accommodate target data
distributions with polynomially large $\delta$-th order moment 
\[
M_{\delta}\coloneqq\big(\mathbb{E}[\Vert X_{0}\Vert_{2}^{\delta}]\big)^{1/\delta}\leq T^{c_{M}},
\]
for any constant $\delta>0$. The same error bound (\ref{eq:error-bound})
holds, provided that $T\gg\max\{1,\delta^{-1}\}d\log^{2}T$. 
\item \textit{Error metric.} Theorem~\ref{thm:main} provides convergence
guarantees to $p_{X_{1}}$ instead of the target data distribution
(i.e., the distribution of $X_{0}$), which is similar to the results
in e.g., \citet{chen2022improved,benton2023linear,li2023towards,li2024sharp}.
On one hand, since $X_{1}=\sqrt{1-\beta_{1}}X_{0}+\sqrt{\beta_{1}}$
and $\beta_{1}=T^{-c_{0}}$ is vanishingly small, the distributions
of $X_{1}$ and $X_{0}$ are exceedingly close. Hence $\mathsf{TV}(p_{X_{1}},p_{Y_{1}})$
is a valid error metric. On the other hand, the smoothness of $p_{X_{1}}$
allows us to circumvent imposing any Lipschitz assumption on the score
functions, which provides technical benefit for the analysis. 
\end{itemize}
It is worth noting that most previous studies on the convergence of
the DDPM sampler (e.g., \citet{chen2022sampling,chen2022improved,benton2023linear,li2023towards})
typically begin by upper bounding the squared TV error using the KL
divergence of the forward process from the reverse process. This is
done through the following argument: 
\begin{align}
\mathsf{TV}^{2}(p_{X_{1}},p_{Y_{1}}) & \leq\frac{1}{2}\mathsf{KL}\left(p_{X_{1}}\Vert p_{Y_{1}}\right)\leq\frac{1}{2}\mathsf{KL}\left(p_{X_{1},\ldots,X_{T}}\Vert p_{Y_{1},\ldots,Y_{T}}\right),\label{eq:error-decomposition}
\end{align}
where the first inequality follows from Pinsker's inequality and the
second from the data-processing inequality. The KL divergence on the
right-hand side of this bound is more tractable and can be further
bounded, for example, using Girsanov's theorem. In fact, \cite[Theorem 7]{chen2022sampling}
provides theoretical evidence that the KL divergence between the forward
and reverse processes is lower bound by $\Omega(d/T)$, even when
the target distribution is as simple as a standard Gaussian and perfect
score estimates are available. This suggests that such an approach
cannot yield error bounds better than $O(\sqrt{d/T})$ in general.

To achieve a sharper convergence rate, we take a different approach
by directly analyzing the total variation error without resorting
to intermediate KL divergence bounds. Specifically, we establish a
fine-grained recursive relation that tracks how the error $\mathsf{TV}(p_{X_{t}},p_{Y_{t}})$
propagates through the reverse process as $t$ decreases from $T$
to $1$. See Section~\ref{sec:proof-main} for more details.

\subsection{Adapting to unknown low-dimensional structure}

In this section, we will provide convergence guarantees for the DDPM
sampler when the target data distribution $p_{\mathsf{data}}$ exhibits
low-dimensional structures. This scenario is particularly important, as natural image data are often concentrated on or near low-dimensional manifolds  \citep{pope2021intrinsic,simoncelli2001natural}. To formalize this, we define the intrinsic dimension of $\mathcal{X}=\mathsf{supp}(p_{\mathsf{data}})$ as follows.

\begin{definition}[Intrinsic dimension] \label{defn:intrinsic}Fix
$\varepsilon=T^{-c_{\varepsilon}}$, where $c_{\varepsilon}>0$ is
some sufficiently large constant. We define the intrinsic dimension
of $\mathcal{X}$ to be some quantity $k>0$ such that 
\[
\log N_{\varepsilon}(\mathcal{X})\leq C_{\mathsf{cover}}k\log T
\]
for some universal constant $C_{\mathsf{cover}}>0$. \end{definition}

This definition relates the intrinsic dimension $k$ to the metric entropy of  $\mathcal{X}$ (see e.g., \citet{wainwright2019high})
and is standard in existing literature (e.g., \citet{li2024adapting,huang2024denoising}).
Importantly, it accommodates approximate low-dimensional structures by requiring only that $\mathcal{X}$ is concentrated near low-dimensional manifolds, which is more general than assuming exact low-dimensionality.

Recent work by \citet{li2024adapting} demonstrated that the following coefficient design is essential for achieving nearly dimension-free convergence bounds for the DDPM sampler:
\begin{equation}
\eta_{t}^{\star}=1-\alpha_{t}\qquad\text{and}\qquad\sigma_{t}^{\star2}=\frac{(1-\alpha_{t})(\alpha_{t}-\overline{\alpha}_{t})}{1-\overline{\alpha}_{t}}.\label{eq:coef-design}
\end{equation}
Specifically, Theorem 1 in \citet{li2024adapting} established that under this coefficient design, the DDPM sampler converges at a rate of $O(k^{2}/\sqrt{T})$ in TV distance. Furthermore, Theorem 2 provided evidence that \eqref{eq:coef-design} is the unique coefficient design enabling nearly (ambient) dimension-free convergence. 

Building on the techniques developed in the proof of Theorem~\ref{thm:main}, we strengthen this result by proving a faster $O(k/T)$ convergence bound under the same coefficient design.

\begin{theorem}\label{thm:main-low-d} Suppose that Assumption~\ref{assumption:moment}
holds. Take the coefficients of the DDPM sampler (\ref{eq:DDPM}) to be $\eta_{t}=\eta_{t}^{\star}$ and $\sigma_{t}^{2}=\sigma_{t}^{\star2}$
(cf.~(\ref{eq:coef-design})). Then there exists some universal constant
$c>0$ such that 
\begin{equation}
\mathsf{TV}(p_{X_{1}},p_{Y_{1}})\leq c\frac{k\log^{3}T}{T}+c\varepsilon_{\mathsf{score}}\sqrt{\log T},\label{eq:error-bound-low-d}
\end{equation}
where $k$ is the intrinsic dimension of $\mathcal{X}$ (see Definition~\ref{defn:intrinsic}).
\end{theorem}

Consider the setup with perfect score estimation (i.e., $\varepsilon_{\mathsf{score}}=0$)
and disregard log factors. Theorem~\ref{thm:main-low-d} demonstrates that, under the coefficient design in (\ref{eq:coef-design}), the convergence rate of the DDPM sampler is $O(k/T)$, extending Theorem~\ref{thm:main} to target data distributions with low-dimensional structure. While the importance of this coefficient design for achieving ambient dimension-free convergence is not new (see \citet{li2024adapting}), our result significantly improves upon prior rates, which are of order $O(\sqrt{\mathsf{poly}(k)/T})$ \citep{li2024adapting,azangulov2024convergence,huang2024denoising,potaptchik2024linear}. For a detailed comparison, refer to Table~\ref{table:comparison-low-d}.

\begin{table}
\centering\label{table:comparison-low-d}%

\begin{tabular}{|c|c|c|c|}
	\hline 
	\multirow{2}{*}{Sampler} & Convergence rate & Data assumption & Intrinsic dimension $k$ \tabularnewline
	& (in TV distance) & ($X_{0}\sim p_{\mathsf{data}}$) & of $\mathcal{X}=\mathsf{supp}(p_\mathsf{data})$ \tabularnewline
	\hline 
	DDPM & \multirow{2}{*}{$\sqrt{k^{4}/T}$} & \multirow{2}{*}{bounded support} & metric entropy \tabularnewline
	\citep{li2024adapting} &  &  & (Definition~\ref{defn:intrinsic}) \tabularnewline
	\hline 
	DDPM & \multirow{2}{*}{$\sqrt{k^{3}/T}$} & bounded support; & \multirow{2}{*}{manifold dimension}\tabularnewline
	\citep{azangulov2024convergence} & & smooth density $p_{\mathsf{data}} \asymp 1$ & \tabularnewline
	\hline 
	DDPM & \multirow{2}{*}{$\sqrt{k/T}$} & bounded support; & \multirow{2}{*}{manifold dimension}\tabularnewline
	\citep{potaptchik2024linear} &  & smooth density $p_{\mathsf{data}} \asymp 1$ & \tabularnewline
	\hline 
	DDPM & \multirow{2}{*}{$\sqrt{k/T}$} & \multirow{2}{*}{bounded support} & metric entropy \tabularnewline
	\citep{huang2024denoising} &  & & (Definition~\ref{defn:intrinsic})\tabularnewline
	\hline 
	{\cellcolor{Gray}DDPM} & {\cellcolor{Gray}} & {\cellcolor{Gray}}  & {\cellcolor{Gray} metric entropy}\tabularnewline
	{\cellcolor{Gray}\textbf{(this paper)}} & {\cellcolor{Gray}\multirow{-2}{*}{$k/T$}} & {\cellcolor{Gray}\multirow{-2}{*}{$\mathbb{E}[\Vert X_{0}\Vert_{2}]<\infty$}} & {\cellcolor{Gray} (Definition~\ref{defn:intrinsic})} \tabularnewline
	\hline 
\end{tabular}

\caption{Comparison with prior convergence rates (ignoring log factors) for the DDPM sampler that adapts to intrinsic low-dimensional structures. Convergence rates in KL divergence are transferred to TV distance using Pinsker's inequality.}
\end{table}

%% file: analysis.tex
\section{Proof of Theorem~\ref{thm:main} \protect\protect\label{sec:proof-main}}

\subsection{Preliminaries}

For each $1\leq t\leq T$ and any $x\in\mathbb{R}^{d}$, it is known
that the score function $s_{t}^{\star}(x)$ associated with $p_{X_{t}}$
admits the following expression 
\[
s_{t}^{\star}(x)=-\frac{1}{1-\overline{\alpha}_{t}}\int p_{X_{0}|X_{t}}(x_{0}\mymid x)\big(x-\sqrt{\overline{\alpha}_{t}}x_{0}\big)\mathrm{d}x_{0}\eqqcolon-\frac{1}{1-\overline{\alpha}_{t}}g_{t}(x).
\]
Let $J_{t}(x)=\partial g_{t}(x)/\partial x$ be the Jacobian matrix
of $g_{t}(x)$, which can be expressed as 
\begin{align}
J_{t}(x) & =I+\frac{1}{1-\overline{\alpha}_{t}}\bigg\{\Big(\int_{x_{0}}p_{X_{0}|X_{t}}(x_{0}\mymid x)\big(x-\sqrt{\overline{\alpha}_{t}}x_{0}\big)\mathrm{d}x_{0}\Big)\Big(\int_{x_{0}}p_{X_{0}|X_{t}}(x_{0}\mymid x)\big(x-\sqrt{\overline{\alpha}_{t}}x_{0}\big)\mathrm{d}x_{0}\Big)^{\top}\nonumber \\
 & \qquad\qquad\qquad\quad-\int_{x_{0}}p_{X_{0}|X_{t}}(x_{0}\mymid x)\big(x-\sqrt{\overline{\alpha}_{t}}x_{0}\big)\big(x-\sqrt{\overline{\alpha}_{t}}x_{0}\big)^{\top}\mathrm{d}x_{0}\bigg\}.\label{eq:Jt-defn}
\end{align}
It is straightforward to check that $I-J_{t}(x_{t})\succeq0$. The
following lemma will be useful in the analysis.

\begin{lemma} \label{lem:typical} Suppose that $x\in\mathbb{R}^{d}$
satisfies $-\log p_{X_{t}}(x)\le\theta d\log T$ for any given $\theta\ge1$.
Then we have 
\[
\|s_{t}^{\star}(x)\|_{2}\leq5\sqrt{\frac{(\theta+c_{0})d\log T}{1-\overline{\alpha}_{t}}}\qquad\text{and}\qquad\mathsf{Tr}(I-J_{t}(x))\leq12(\theta+c_{0})d\log T,
\]
where the constant $c_{0}>0$ is defined in (\ref{eq:learning-rate}).
In addition, there exists universal constant $C_{0}>0$ such that
\[
\sum_{t=2}^{T}\frac{1-\alpha_{t}}{1-\overline{\alpha}_{t}}\int_{x_{t}}\Vert J_{t}(x_{t})\Vert_{\mathrm{F}}^{2}\,p_{X_{t}}(x_{t})\mathrm{d}x_{t}\leq C_{0}d\log T.
\]
\end{lemma}\begin{proof}See Appendix~\ref{sec:proof-lemma-typical}.\end{proof}

For some sufficiently large constants $C_{1},C_{2}>0$, we define
for each $2\leq t\leq T$ the set
\begin{align}
\mathcal{E}_{t,1} & \coloneqq\big\{ x_{t}:-\log p_{X_{t}}(x_{t})\leq C_{1}d\log T,\|x_{t}\|_{2}\leq\sqrt{\overline{\alpha}_{t}}T^{2c_{R}}+C_{2}\sqrt{d(1-\overline{\alpha}_{t})\log T}\big\}\label{eq:defn-E-t-1}
\end{align}
Define the extended $d$-dimensional Euclidean space $\mathbb{R}^{d}\cup\{\infty\}$
by adding a point $\infty$ to $\mathbb{R}^{d}$. From now on, the
random vectors can take value in $\mathbb{R}^{d}\cup\{\infty\}$,
namely, they can be constructed in the following way: 
\[
X=\begin{cases}
X', & \text{with probability }\theta,\\
\infty, & \text{with probability }1-\theta,
\end{cases}
\]
where $\theta\in[0,1]$ and $X'$ is a random vector in $\mathbb{R}^{d}$
in the usual sense. If $X'$ has a density, denoted by $p_{X'}(\cdot)$,
then the generalized density of $X$ is 
\[
p_{X}(x)=\theta p_{X'}(x)\ind\{x\in\mathbb{R}^{d}\}+(1-\theta)\delta_{\infty}.
\]
To simplify presentation, we will abbreviate generalized density to
density.

\subsection{Step 1: introducing auxiliary sequences}

We first define an auxiliary reverse process that uses the true score
function: 
\begin{equation}
Y_{T}^{\star}\sim\mathcal{N}(0,I_{d}),\qquad Y_{t-1}^{\star}=\frac{1}{\sqrt{\alpha_{t}}}\Big(Y_{t}^{\star}+(1-\alpha_{t})s_{t}^{\star}(Y_{t}^{\star})+\sqrt{1-\alpha_{t}}Z_{t}\Big)\quad\text{for }t=T,\ldots,1.\label{eq:SDE-Ystar}
\end{equation}
To control discretization error, we introduce an auxiliary sequence
$\{\overline{Y}_{t}:t=T,\ldots,1\}$ along with intermediate variables
$\{\overline{Y}_{t}^{-}:t=T,\ldots,1\}$ as follows. \begin{subequations}\label{subeq:defn-Y-bar}
\begin{enumerate}
\item (Initialization) Define $\overline{Y}_{T}^{-}=Y_{T}$ if $Y_{T}\in\mathcal{E}_{T,1}$
and $\overline{Y}_{T}^{-}=\infty$ otherwise. The density of $\overline{Y}_{T}^{-}$
is 
\begin{equation}
p_{\overline{Y}_{T}^{-}}(y_{T}^{-})=p_{Y_{T}}(y_{T}^{-})\ind\big\{ y_{T}^{-}\in\mathcal{E}_{T,1}\big\}+\int_{y\in\mathcal{E}_{T,1}^{\mathrm{c}}}p_{Y_{T}}(y)\mathrm{d}y\delta_{\infty}.\label{eq:transition-YTbarminus}
\end{equation}
\item (Transition from $\overline{Y}_{t}^{-}$ to $\overline{Y}_{t}$) For
$t=T,\ldots,1$, the conditional density of $\overline{Y}_{t}$ given
$\overline{Y}_{t}^{-}=y_{t}^{-}$ is 
\begin{equation}
p_{\overline{Y}_{t}|\overline{Y}_{t}^{-}}(y_{t}\mymid y_{t}^{-})=\min\big\{ p_{X_{t}}(y_{t}^{-})/p_{\overline{Y}_{t}^{-}}(y_{t}^{-}),1\big\}\delta_{y_{t}^{-}}+\big(1-\min\big\{ p_{X_{t}}(y_{t}^{-})/p_{\overline{Y}_{t}^{-}}(y_{t}^{-}),1\big\}\big)\delta_{\infty}.\label{eq:transition-Ytbarminus-Ytbar}
\end{equation}
\item (Transition from $\overline{Y}_{t}$ to $\overline{Y}_{t-1}^{-}$)
For $t=T,\ldots,2$, the conditional density of $\overline{Y}_{t-1}^{-}$
given $\overline{Y}_{t}=y_{t}$ is defined as follows: if $y_{t}\in\mathcal{E}_{t,1}$,
then 
\begin{equation}
p_{\overline{Y}_{t-1}^{-}|\overline{Y}_{t}}(y_{t-1}^{-}\mymid y_{t})=p_{Y_{t-1}^{\star}|Y_{t}^{\star}}(y_{t-1}^{-}\mymid y_{t});\label{eq:transition-Ytbar-Yt-1barminus}
\end{equation}
otherwise, we let $p_{\overline{Y}_{t-1}^{-}|\overline{Y}_{t}}(y_{t-1}^{-}\mymid y_{t})=\delta_{\infty}$. 
\end{enumerate}
\end{subequations}This defines a Markov chain 
\begin{equation}
Y_{T}\to\overline{Y}_{T}^{-}\to\overline{Y}_{T}\to\overline{Y}_{T-1}^{-}\to\overline{Y}_{T-1}\to\cdots\to\overline{Y}_{1}^{-}\to\overline{Y}_{1}.\label{eq:Ybar-chain}
\end{equation}
An important consequence of the construction (\ref{eq:transition-Ytbarminus-Ytbar})
is that, for any $y_{t}\neq\infty$, 
\begin{align}
p_{\overline{Y}_{t}}(y_{t}) & =\int_{\mathbb{R}^{d}}p_{\overline{Y}_{t}|\overline{Y}_{t}^{-}}(y_{t}\mymid y_{t}^{-})p_{\overline{Y}_{t}^{-}}(y_{t}^{-})\mathrm{d}y_{t}^{-}=\min\big\{ p_{X_{t}}(y_{t}),p_{\overline{Y}_{t}^{-}}(y_{t})\big\}.\label{eq:transition-fact-1}
\end{align}
To control estimation error, we introduce another auxiliary sequence
$\{\widehat{Y}_{t}:t=T,\ldots,1\}$ along with intermediate variables
$\{\widehat{Y}_{t}^{-}:t=T,\ldots,1\}$ as follows. \begin{subequations}\label{subeq:defn-Y-hat}
\begin{enumerate}
\item (Initialization) Let $\widehat{Y}_{T}^{-}=\overline{Y}_{T}^{-}$. 
\item (Transition from $\widehat{Y}_{t}^{-}$ to $\widehat{Y}_{t}$) For
$t=T,\ldots,1$, the conditional density of $\widehat{Y}_{t}$ given
$\widehat{Y}_{t}^{-}=y_{t}^{-}$ is 
\begin{equation}
p_{\widehat{Y}_{t}|\widehat{Y}_{t}^{-}}(y_{t}\mymid y_{t}^{-})=p_{\overline{Y}_{t}|\overline{Y}_{t}^{-}}(y_{t}\mymid y_{t}^{-}).\label{eq:transition-Ythatminus-Ythat}
\end{equation}
\item (Transition from $\widehat{Y}_{t}$ to $\widehat{Y}_{t-1}^{-}$) For
$t=T,\ldots,2$, the conditional density of $\widehat{Y}_{t-1}^{-}$
given $\widehat{Y}_{t}=y_{t}$ is defined as follows: if $y_{t}\in\mathcal{E}_{t,1}$,
then 
\begin{equation}
p_{\widehat{Y}_{t-1}^{-}|\widehat{Y}_{t}}(y_{t-1}^{-}\mymid y_{t})=p_{Y_{t-1}|Y_{t}}(y_{t-1}^{-}\mymid y_{t});\label{eq:transition-Ythat-Yt-1hatminus}
\end{equation}
otherwise, we let $p_{\widehat{Y}_{t-1}^{-}|\widehat{Y}_{t}}(y_{t-1}^{-}\mymid y_{t})=\delta_{\infty}$. 
\end{enumerate}
\end{subequations} This defines another Markov chain 
\begin{equation}
Y_{T}\to\widehat{Y}_{T}^{-}\to\widehat{Y}_{T}\to\widehat{Y}_{T-1}^{-}\to\widehat{Y}_{T-1}\to\cdots\to\widehat{Y}_{1}^{-}\to\widehat{Y}_{1},\label{eq:Yhat-chain}
\end{equation}
which is similar to (\ref{eq:Ybar-chain}) except that now the transitions
from $\widehat{Y}_{t}$ to $\widehat{Y}_{t-1}^{-}$ are constructed
using the estimated score functions. We can use induction to show
that 
\begin{equation}
p_{Y_{t}}(y_{t})\ge p_{\widehat{Y}_{t}}(y_{t}),\qquad\forall\,y_{t}\ne\infty\label{eq:transition-fact-2}
\end{equation}
holds for all $t=T,\ldots,1$. First, it is straightforward to check
that (\ref{eq:transition-fact-2}) holds for $t=T$. Suppose that
(\ref{eq:transition-fact-2}) holds for $t+1$. Then for any $y_{t}\neq\infty$,
we have 
\begin{align*}
p_{\widehat{Y}_{t}}(y_{t}) & =\int_{\mathbb{R}^{d}}p_{\widehat{Y}_{t}|\widehat{Y}_{t}^{-}}(y_{t}\mymid y_{t}^{-})p_{\widehat{Y}_{t}^{-}}(y_{t}^{-})\mathrm{d}y_{t}^{-}\overset{\text{(i)}}{=}\min\big\{ p_{X_{t}}(y_{t})/p_{\overline{Y}_{t}^{-}}(y_{t}),1\big\} p_{\widehat{Y}_{t}^{-}}(y_{t})\leq p_{\widehat{Y}_{t}^{-}}(y_{t})\\
 & =\int_{\mathbb{R}^{d}}p_{\widehat{Y}_{t}^{-}|\widehat{Y}_{t+1}}(y_{t}\mymid y_{t+1})p_{\widehat{Y}_{t+1}}(y_{t+1})\mathrm{d}y_{t+1}\overset{\text{(ii)}}{\leq}\int p_{Y_{t}|Y_{t+1}}(y_{t}\mymid y_{t+1})p_{Y_{t+1}}(y_{t+1})\mathrm{d}y_{t+1}=p_{Y_{t}}(y_{t}).
\end{align*}
Here step (i) follows from (\ref{eq:transition-Ythatminus-Ythat})
and (\ref{eq:transition-Ytbarminus-Ytbar}), while step (ii) follows
from the induction hypothesis and (\ref{eq:transition-Ythat-Yt-1hatminus}).

\subsection{Step 2: controlling discretization error\protect\label{subsec:proof-main-discretization}}

In this section, we will bound the total variation distance between
$p_{X_{1}}$ and $p_{\overline{Y}_{1}}$. For each $t=T,\ldots,1$,
let 
\begin{equation}
\Delta_{t}(x)\coloneqq p_{X_{t}}(x)-p_{\overline{Y}_{t}}(x),\qquad\forall\,x\in\mathbb{R}^{d}.\label{eq:induction-1}
\end{equation}
We emphasize that $\Delta_{t}(\cdot)$ is not defined at $\infty$.
In view of (\ref{eq:transition-fact-1}), we know that $\Delta_{t}(x_{t})\ge0$
for any $x_{t}\neq\infty$. The following lemma characterizes the
propagation of the error $\int\Delta_{t}(x)\mathrm{d}x$ through the
reverse process.

\begin{lemma}\label{lemma:recursion} There exists some universal
constant $C_{4}>0$ such that, for $t=T,\ldots,2$, 
\[
\int\Delta_{t-1}(x)\mathrm{d}x\le\int\Delta_{t}(x)\mathrm{d}x+C_{4}\Big(\frac{1-\alpha_{t}}{1-\overline{\alpha}_{t}}\Big)^{2}\int_{x_{t}\in\mathcal{E}_{t,1}}\big(d\log T+\|J_{t}(x_{t})\|_{\mathsf{F}}^{2}\big)p_{X_{t}}(x_{t})\mathrm{d}x_{t}+T^{-3}.
\]
In addition, we have $\int\Delta_{T}(x)\mathrm{d}x\leq T^{-4}$. \end{lemma}\begin{proof}See
Appendix~\ref{sec:proof-lemma-recursion}.\end{proof}

We can apply Lemma~\ref{lemma:recursion} recursively to achieve
\begin{align*}
\int\Delta_{1}(x)\mathrm{d}x & \leq\int\Delta_{T}(x)\mathrm{d}x+\sum_{t=2}^{T}\Big[C_{4}\Big(\frac{1-\alpha_{t}}{1-\overline{\alpha}_{t}}\Big)^{2}\int_{x_{t}\in\mathcal{E}_{t,1}}\big(d\log T+\|J_{t}(x_{t})\Vert_{\mathrm{F}}^{2}\big)p_{X_{t}}(x_{t})\mathrm{d}x_{t}+T^{-3}\Big]\\
 & \overset{\text{(a)}}{\leq}8c_{1}C_{4}\frac{\log T}{T}\sum_{t=2}^{T}\frac{1-\alpha_{t}}{1-\overline{\alpha}_{t}}\int_{x_{t}\in\mathcal{E}_{t,1}}\|J_{t}(x_{t})\Vert_{\mathrm{F}}^{2}p_{X_{t}}(x_{t})\mathrm{d}x_{t}+64c_{1}^{2}C_{4}\frac{d\log^{3}T}{T}+T^{-2}\\
 & \overset{\text{(b)}}{\leq}8c_{1}C_{4}C_{0}\frac{d\log^{2}T}{T}+64c_{1}^{2}C_{4}\frac{d\log^{3}T}{T}+T^{-3}\leq C_{5}\frac{d\log^{3}T}{T}.
\end{align*}
Here step (a) utilizes Lemma~\ref{lemma:step-size}; step (b) follows
from Lemma~\ref{lem:typical}; while step (c) holds provided that
$C_{5}\gg c_{1}^{2}C_{4}C_{0}$. This further implies that 
\begin{align}
\mathsf{TV}(p_{X_{1}},p_{\overline{Y}_{1}}) & =\int_{p_{X_{1}}(x)>p_{\overline{Y}_{1}}(x)}\big(p_{X_{1}}(x)-p_{\overline{Y}_{1}}(x)\big)\mathrm{d}x=\int\Delta_{1}(x)\mathrm{d}x\leq C_{5}\frac{d\log^{3}T}{T}.\label{eq:proof-TV-1}
\end{align}

\subsection{Step 3: controlling estimation error\protect\label{subsec:proof-main-estimation}}

In this section, we will bound the total variation distance between
$p_{Y_{1}}$ and $p_{\overline{Y}_{1}}$. Note that 
\begin{align}
\mathsf{TV}\big(p_{Y_{1}},p_{\overline{Y}_{1}}\big) & =\int_{\mathbb{R}^{d}}\big(p_{\overline{Y}_{1}}(x)-p_{Y_{1}}(x)\big)\ind\big\{ p_{\overline{Y}_{1}}(x)>p_{Y_{1}}(x)\big\}\mathrm{d}x+\mathbb{P}\big(\overline{Y}_{1}=\infty\big)\nonumber \\
 & \overset{\text{(i)}}{\leq}\int_{\mathbb{R}^{d}}\big(p_{\overline{Y}_{1}}(x)-p_{\widehat{Y}_{1}}(x)\big)\ind\big\{ p_{\overline{Y}_{1}}(x)>p_{\widehat{Y}_{1}}(x)\big\}\mathrm{d}x+\mathbb{P}\big(\overline{Y}_{1}=\infty\big)\nonumber \\
 & \overset{\text{(ii)}}{\leq}\mathsf{TV}\big(p_{\overline{Y}_{1}},p_{\widehat{Y}_{1}}\big)+\mathsf{TV}\big(p_{X_{1}},p_{\overline{Y}_{1}}\big)\overset{\text{(iii)}}{\leq}\sqrt{\mathsf{KL}\big(p_{\overline{Y}_{1}}\Vert p_{\widehat{Y}_{1}}\big)}+C_{5}\frac{d\log^{3}T}{T}.\label{eq:proof-main-6}
\end{align}
Here step (i) follows from (\ref{eq:transition-fact-2}); step (ii)
follows from $\mathbb{P}(\overline{Y}_{1}=\infty)\leq\mathsf{TV}(p_{X_{1}},p_{\overline{Y}_{1}})$,
which holds since $X_{1}$ does not take value at $\infty$; step
(iii) utilizes Pinsker's inequality and (\ref{eq:proof-TV-1}). Hence
it suffices to bound $\mathsf{KL}(p_{\overline{Y}_{1}}\parallel p_{\widehat{Y}_{1}})$,
which can be decomposed into 
\begin{align}
 & \mathsf{KL}\big(p_{\overline{Y}_{1}}\Vert p_{\widehat{Y}_{1}}\big)\overset{\text{(a)}}{\leq}\mathsf{KL}\big(p_{\overline{Y}_{1},\overline{Y}_{1}^{-},\ldots,\overline{Y}_{T},\overline{Y}_{T}^{-}}\Vert p_{\widehat{Y}_{1},\widehat{Y}_{1}^{-},\ldots,\widehat{Y}_{T},\widehat{Y}_{T}^{-}}\big)\nonumber \\
 & \quad\overset{\text{(b)}}{=}\mathsf{KL}\big(p_{\overline{Y}_{T}^{-}}\Vert p_{\widehat{Y}_{T}^{-}}\big)+\sum_{t=2}^{T}\mathop{\mathbb{E}}_{x_{t}\sim p_{\overline{Y}_{t}}}\hspace{-1ex}\big[\mathsf{KL}\big(p_{\overline{Y}_{t-1}^{-}|\overline{Y}_{t}=x_{t}}\Vert p_{\widehat{Y}_{t-1}^{-}|\widehat{Y}_{t}=x_{t}}\big)\big]+\sum_{t=1}^{T}\mathop{\mathbb{E}}_{x_{t}\sim p_{\overline{Y}_{t}^{-}}}\hspace{-1ex}\big[\mathsf{KL}\big(p_{\overline{Y}_{t}|\overline{Y}_{t}^{-}=x_{t}}\Vert p_{\widehat{Y}_{t}|\widehat{Y}_{t}^{-}=x_{t}}\big)\big]\nonumber \\
 & \quad\overset{\text{(c)}}{=}\sum_{t=2}^{T}\mathbb{E}_{x_{t}\sim p_{\overline{Y}_{t}}}\big[\mathsf{KL}\big(p_{\overline{Y}_{t-1}^{-}|\overline{Y}_{t}=x_{t}}\Vert p_{\widehat{Y}_{t-1}^{-}|\widehat{Y}_{t}=x_{t}}\big)\big].\label{eq:proof-main-7}
\end{align}
Here step (a) follows from the data-processing inequality; step (b)
uses the chain rule of KL divergence, where we use the fact that (\ref{eq:Ybar-chain})
and (\ref{eq:Yhat-chain}) are both Markov chains; step (c) follows
from the facts that, by construction, $\overline{Y}_{T}^{-}=\widehat{Y}_{T}^{-}$,
and for any $x\neq\infty$, the conditional distributions of $\widehat{Y}_{t}$
given $\widehat{Y}_{t}^{-}=x$ and $\overline{Y}_{t}$ given $\overline{Y}_{t}^{-}=x$
are identical. For any $x_{t}\in\mathcal{E}_{t,1}$, we have
\begin{align}
\mathsf{KL}\big(p_{\overline{Y}_{t-1}^{-}|\overline{Y}_{t}=x_{t}}\Vert p_{\widehat{Y}_{t-1}^{-}|\widehat{Y}_{t}=x_{t}}\big) & \overset{\text{(i)}}{=}\frac{1-\alpha_{t}}{2}\Vert s_{t}(x_{t})-s_{t}^{\star}(x_{t})\Vert_{2}^{2}\overset{\text{(ii)}}{\leq}\frac{c_{1}\log T}{2T}\Vert s_{t}(x_{t})-s_{t}^{\star}(x_{t})\Vert_{2}^{2}.\label{eq:proof-main-8}
\end{align}
Here step (i) follows from the transition probability (\ref{eq:transition-Ytbar-Yt-1barminus})
and (\ref{eq:transition-Ythat-Yt-1hatminus}), which give
\begin{align*}
\overline{Y}_{t-1}^{-}|\overline{Y}_{t}=x_{t} & \sim\mathcal{N}\left(\frac{x_{t}+(1-\alpha_{t})s_{t}^{\star}(x_{t})}{\sqrt{\alpha_{t}}},\frac{1-\alpha_{t}}{\alpha_{t}}I_{d}\right)\quad\text{and}\\
\widehat{Y}_{t-1}^{-}|\widehat{Y}_{t}=x_{t} & \sim\mathcal{N}\left(\frac{x_{t}+(1-\alpha_{t})s_{t}(x_{t})}{\sqrt{\alpha_{t}}},\frac{1-\alpha_{t}}{\alpha_{t}}I_{d}\right),
\end{align*}
and the KL divergence between two Gaussian measures can be computed
in closed-form; step (ii) utilizes Lemma~\ref{lemma:step-size}.
On the other hand, for any $x_{t}\in\mathcal{E}_{t,1}^{\mathrm{c}}$,
we have 
\begin{equation}
\mathsf{KL}\big(p_{\overline{Y}_{t-1}^{-}|\overline{Y}_{t}=x_{t}}\parallel p_{\widehat{Y}_{t-1}^{-}|\widehat{Y}_{t}=x_{t}}\big)=0.\label{eq:proof-main-9}
\end{equation}
Therefore we have 
\begin{align}
\mathsf{KL}\big(p_{\overline{Y}_{1}}\parallel p_{\widehat{Y}_{1}}\big) & \overset{\text{(i)}}{\leq}\sum_{t=2}^{T}\mathbb{E}_{x_{t}\sim p_{X_{t}}}\Big[\mathsf{KL}\big(p_{\overline{Y}_{t-1}^{-}|\overline{Y}_{t}=x_{t}}\parallel p_{\widehat{Y}_{t-1}^{-}|\widehat{Y}_{t}=x_{t}}\big)\Big]\overset{\text{(ii)}}{\leq}\frac{c_{1}}{2}\varepsilon_{\mathsf{score}}^{2}\log T.\label{eq:proof-main-10}
\end{align}
Here step (i) follows from (\ref{eq:proof-main-7}) and the relation
$p_{\overline{Y}_{t}}(x)\leq p_{X_{t}}(x)$ for any $x\neq\infty$
(see (\ref{eq:transition-fact-1})); while step (ii) follows from
(\ref{eq:proof-main-8}) and (\ref{eq:proof-main-9}). Substitution
of the bound (\ref{eq:proof-main-10}) into (\ref{eq:proof-main-6})
yields 
\begin{equation}
\mathsf{TV}\big(p_{Y_{1}},p_{\overline{Y}_{1}}\big)\leq\sqrt{\frac{c_{1}}{2}\log T}\varepsilon_{\mathsf{score}}+C_{5}\frac{d\log^{3}T}{T}.\label{eq:proof-TV-2}
\end{equation}
Taking the two bounds (\ref{eq:proof-TV-1}) and (\ref{eq:proof-TV-2})
collectively, we achieve the desired result 
\[
\mathsf{TV}(p_{X_{1}},p_{Y_{1}})\leq\mathsf{TV}(p_{X_{1}},p_{\overline{Y}_{1}})+\mathsf{TV}\big(p_{Y_{1}},p_{\overline{Y}_{1}}\big)\leq C\frac{d\log^{3}T}{T}+C\varepsilon_{\mathsf{score}}\sqrt{\log T}
\]
for some constant $C\gg\sqrt{c_{1}}+C_{5}$.

\section{Proof of Theorem~\ref{thm:main-low-d} \protect\label{sec:proof-main-low-d}}

This section provides the proof of Theorem \ref{thm:main-low-d}.
While the high-level analysis idea is similar to the proof of Theorem~\ref{thm:main},
we need to carry out more careful analysis in order to precisely capture
the low-dimensional structure. The constants $C_{1},C_{2},\ldots$
in this section are different from the ones in Section~\ref{sec:proof-main}.

\subsection{Preliminaries}

For simplicity of presentation, we assume without loss of generality
that $k\geq\log d$. In fact, if $k<\log d$, we can simply redefine
$k\coloneqq\log d$, which does not change the desired bound (\ref{eq:error-bound-low-d}).
Let $\{x_{i}^{\star}\}_{1\leq i\leq N_{\varepsilon}}$ be an $\varepsilon$-net
of $\mathcal{X}=\mathsf{supp}(p_{\mathsf{data}})$, where $\varepsilon$
is sufficiently small
\begin{equation}
\varepsilon\ll\sqrt{\frac{1-\overline{\alpha}_{t}}{\overline{\alpha}_{t}}}\min\left\{ 1,\sqrt{\frac{k\log T}{d}}\right\} ,\label{eq:eps-condition}
\end{equation}
 and let $\{\mathcal{B}_{i}\}_{1\leq i\leq N_{\varepsilon}}$ be a
disjoint $\varepsilon$-cover for $\mathcal{X}$ such that $x_{i}^{\star}\in\mathcal{B}_{i}$.
Let 
\begin{align*}
\mathcal{I} & \coloneqq\left\{ 1\leq i\leq N_{\varepsilon}:\mathbb{P}(X_{0}\in\mathcal{B}_{i})\geq\exp(-\theta k\log T)\right\} ,\\
\mathcal{G} & \coloneqq\big\{\omega\in\mathbb{R}^{d}:\Vert\omega\Vert_{2}\leq2\sqrt{d}+\sqrt{\theta k\log T},\quad\text{and}\\
 & \qquad\qquad\qquad\vert(x_{i}^{\star}-x_{j}^{\star})^{\top}\omega\vert\leq\sqrt{\theta k\log T}\Vert x_{i}^{\star}-x_{j}^{\star}\Vert_{2}\quad\text{for all}\quad1\leq i,j\leq N_{\varepsilon}\big\},
\end{align*}
where $\theta>0$ is some sufficiently large constant. Then $\cup_{i\in\mathcal{I}}\mathcal{B}_{i}$
and $\mathcal{G}$ can be interpreted as high probability sets for
the random variable $X_{0}$ and a standard Gaussian variable in $\mathbb{R}^{d}$.
For each $t=1,\ldots T$, we define a typical set for each $X_{t}$
as follows 
\begin{equation}
\mathcal{E}_{t,1}\coloneqq\left\{ \sqrt{\overline{\alpha}_{t}}x_{0}+\sqrt{1-\overline{\alpha}_{t}}\omega:x_{0}\in\cup_{i\in\mathcal{I}}\mathcal{B}_{i},\omega\in\mathcal{G}\right\} .\label{eq:E-t-1-low-d}
\end{equation}
This means that for any $x_{t}\in\mathcal{E}_{t,1}$, there exists
an index $i(x_{t})\in\mathcal{I}$ and two points $x_{0}(x_{t})\in\mathcal{B}_{i(x_{t})}$
and $\omega\in\mathcal{G}$ such that 
\begin{equation}
x_{t}=\sqrt{\overline{\alpha}_{t}}x_{0}(x_{t})+\sqrt{1-\overline{\alpha}_{t}}\omega.\label{eq:xt-decom}
\end{equation}
It is worth mentioning that such $i(x_{t})$, $x_{0}(x_{t})$ and
$\omega$ might not be unique, and we only need to arbitrarily choose
and fix one of them. For any $x_{t}\in\mathcal{E}_{t,1}$ and any
$r>0$, define
\begin{equation}
\mathcal{I}\left(x_{t};r\right)\coloneqq\left\{ 1\leq i\leq N_{\varepsilon}:\overline{\alpha}_{t}\Vert x_{i}^{\star}-x_{i(x_{t})}^{\star}\Vert_{2}^{2}\leq r\cdot k(1-\overline{\alpha}_{t})\log T\right\} .\label{eq:I-defn}
\end{equation}
The following technical lemma will be crucial in the analysis.

\begin{lemma} \label{lem:cond-low-dim} There exists some universal
constant $C_{1}\gg\theta$ such that
\begin{align*}
\mathbb{P}\left(X_{0}\in\mathcal{B}_{i}\mymid X_{t}=x_{t}\right) & \leq\exp\left(-\frac{\overline{\alpha}_{t}}{16\left(1-\overline{\alpha}_{t}\right)}\Vert x_{i(x_{t})}^{\star}-x_{i}^{\star}\Vert_{2}^{2}\right)\mathbb{P}\left(X_{0}\in\mathcal{B}_{i}\right)
\end{align*}
for any $x_{t}\in\mathcal{E}_{t,1}$ and $i\notin\mathcal{I}(x_{t};C_{1}\theta)$.
\end{lemma}\begin{proof} The proof can be found in \cite[Appendix A.2]{li2024adapting}
and is omitted here for brevity.\end{proof}

\subsection{Main proof}

We first define an auxiliary reverse process that uses the true score
function: 
\begin{equation}
Y_{T}^{\star}\sim\mathcal{N}(0,I_{d}),\qquad Y_{t-1}^{\star}=\frac{1}{\sqrt{\alpha_{t}}}(Y_{t}^{\star}+\eta_{t}^{\star}s_{t}^{\star}(Y_{t}^{\star})+\sigma_{t}^{\star}Z_{t})\quad\text{for }t=T,\ldots,1.\label{eq:SDE-Ystar-low-d}
\end{equation}
We introduce auxiliary sequences $\{\overline{Y}_{t}:t=T,\ldots,1\}$
and $\{\overline{Y}_{t}^{-}:t=T,\ldots,1\}$ as in \eqref{subeq:defn-Y-bar},
as well as $\{\widehat{Y}_{t}:t=T,\ldots,1\}$ and $\{\widehat{Y}_{t}^{-}:t=T,\ldots,1\}$
as in \eqref{subeq:defn-Y-hat}. It is worth mentioning that here
we use $\mathcal{E}_{t,1}$ in (\ref{eq:E-t-1-low-d}) as well as
the sequence $\{Y_{t}^{\star}:t=T,\ldots,1\}$ in (\ref{eq:SDE-Ystar-low-d})
when defining these auxiliary sequences. In addition, we define $\Delta_{t}(x)=p_{X_{t}}(x)-p_{\overline{Y}_{t}}(x)$
as in (\ref{eq:induction-1}).

The following lemma establishes bounds similar to Lemma~\ref{lem:typical}.
In order to avoid incurring polynomial dependency in $d$, it is important
to focus on $I-J_{t}(x_{t})$ instead of $J_{t}(x)$ itself. 

\begin{lemma} \label{lem:Jt} There exists some universal constant
$C_{2}\gg C_{1}$ such that for any $x_{t}\in\mathcal{E}_{t,1}$,
\begin{equation}
\|I-J_{t}(x_{t})\|\le\|I-J_{t}(x_{t})\|_{\mathrm{F}}\leq\vert\mathsf{Tr}(I-J_{t}(x_{t}))\vert\leq C_{2}\theta k\log T,\label{eq:I-Jt-bound-low-d}
\end{equation}
where $J_{t}(\cdot)$ is defined in (\ref{eq:Jt-defn}). In addition,
there exists universal constant $C_{0}>0$ such that 
\begin{equation}
\sum_{t=2}^{T}\frac{1-\alpha_{t}}{1-\overline{\alpha}_{t}}\int_{x_{t}}\Vert I-J_{t}(x_{t})\Vert_{\mathrm{F}}^{2}\,p_{X_{t}}(x_{t})\mathrm{d}x_{t}\leq C_{0}k\log T.\label{eq:Jt_sum-low-d}
\end{equation}
 \end{lemma}\begin{proof} See Appendix~\ref{subsec:proof-lem-Jt}.\end{proof}

It is worth mentioning that unlike $I-J_{t}(x_{t})$, the order of
$s_{t}^{\star}(x_{t})$ scales linearly with $\sqrt{d}$ even for
$x_{t}\in\mathcal{E}_{t,1}$ as in Lemma~\ref{lem:typical}. Therefore
the key difficulty of this proof is to avoid introducing any error
term that scales with $\Vert s_{t}^{\star}(x_{t})\Vert_{2}$. Next,
we establish the following lemma in analogy to Lemma~\ref{lemma:recursion}.

\begin{lemma}\label{lemma:recursion-low-d} There exists some universal
constant $C_{3}>0$ such that, for $t=T,\ldots,2$, 
\[
\int\Delta_{t-1}(x)\mathrm{d}x\le\int\Delta_{t}(x)\mathrm{d}x+C_{4}\Big(\frac{1-\alpha_{t}}{1-\overline{\alpha}_{t}}\Big)^{2}\int_{x_{t}\in\mathcal{E}_{t,1}}\big(\vert\mathsf{Tr}(I-J_{t}(x_{t}))\vert+\|I-J_{t}(x_{t})\|_{\mathsf{F}}^{2}\big)p_{X_{t}}(x_{t})\mathrm{d}x_{t}+T^{-3}.
\]
In addition, we have $\int\Delta_{T}(x)\mathrm{d}x\leq T^{-4}$. \end{lemma}

\begin{proof}See Appendix~\ref{sec:proof-lemma-recursion-low-d}.
\end{proof}

We can apply Lemma~\ref{lemma:recursion} recursively to achieve
\begin{align*}
\int\Delta_{1}(x)\mathrm{d}x & \leq\int\Delta_{T}(x)\mathrm{d}x+\sum_{t=2}^{T}\Big[C_{4}\Big(\frac{1-\alpha_{t}}{1-\overline{\alpha}_{t}}\Big)^{2}\int_{x_{t}\in\mathcal{E}_{t,1}}\big(\vert\mathsf{Tr}(I-J_{t}(x_{t}))\vert+\|I-J_{t}(x_{t})\|_{\mathsf{F}}^{2}\big)p_{X_{t}}(x_{t})\mathrm{d}x_{t}+T^{-3}\Big]\\
 & \overset{\text{(a)}}{\leq}8c_{1}C_{4}\frac{\log T}{T}\sum_{t=2}^{T}\frac{1-\alpha_{t}}{1-\overline{\alpha}_{t}}\int_{x_{t}\in\mathcal{E}_{t,1}}\|I-J_{t}(x_{t})\|_{\mathsf{F}}^{2}p_{X_{t}}(x_{t})\mathrm{d}x_{t}+64c_{1}^{2}C_{4}\frac{\theta k\log^{3}T}{T}+T^{-2}\\
 & \overset{\text{(b)}}{\leq}8c_{1}C_{4}C_{0}\frac{k\log^{2}T}{T}+64c_{1}^{2}C_{4}\frac{\theta k\log^{3}T}{T}+T^{-3}\leq C_{5}\frac{k\log^{3}T}{T}.
\end{align*}
Here step (a) utilizes Lemma~\ref{lemma:step-size}; step (b) follows
from Lemma~\ref{lem:typical}; while step (c) holds provided that
$C_{5}\gg c_{1}^{2}C_{4}C_{0}\theta$. This further implies that 
\begin{align}
\mathsf{TV}(p_{X_{1}},p_{\overline{Y}_{1}}) & =\int_{p_{X_{1}}(x)>p_{\overline{Y}_{1}}(x)}\big(p_{X_{1}}(x)-p_{\overline{Y}_{1}}(x)\big)\mathrm{d}x=\int\Delta_{1}(x)\mathrm{d}x\leq C_{5}\frac{k\log^{3}T}{T}.\label{eq:proof-TV-1-low-d}
\end{align}
Equipped with (\ref{eq:proof-TV-1-low-d}), we can follow the same
steps in Section~\ref{subsec:proof-main-estimation} to control the
estimation error, which gives
\[
\mathsf{TV}(p_{X_{1}},p_{Y_{1}})\leq C\frac{k\log^{3}T}{T}+C\varepsilon_{\mathsf{score}}\sqrt{\log T}
\]
as claimed, provided that $C\gg\sqrt{c_{1}}+C_{5}$.

%% file: discussion.tex
\section{Discussion}

In this paper, we establish an $O(d/T)$ convergence theory for the
DDPM sampler, assuming access to $\ell_{2}$-accurate score estimates. This significantly improves upon the state-of-the-art convergence
rate of $O(\sqrt{d/T})$ in \citet{benton2023linear}. Compared to the recent work of \citet{li2024sharp}, which also achieves an $O(d/T)$ rate for another DDIM sampler, our approach relaxes stringent score estimation requirements, such as the need for the Jacobian of the score estimates to closely match that of the true score functions. Furthermore, to account for low-dimensional structures in the target data distribution, we extend our theory to achieve an $O(k/T)$ convergence bound under careful coefficient design, where $k$ is the intrinsic dimension. This improves upon the prior convergence rate of $O(\sqrt{k/T})$ established in \citet{potaptchik2024linear,huang2024denoising}.

This work opens several promising directions for future research.
For example, it remains unclear whether the $O(d/T)$ convergence rate is tight for the DDPM sampler; it would be of interest to develop lower bounds on certain
hard instances. 
Another intriguing direction is to explore whether
the analysis in this paper can extend to developing convergence theory
in Wasserstein distance (e.g., \citet{gao2024convergence,benton2023error}). 
Finally, while this paper focuses on analyzing the discretization error of the DDPM sampler and treats the score matching stage as a black box, it would be worthwhile to design score matching algorithms that adapt to unknown low-dimensional structures in the target data distribution.

%% file: appendix_sde.tex
\section{Proof of auxiliary lemmas in Section~\ref{sec:proof-main}}

\subsection{Proof of Lemma~\ref{lem:typical}\protect\protect\label{sec:proof-lemma-typical}}

For any pairs $(x,x_{0})\in\mathbb{R}^{d}\times\mathbb{R}^{d}$ satisfying
\begin{equation}
\|x-\sqrt{\overline{\alpha}_{t}}x_{0}\|_{2}^{2}\geq(6\theta+3c_{0})d(1-\overline{\alpha}_{t})\log T\eqqcolon R^{2}\label{eq:proof-lemma-1-1}
\end{equation}
where $c_{0}$ is defined in (\ref{eq:learning-rate}), we have 
\begin{align}
p_{X_{0}|X_{t}}(x_{0}\mymid x) & =\frac{p_{X_{0}}(x_{0})}{p_{X_{t}}(x)}p_{X_{t}|X_{0}}(x\mymid x_{0})\nonumber \\
 & \overset{\text{(i)}}{=}p_{X_{0}}(x_{0})\cdot\big(2\pi(1-\overline{\alpha}_{t})\big)^{-d/2}\exp\Big(-\frac{\|x-\sqrt{\overline{\alpha}_{t}}x_{0}\|_{2}^{2}}{2(1-\overline{\alpha}_{t})}-\log p_{X_{t}}(x)\Big)\nonumber \\
 & \overset{\text{(ii)}}{\leq}p_{X_{0}}(x_{0})\exp\Big(-\frac{\|x-\sqrt{\overline{\alpha}_{t}}x_{0}\|_{2}^{2}}{3(1-\overline{\alpha}_{t})}\Big).\label{eq:proof-lemma-1-2}
\end{align}
Here step (i) uses the fact that $X_{t}\mymid X_{0}=x_{0}\sim\mathcal{N}(\sqrt{\overline{\alpha}_{t}}x_{0},(1-\overline{\alpha}_{t})I_{d})$,
while step (ii) holds since 
\begin{align*}
-\frac{d}{2}\log2\pi(1-\overline{\alpha}_{t})-\frac{\|x-\sqrt{\overline{\alpha}_{t}}x_{0}\|_{2}^{2}}{2(1-\overline{\alpha}_{t})}-\log p_{X_{t}}(x) & \overset{\text{(iii)}}{\leq}\frac{c_{0}}{2}d\log T-\frac{\|x-\sqrt{\overline{\alpha}_{t}}x_{0}\|_{2}^{2}}{2(1-\overline{\alpha}_{t})}+\theta d\log T\\
 & \overset{\text{(iv)}}{\leq}-\frac{\|x-\sqrt{\overline{\alpha}_{t}}x_{0}\|_{2}^{2}}{3(1-\overline{\alpha}_{t})},
\end{align*}
where step (iii) follows from the fact that $1-\overline{\alpha}_{t}\geq1-\alpha_{1}=\beta_{1}$
for any $1\leq t\leq T$, and $-\log p_{X_{t}}(x)\le\theta d\log T$;
step (iv) follows from (\ref{eq:proof-lemma-1-1}). Recall that 
\begin{equation}
s_{t}^{\star}(x)=-\frac{1}{1-\overline{\alpha}_{t}}\int_{x_{0}}p_{X_{0}|X_{t}}(x_{0}\mymid x)\big(x-\sqrt{\overline{\alpha}_{t}}x_{0}\big)\mathrm{d}x_{0}\label{eq:proof-lemma-1-3}
\end{equation}
and 
\begin{equation}
\mathsf{Tr}\left(I-J_{t}(x)\right)=\frac{1}{1-\overline{\alpha}_{t}}\Big(\int_{x_{0}}p_{X_{0}|X_{t}}(x_{0}\mymid x)\Vert x-\sqrt{\overline{\alpha}_{t}}x_{0}\Vert_{2}^{2}\mathrm{d}x_{0}-\big\Vert\int_{x_{0}}p_{X_{0}|X_{t}}(x_{0}\mymid x)\big(x-\sqrt{\overline{\alpha}_{t}}x_{0}\big)\mathrm{d}x_{0}\big\Vert_{2}^{2}\Big).\label{eq:proof-lemma-1-4}
\end{equation}
Then we have 
\begin{align}
 & \Vert s_{t}^{\star}(x)\Vert_{2}=\frac{1}{1-\overline{\alpha}_{t}}\Big\|\int_{x_{0}}p_{X_{0}|X_{t}}(x_{0}\mymid x)\big(x-\sqrt{\overline{\alpha}_{t}}x_{0}\big)\mathrm{d}x_{0}\Big\|_{2}\overset{\text{(a)}}{\leq}\frac{1}{1-\overline{\alpha}_{t}}\int_{x_{0}}p_{X_{0}|X_{t}}(x_{0}\mymid x)\Vert x-\sqrt{\overline{\alpha}_{t}}x_{0}\Vert_{2}\mathrm{d}x_{0}\nonumber \\
 & \qquad\leq\frac{1}{1-\overline{\alpha}_{t}}\int p_{X_{0}|X_{t}}(x_{0}\mymid x)\Vert x-\sqrt{\overline{\alpha}_{t}}x_{0}\Vert_{2}\ind\left\{ \|x-\sqrt{\overline{\alpha}_{t}}x_{0}\|_{2}\leq R\right\} \mathrm{d}x_{0}\nonumber \\
 & \qquad\qquad+\frac{1}{1-\overline{\alpha}_{t}}\int p_{X_{0}|X_{t}}(x_{0}\mymid x)\Vert x-\sqrt{\overline{\alpha}_{t}}x_{0}\Vert_{2}\ind\left\{ \|x-\sqrt{\overline{\alpha}_{t}}x_{0}\|_{2}>R\right\} \mathrm{d}x_{0}\nonumber \\
 & \qquad\overset{\text{(b)}}{\leq}\frac{R}{1-\overline{\alpha}_{t}}+\frac{1}{1-\overline{\alpha}_{t}}\int p_{X_{0}}(x_{0})\exp\Big(-\frac{\|x-\sqrt{\overline{\alpha}_{t}}x_{0}\|_{2}^{2}}{3(1-\overline{\alpha}_{t})}\Big)\|x-\sqrt{\overline{\alpha}_{t}}x_{0}\|_{2}\ind\left\{ \|x-\sqrt{\overline{\alpha}_{t}}x_{0}\|_{2}>R\right\} \mathrm{d}x_{0}\nonumber \\
 & \qquad\overset{\text{(c)}}{\leq}\frac{R}{1-\overline{\alpha}_{t}}+\sqrt{\frac{3}{1-\overline{\alpha}_{t}}}\int p_{X_{0}}(x_{0})\exp\Big(-\frac{\|x-\sqrt{\overline{\alpha}_{t}}x_{0}\|_{2}^{2}}{6(1-\overline{\alpha}_{t})}\Big)\ind\left\{ \|x-\sqrt{\overline{\alpha}_{t}}x_{0}\|_{2}>R\right\} \mathrm{d}x_{0}\nonumber \\
 & \qquad\leq\frac{R}{1-\overline{\alpha}_{t}}+\sqrt{\frac{3}{1-\overline{\alpha}_{t}}}\exp\Big(-\frac{R^{2}}{6(1-\overline{\alpha}_{t})}\Big)\overset{\text{(d)}}{\leq}\frac{2R}{1-\overline{\alpha}_{t}}.\label{eq:proof-lemma-1-5}
\end{align}
Here step (a) utilizes Jensen's inequality; step (b) follows from
(\ref{eq:proof-lemma-1-2}); step (c) follows from the fact that $z\exp(-z^{2})\leq\exp\left(-z^{2}/2\right)$
holds for any $z\geq0$; whereas step (d) holds provided that $c_{0}$
is sufficiently large. In addition, we have 
\begin{align*}
\mathsf{Tr}(I-J_{t}(x)) & \leq\frac{1}{1-\overline{\alpha}_{t}}\mathbb{E}\left[\Vert X_{t}-\sqrt{\overline{\alpha}_{t}}X_{0}\Vert_{2}^{2}\mymid X_{t}=x\right]=\frac{1}{1-\overline{\alpha}_{t}}\int_{x_{0}}p_{X_{0}|X_{t}}(x_{0}\mymid x)\Vert x-\sqrt{\overline{\alpha}_{t}}x_{0}\Vert_{2}^{2}\mathrm{d}x_{0}.
\end{align*}
Then we can use the analysis similar to (\ref{eq:proof-lemma-1-5})
to show that 
\begin{align}
 & \mathsf{Tr}(I-J_{t}(x))\overset{\text{(i)}}{\leq}\frac{1}{1-\overline{\alpha}_{t}}\int_{x_{0}}p_{X_{0}|X_{t}}(x_{0}\mymid x)\Vert x-\sqrt{\overline{\alpha}_{t}}x_{0}\Vert_{2}^{2}\mathrm{d}x_{0}\nonumber \\
 & \qquad\leq\frac{1}{1-\overline{\alpha}_{t}}\int p_{X_{0}|X_{t}}(x_{0}\mymid x)\Vert x-\sqrt{\overline{\alpha}_{t}}x_{0}\Vert_{2}^{2}\ind\left\{ \|x-\sqrt{\overline{\alpha}_{t}}x_{0}\|_{2}\leq R\right\} \mathrm{d}x_{0}\nonumber \\
 & \qquad\qquad+\frac{1}{1-\overline{\alpha}_{t}}\int p_{X_{0}|X_{t}}(x_{0}\mymid x)\Vert x-\sqrt{\overline{\alpha}_{t}}x_{0}\Vert_{2}^{2}\ind\left\{ \|x-\sqrt{\overline{\alpha}_{t}}x_{0}\|_{2}>R\right\} \mathrm{d}x_{0}\nonumber \\
 & \qquad\overset{\text{(ii)}}{\leq}\frac{R^{2}}{1-\overline{\alpha}_{t}}+\frac{1}{1-\overline{\alpha}_{t}}\int p_{X_{0}}(x_{0})\exp\Big(-\frac{\|x-\sqrt{\overline{\alpha}_{t}}x_{0}\|_{2}^{2}}{3(1-\overline{\alpha}_{t})}\Big)\|x-\sqrt{\overline{\alpha}_{t}}x_{0}\|_{2}^{2}\ind\left\{ \|x-\sqrt{\overline{\alpha}_{t}}x_{0}\|_{2}>R\right\} \mathrm{d}x_{0}\nonumber \\
 & \qquad\overset{\text{(iii)}}{\leq}\frac{R^{2}}{1-\overline{\alpha}_{t}}+3\int p_{X_{0}}(x_{0})\exp\Big(-\frac{\|x-\sqrt{\overline{\alpha}_{t}}x_{0}\|_{2}^{2}}{6(1-\overline{\alpha}_{t})}\Big)\ind\left\{ \|x-\sqrt{\overline{\alpha}_{t}}x_{0}\|_{2}>R\right\} \mathrm{d}x_{0}\nonumber \\
 & \qquad\leq\frac{R^{2}}{1-\overline{\alpha}_{t}}+3\exp\Big(-\frac{R^{2}}{6(1-\overline{\alpha}_{t})}\Big)\overset{\text{(iv)}}{\leq}\frac{2R^{2}}{1-\overline{\alpha}_{t}}.\label{eq:proof-lemma-1-6}
\end{align}
Here step (i) follows from (\eqref{eq:proof-lemma-1-4}); step (ii)
follows from (\ref{eq:proof-lemma-1-2}); step (iii) follows from
the fact that $x\exp(-x)\leq\exp\left(-x/2\right)$ holds for any
$z\geq0$; while step (iv) holds provided that $c_{0}$ is sufficiently
large.

Finally, we invoke Lemma~\ref{lemma:jacob-sum} to achieve 
\begin{equation}
\sum_{t=2}^{T}\frac{1-\alpha_{t}}{1-\overline{\alpha}_{t}}\mathsf{Tr}\big(\mathbb{E}\big[\big(\Sigma_{\overline{\alpha}_{t}}(X_{t})\big)^{2}\big]\big)\leq C_{J}d\log T,\label{eq:proof-lemma-1-7}
\end{equation}
where the matrix function $\Sigma_{\overline{\alpha}_{t}}(\cdot)$
is defined in Lemma~\ref{lemma:jacob-sum} as 
\[
\Sigma_{\overline{\alpha}_{t}}(x)\coloneqq\mathsf{Cov}\big(Z\mymid\sqrt{\overline{\alpha}_{t}}X_{0}+\sqrt{1-\overline{\alpha}_{t}}Z=x\big)
\]
for an independent $Z\sim\mathcal{N}(0,I_{d})$. It is straightforward
to check that $J_{t}(x)=I_{d}-\Sigma_{\overline{\alpha}_{t}}(x)$,
therefore we have 
\begin{align}
\sum_{t=2}^{T}\frac{1-\alpha_{t}}{1-\overline{\alpha}_{t}}\mathsf{Tr}\big(\mathbb{E}\big[\big(\Sigma_{\overline{\alpha}_{t}}(X_{t})\big)^{2}\big]\big) & =\sum_{t=2}^{T}\frac{1-\alpha_{t}}{1-\overline{\alpha}_{t}}\mathbb{E}\big[\mathsf{Tr}\big((I_{d}-J_{t}(X_{t}))^{2}\big)\big]\nonumber \\
 & =\sum_{t=2}^{T}\frac{1-\alpha_{t}}{1-\overline{\alpha}_{t}}\mathbb{E}\big[\Vert I_{d}-J_{t}(X_{t})\Vert_{\mathrm{F}}^{2}\big].\label{eq:proof-lemma-1-8}
\end{align}
Here the last relation holds since $\mathsf{Tr}(A^{2})=\Vert A\Vert_{\mathrm{F}}^{2}$
for any symmetric matrix $A$. We conclude that 
\begin{align*}
\sum_{t=2}^{T}\frac{1-\alpha_{t}}{1-\overline{\alpha}_{t}}\int_{x_{t}}\Vert J_{t}(x_{t})\Vert_{\mathrm{F}}^{2}p_{X_{t}}(x_{t})\mathrm{d}x_{t} & =\sum_{t=2}^{T}\frac{1-\alpha_{t}}{1-\overline{\alpha}_{t}}\mathbb{E}\big[\Vert J_{t}(X_{t})\Vert_{\mathrm{F}}^{2}\big]\\
 & \overset{\text{(a)}}{\leq}\sum_{t=2}^{T}\frac{1-\alpha_{t}}{1-\overline{\alpha}_{t}}\mathbb{E}\big[2\Vert I_{d}-J_{t}(X_{t})\Vert_{\mathrm{F}}^{2}+2\Vert I_{d}\Vert_{\mathrm{F}}^{2}\big]\\
 & \overset{\text{(b)}}{\leq}2C_{J}d\log T+16c_{1}d\log T\overset{\text{(c)}}{\leq}C_{0}d\log T.
\end{align*}
Here step (a) utilizes the triangle inequality and the AM-GM inequality;
step (b) follows from (\ref{eq:proof-lemma-1-7}), (\ref{eq:proof-lemma-1-8})
and Lemma~\ref{lemma:step-size}; while step (c) holds provided that
$C_{0}\gg C_{J}+c_{1}$.

\subsection{Proof of Lemma~\ref{lemma:recursion}\protect\protect\label{sec:proof-lemma-recursion}}

We first observe that 
\begin{align}
p_{\overline{Y}_{t-1}^{-}}(x_{t-1}) & \geq\int_{\mathbb{R}^{d}}p_{\overline{Y}_{t-1}^{-}|\overline{Y}_{t}}(x_{t-1}\mymid x_{t})p_{\overline{Y}_{t}}(x_{t})\mathrm{d}x_{t}\overset{\text{(i)}}{\geq}\int_{x_{t}\in\mathcal{E}_{t,1}}p_{Y_{t-1}^{\star}|Y_{t}^{\star}}(x_{t-1}\mymid x_{t})p_{\overline{Y}_{t}}(x_{t})\mathrm{d}x_{t}\nonumber \\
 & \overset{\text{(ii)}}{=}\int_{x_{t}\in\mathcal{E}_{t,1}}p_{Y_{t-1}^{\star}|Y_{t}^{\star}}(x_{t-1}\mymid x_{t})p_{X_{t}}(x_{t})\mathrm{d}x_{t}-\Delta_{t\to t-1}(x_{t-1})\label{eq:proof-main-0}
\end{align}
where we define
\[
\Delta_{t\to t-1}(x_{t-1})\coloneqq\int_{x_{t}\in\mathcal{E}_{t,1}}p_{Y_{t-1}^{\star}|Y_{t}^{\star}}(x_{t-1}\mymid x_{t})\Delta_{t}(x_{t})\mathrm{d}x_{t}\geq0.
\]
Here step (i) follows from (\ref{eq:transition-Ytbar-Yt-1barminus}),
while step (ii) makes use of the definition \eqref{eq:induction-1}.
It is straightforward to check that 
\begin{equation}
\int\Delta_{t\to t-1}(x)\mathrm{d}x=\int_{x_{t-1}}\int_{x_{t}\in\mathcal{E}_{t,1}}p_{Y_{t-1}^{\star}|Y_{t}^{\star}}(x_{t-1}\mymid x_{t})\Delta_{t}(x_{t})\mathrm{d}x_{t}\mathrm{d}x_{t-1}\leq\int\Delta_{t}(x)\mathrm{d}x.\label{eq:proof-main-0.5}
\end{equation}
For any $x_{t-1}$ such that $\Delta_{t-1}(x_{t-1})>0$, we have 
\begin{align}
 & p_{X_{t-1}}(x_{t-1})-\Delta_{t-1}(x_{t-1})+\Delta_{t\to t-1}(x_{t-1})\nonumber \\
 & \quad\overset{\text{(a)}}{=}p_{\overline{Y}_{t-1}^{-}}(x_{t-1})+\Delta_{t\to t-1}(x_{t-1})\overset{\text{(b)}}{\geq}\int_{x_{t}\in\mathcal{E}_{t,1}}p_{Y_{t-1}^{\star}|Y_{t}^{\star}}(x_{t-1}\mymid x_{t})p_{X_{t}}(x_{t})\mathrm{d}x_{t}\nonumber \\
 & \quad\overset{\text{(c)}}{=}\int_{x_{t}\in\mathcal{E}_{t,1}}p_{X_{t}}(x_{t})\Big(\frac{\alpha_{t}}{2\pi(1-\alpha_{t})}\Big)^{d/2}\exp\Big(-\frac{\big\|\sqrt{\alpha_{t}}x_{t-1}-\big(x_{t}+(1-\alpha_{t})s_{t}^{\star}(x_{t})\big)\big\|^{2}}{2(1-\alpha_{t})}\Big)\mathrm{d}x_{t}\nonumber \\
 & \quad\overset{\text{(d)}}{=}\int_{x_{t}\in\mathcal{E}_{t,1}}\mathsf{det}\Big(I-\frac{1-\alpha_{t}}{1-\overline{\alpha}_{t}}J_{t}(x_{t})\Big)^{-1}p_{X_{t}}(x_{t})\Big(\frac{\alpha_{t}}{2\pi(1-\alpha_{t})}\Big)^{d/2}\exp\Big(-\frac{\big\|\sqrt{\alpha_{t}}x_{t-1}-u_{t}\big\|^{2}}{2(1-\alpha_{t})}\Big)\mathrm{d}u_{t}.\label{eq:proof-main-1}
\end{align}
Here step (a) utilizes the definition (\ref{eq:induction-1}) and
$p_{\overline{Y}_{t-1}}(x_{t-1})=p_{\overline{Y}_{t-1}^{-}}(x_{t-1})$,
which is a consequence of (\ref{eq:transition-fact-1}) and $\Delta_{t-1}(x_{t-1})>0$;
step (b) follows from (\ref{eq:proof-main-0}); step (c) follows from
the definition (\ref{eq:SDE-Ystar}); whereas step (d) applies the
change of variable $u_{t}=x_{t}+(1-\alpha_{t})s_{t}^{\star}(x_{t})$.
Moving forward, we need the following lemma.

\begin{lemma}\label{lemma:det}For any $x_{t}\in\mathcal{E}_{t,1}$,
we have 
\begin{align}
 & \mathsf{det}\Big(I-\frac{1-\alpha_{t}}{1-\overline{\alpha}_{t}}J_{t}(x_{t})\Big)^{-1}p_{X_{t}}(x_{t})\nonumber \\
 & \qquad=\big(2\pi(2\alpha_{t}-1-\overline{\alpha}_{t})\big)^{-d/2}\int_{x_{0}}p_{X_{0}}(x_{0})\exp\Big(-\frac{\|u_{t}-\sqrt{\overline{\alpha}_{t}}x_{0}\|^{2}}{2(2\alpha_{t}-1-\overline{\alpha}_{t})}\Big)\mathrm{d}x_{0}\nonumber \\
 & \qquad\qquad\cdot\exp\Big(\xi_{t}(x_{t})+O\Big(\Big(\frac{1-\alpha_{t}}{1-\overline{\alpha}_{t}}\Big)^{2}\big(d\log T+\|J_{t}(x_{t})\|_{\mathsf{F}}^{2}\big)\Big)\Big),\label{eq:lemma-det-1}
\end{align}
where $\xi_{t}(x_{t})\le0$ satisfies 
\begin{equation}
\int_{x_{t}\in\mathcal{E}_{t,1}}|\xi_{t}(x_{t})|p_{X_{t}}(x_{t})\mathrm{d}x_{t}\leq C_{3}\Big(\frac{1-\alpha_{t}}{1-\overline{\alpha}_{t}}\Big)^{2}\int_{x_{t}\in\mathcal{E}_{t,1}}\big(d\log T+\|J_{t}(x_{t})\|_{\mathsf{F}}^{2}\big)p_{X_{t}}(x_{t})\mathrm{d}x_{t}+T^{-4}\label{eq:lemma-det-2}
\end{equation}
for some universal constant $C_{3}>0$. \end{lemma}

\begin{proof}See Appendix~\ref{sec:proof-lemma-det}.\end{proof}

Taking the decomposition (\ref{eq:lemma-det-1}) and (\ref{eq:proof-main-1})
collectively, we have 
\begin{align}
 & p_{X_{t-1}}(x_{t-1})-\Delta_{t-1}(x_{t-1})+\Delta_{t\to t-1}(x_{t-1})+\delta_{t-1}(x_{t-1})\nonumber \\
 & \quad\geq\int_{x_{0}}\int_{x_{t}}\exp\bigg(\bigg[\xi_{t}(x_{t})+O\Big(\Big(\frac{1-\alpha_{t}}{1-\overline{\alpha}_{t}}\Big)^{2}\big(d\log T+\|J_{t}(x_{t})\|_{\mathsf{F}}^{2}\big)\Big)\bigg]\ind\left\{ x_{t}\in\mathcal{E}_{t,1}\right\} \bigg)p_{X_{0}}(x_{0})\nonumber \\
 & \quad\quad\cdot\Big(\frac{\alpha_{t}}{4\pi^{2}(1-\alpha_{t})(2\alpha_{t}-1-\overline{\alpha}_{t})}\Big)^{d/2}\exp\Big(-\frac{\|u_{t}-\sqrt{\overline{\alpha}_{t}}x_{0}\|^{2}}{2(2\alpha_{t}-1-\overline{\alpha}_{t})}\Big)\exp\Big(-\frac{\big\|\sqrt{\alpha_{t}}x_{t-1}-u_{t}\big\|^{2}}{2(1-\alpha_{t})}\Big)\mathrm{d}u_{t}\mathrm{d}x_{0},\label{eq:proof-main-2}
\end{align}
where we define 
\begin{align}
\delta_{t-1}(x_{t-1}) & :=\int_{x_{0}}\int_{x_{t}\notin\mathcal{E}_{t,1}}p_{X_{0}}(x_{0})\Big(\frac{\alpha_{t}}{4\pi^{2}(1-\alpha_{t})(2\alpha_{t}-1-\overline{\alpha}_{t})}\Big)^{d/2}\nonumber \\
 & \qquad\qquad\cdot\exp\Big(-\frac{\|u_{t}-\sqrt{\overline{\alpha}_{t}}x_{0}\|^{2}}{2(2\alpha_{t}-1-\overline{\alpha}_{t})}\Big)\exp\Big(-\frac{\big\|\sqrt{\alpha_{t}}x_{t-1}-u_{t}\big\|^{2}}{2(1-\alpha_{t})}\Big)\mathrm{d}u_{t}\mathrm{d}x_{0}.\label{eq:delta-defn}
\end{align}
Moreover, it is straightforward to check that 
\begin{align}
 & \int_{x_{0}}\int_{x_{t}}p_{X_{0}}(x_{0})\Big(\frac{\alpha_{t}}{4\pi^{2}(1-\alpha_{t})(2\alpha_{t}-1-\overline{\alpha}_{t})}\Big)^{d/2}\exp\Big(-\frac{\|u_{t}-\sqrt{\overline{\alpha}_{t}}x_{0}\|^{2}}{2(2\alpha_{t}-1-\overline{\alpha}_{t})}\Big)\nonumber \\
 & \qquad\qquad\qquad\qquad\cdot\exp\Big(-\frac{\big\|\sqrt{\alpha_{t}}x_{t-1}-u_{t}\big\|^{2}}{2(1-\alpha_{t})}\Big)\mathrm{d}u_{t}\mathrm{d}x_{0}=p_{X_{t-1}}(x_{t-1}).\label{eq:proof-main-3}
\end{align}
Then we can continue the derivation in (\ref{eq:proof-main-2}): 
\begin{align*}
 & p_{X_{t-1}}(x_{t-1})-\Delta_{t-1}(x_{t-1})+\Delta_{t\to t-1}(x_{t-1})+\delta_{t-1}(x_{t-1})\\
 & \quad\overset{\text{(i)}}{\geq}\int_{x_{0}}\int_{x_{t}}\bigg(1+\bigg[\xi_{t}(x_{t})+O\Big(\Big(\frac{1-\alpha_{t}}{1-\overline{\alpha}_{t}}\Big)^{2}\big(d\log T+\|J_{t}(x_{t})\|_{\mathsf{F}}^{2}\big)\Big)\bigg]\ind\left\{ x_{t}\in\mathcal{E}_{t,1}\right\} \bigg)p_{X_{0}}(x_{0})\\
 & \quad\qquad\cdot\Big(\frac{\alpha_{t}}{4\pi^{2}(1-\alpha_{t})(2\alpha_{t}-1-\overline{\alpha}_{t})}\Big)^{d/2}\exp\Big(-\frac{\|u_{t}-\sqrt{\overline{\alpha}_{t}}x_{0}\|^{2}}{2(2\alpha_{t}-1-\overline{\alpha}_{t})}\Big)\exp\Big(-\frac{\big\|\sqrt{\alpha_{t}}x_{t-1}-u_{t}\big\|^{2}}{2(1-\alpha_{t})}\Big)\mathrm{d}u_{t}\mathrm{d}x_{0}\\
 & \quad\overset{\text{(ii)}}{=}p_{X_{t-1}}(x_{t-1})+\int_{x_{0}}\int_{x_{t}\in\mathcal{E}_{t,1}}\bigg[\xi_{t}(x_{t})+O\Big(\Big(\frac{1-\alpha_{t}}{1-\overline{\alpha}_{t}}\Big)^{2}\big(d\log T+\|J_{t}(x_{t})\|_{\mathsf{F}}^{2}\big)\Big)\bigg]p_{X_{0}}(x_{0})\\
 & \quad\qquad\cdot\Big(\frac{\alpha_{t}}{4\pi^{2}(1-\alpha_{t})(2\alpha_{t}-1-\overline{\alpha}_{t})}\Big)^{d/2}\exp\Big(-\frac{\|u_{t}-\sqrt{\overline{\alpha}_{t}}x_{0}\|^{2}}{2(2\alpha_{t}-1-\overline{\alpha}_{t})}\Big)\exp\Big(-\frac{\big\|\sqrt{\alpha_{t}}x_{t-1}-u_{t}\big\|^{2}}{2(1-\alpha_{t})}\Big)\mathrm{d}u_{t}\mathrm{d}x_{0}.
\end{align*}
Here step (i) follows from the fact that $e^{x}\ge1+x$ for all $x\in\mathbb{R}$,
while step (ii) follows from (\ref{eq:proof-main-3}). By rearranging
terms and integrate over the variable $x_{t-1}$, we arrive at 
\begin{align}
 & \int_{x_{t-1}}\Delta_{t-1}(x_{t-1})\mathrm{d}x_{t-1}\leq\int_{x_{t-1}}\big(\Delta_{t}(x_{t-1})+\delta_{t-1}(x_{t-1})\big)\mathrm{d}x_{t-1}\nonumber \\
 & \qquad+\int_{x_{0}}\int_{x_{t}\in\mathcal{E}_{t,1}}\bigg(|\xi_{t}(x_{t})|+O\Big(\Big(\frac{1-\alpha_{t}}{1-\overline{\alpha}_{t}}\Big)^{2}\big(d\log T+\|J_{t}(x_{t})\|_{\mathsf{F}}^{2}\big)\Big)\bigg)p_{X_{0}}(x_{0})\nonumber \\
 & \qquad\qquad\cdot\big(2\pi(2\alpha_{t}-1-\overline{\alpha}_{t})\big)^{-d/2}\exp\Big(-\frac{\|u_{t}-\sqrt{\overline{\alpha}_{t}}x_{0}\|_{2}^{2}}{2(2\alpha_{t}-1-\overline{\alpha}_{t})}\Big)\mathrm{d}u_{t}\mathrm{d}x_{0},\label{eq:proof-main-5}
\end{align}
where we used (\ref{eq:proof-main-0.5}) and for any fixed $u_{t}$,
the function 
\[
\left(2\pi\frac{1-\alpha_{t}}{\alpha_{t}}\right)^{-d/2}\exp\Big(-\frac{\big\|\sqrt{\alpha_{t}}x_{t-1}-u_{t}\big\|_{2}^{2}}{2(1-\alpha_{t})}\Big)
\]
is a density function of $x_{t-1}$. To establish the desired result,
we need the following two lemmas.

\begin{lemma}\label{lemma:u-x} For $x_{t}\in\mathcal{E}_{t,1}$,
we have 
\[
\int_{x_{0}}p_{X_{0}}(x_{0})\big(2\pi(2\alpha_{t}-1-\overline{\alpha}_{t})\big)^{-d/2}\exp\Big(-\frac{\|u_{t}-\sqrt{\overline{\alpha}_{t}}x_{0}\|^{2}}{2(2\alpha_{t}-1-\overline{\alpha}_{t})}\Big)\mathrm{d}x_{0}\leq20\det\Big(I-\frac{1-\alpha_{t}}{1-\overline{\alpha}_{t}}J_{t}(x_{t})\Big)^{-1}p_{X_{t}}(x_{t}).
\]
\end{lemma}\begin{proof}See Appendix~\ref{sec:proof-lemma-u-x}.\end{proof}

\begin{lemma}\label{lemma:delta} For the function $\delta_{t-1}(\cdot)$
defined in (\ref{eq:delta-defn}), we have 
\[
\int_{x_{t-1}}\delta_{t-1}(x_{t-1})\mathrm{d}x_{t-1}\leq T^{-4}.
\]
\end{lemma}\begin{proof}See Appendix~\ref{sec:proof-lemma-delta}.\end{proof}

Equipped with these two lemmas, we can continue the derivation in
(\ref{eq:proof-main-5}) as follows: 
\begin{align*}
 & \int_{x_{t-1}}\Delta_{t-1}(x_{t-1})\mathrm{d}x_{t-1}\\
 & \qquad\overset{\text{(a)}}{\leq}\int_{x_{t}}\Delta_{t}(x_{t})\mathrm{d}x_{t}+20\int_{x_{t}\in\mathcal{E}_{t,1}}\bigg(|\xi_{t}(x_{t})|+O\Big(\Big(\frac{1-\alpha_{t}}{1-\overline{\alpha}_{t}}\Big)^{2}\big(d\log T+\|J_{t}(x_{t})\|_{\mathsf{F}}^{2}\big)\Big)\bigg)\\
 & \qquad\qquad\qquad\qquad\qquad\qquad\qquad\qquad\cdot\det\Big(I-\frac{1-\alpha_{t}}{1-\overline{\alpha}_{t}}J_{t}(x_{t})\Big)^{-1}p_{X_{t}}(x_{t})\mathrm{d}u_{t}+T^{-4}\\
 & \qquad\overset{\text{(b)}}{=}\int_{x_{t}}\Delta_{t}(x_{t})\mathrm{d}x_{t}+T^{-4}+20\int_{x_{t}\in\mathcal{E}_{t,1}}\bigg(|\xi_{t}(x_{t})|+O\Big(\Big(\frac{1-\alpha_{t}}{1-\overline{\alpha}_{t}}\Big)^{2}\big(d\log T+\|J_{t}(x_{t})\|_{\mathsf{F}}^{2}\big)\Big)\bigg)p_{X_{t}}(x_{t})\mathrm{d}x_{t}\\
 & \qquad\overset{\text{(c)}}{\leq}\int_{x_{t}}\Delta_{t}(x_{t})\mathrm{d}x_{t}+T^{-3}+C_{4}\Big(\frac{1-\alpha_{t}}{1-\overline{\alpha}_{t}}\Big)^{2}\int_{x_{t}\in\mathcal{E}_{t,1}}\big(d\log T+\|J_{t}(x_{t})\|_{\mathsf{F}}^{2}\big)p_{X_{t}}(x_{t})\mathrm{d}x_{t},
\end{align*}
which establishes the desired recursive relation. Here step (a) follows
from Lemmas~\ref{lemma:u-x}~and~\ref{lemma:delta}; step (b) follows
from $u_{t}=x_{t}+(1-\alpha_{t})s_{t}^{\star}(x_{t})$, hence 
\[
\mathrm{d}u_{t}=\mathsf{det}\Big(I-\frac{1-\alpha_{t}}{1-\overline{\alpha}_{t}}J_{t}(x_{t})\Big)\mathrm{d}x_{t};
\]
whereas step (c) uses (\ref{eq:lemma-det-2}) in Lemma~\ref{lemma:det},
and holds provided that $C_{4}\gg C_{3}$ is sufficiently large.

Finally, we control the error $\int\Delta_{T}(x)\mathrm{d}x$ in the
initial step of the reverse process. Notice that 
\begin{align}
\int\Delta_{T}(x)\mathrm{d}x & =\int_{x_{T}\neq\infty}\big(p_{X_{T}}(x_{T})-p_{\overline{Y}_{T}}(x_{T})\big)\mathrm{d}x_{T}\overset{\text{(i)}}{=}\mathsf{TV}\big(p_{X_{T}},p_{\overline{Y}_{T}^{-}}\big)\nonumber \\
 & \overset{\text{(ii)}}{\leq}\mathsf{TV}\big(p_{X_{T}},p_{Y_{T}}\big)+\mathsf{TV}\big(p_{Y_{T}},p_{\overline{Y}_{T}^{-}}\big),\label{eq:proof-Delta-T-1}
\end{align}
where step (i) follows from (\ref{eq:transition-fact-1}) and step
(ii) utilizes the triangle inequality. The first term can be bounded
by Lemma~\ref{lemma:initialization-error}, so it boils down to bounding
the second. By definition of $\overline{Y}_{T}^{-}$ in (\ref{eq:transition-YTbarminus}),
we have 
\begin{align}
 & \mathsf{TV}\big(p_{Y_{T}},p_{\overline{Y}_{T}^{-}}\big)=\int_{y\in\mathcal{E}_{T,1}^{\mathrm{c}}}p_{Y_{T}}(y)\mathrm{d}y\nonumber \\
 & \qquad\overset{\text{(a)}}{=}\int p_{Y_{T}}(y)\ind\big\{-\log p_{X_{T}}(y)>C_{1}d\log T,\|y\|_{2}\leq\sqrt{\overline{\alpha}_{T}}T^{2c_{R}}+C_{2}\sqrt{d(1-\overline{\alpha}_{T})\log T}\big\}\mathrm{d}y\nonumber \\
 & \qquad\qquad+\int p_{Y_{T}}(y)\ind\big\{\|y\|_{2}>\sqrt{\overline{\alpha}_{T}}T^{2c_{R}}+C_{2}\sqrt{d(1-\overline{\alpha}_{T})\log T}\big\}\mathrm{d}y\nonumber \\
 & \qquad\overset{\text{(b)}}{\leq}\int p_{X_{T}}(y)\ind\big\{-\log p_{X_{T}}(y)>C_{1}d\log T,\|y\|_{2}\leq\sqrt{\overline{\alpha}_{T}}T^{2c_{R}}+C_{2}\sqrt{d(1-\overline{\alpha}_{T})\log T}\big\}\mathrm{d}y\nonumber \\
 & \qquad\qquad+\mathsf{TV}\big(p_{X_{T}},p_{Y_{T}}\big)+\mathbb{P}\big(\Vert Y_{T}\Vert_{2}>\sqrt{\overline{\alpha}_{T}}T^{2c_{R}}+C_{2}\sqrt{d(1-\overline{\alpha}_{T})\log T}\big)\nonumber \\
 & \qquad\overset{\text{(c)}}{\leq}\big[2\sqrt{\overline{\alpha}_{T}}T^{2c_{R}}+2C_{2}\sqrt{d(1-\overline{\alpha}_{T})\log T}\big]^{d}\exp(-C_{1}d\log T)+\mathbb{P}\big(\Vert Y_{T}\Vert_{2}>\frac{C_{2}}{2}\sqrt{d\log T}\big)+\mathsf{TV}\big(p_{X_{T}},p_{Y_{T}}\big)\nonumber \\
 & \qquad\overset{\text{(d)}}{\leq}\exp\big(-\frac{C_{1}}{2}d\log T\big)+\mathbb{P}\big(\Vert Y_{T}\Vert_{2}>\frac{C_{2}}{2}\sqrt{d\log T}\big)+\mathsf{TV}\big(p_{X_{T}},p_{Y_{T}}\big).\label{eq:proof-Delta-T-2}
\end{align}
Here step (a) follows from the definition of $\mathcal{E}_{T,1}$
in (\ref{eq:defn-E-t-1}); step (b) follows from the definition of
total variation distance, i.e., $\mathsf{TV}(p,q)=\sup_{B}|p(B)-q(B)|$,
where the supremum is taken over all Borel set $B$ in $\mathbb{R}^{d}$;
step (c) holds since $\overline{\alpha}_{T}\leq T^{-c_{1}/2}$ (see
Lemma~\ref{lemma:step-size}), provided that $C_{2}$ is sufficiently
large; whereas step (d) holds provided that $C_{1}\gg c_{R}$ and
$T\gg d\log T$. By putting (\ref{eq:proof-Delta-T-1}) and (\ref{eq:proof-Delta-T-2})
together, we have 
\[
\int\Delta_{T}(x)\mathrm{d}x\leq2\mathsf{TV}\big(p_{X_{T}},p_{Y_{T}}\big)+\exp\big(-\frac{C_{1}}{2}d\log T\big)+\mathbb{P}\big(\Vert Y_{T}\Vert_{2}>\frac{C_{2}}{2}\sqrt{d\log T}\big)\leq T^{-4},
\]
where the last relation follows from Lemmas~\ref{lemma:initialization-error}~and~\ref{lemma:concentration},
provided that $C_{1},C_{2}>0$ are both sufficiently large.

\subsection{Proof of Lemma~\ref{lemma:det} \protect\protect\protect\label{sec:proof-lemma-det}}

Consider any $x_{t}\in\mathcal{E}_{t,1}$. Recall the definition $u_{t}=x_{t}+(1-\alpha_{t})s_{t}^{\star}(x_{t})$,
and we decompose 
\begin{align*}
 & \frac{\|u_{t}-\sqrt{\overline{\alpha}_{t}}x_{0}\|_{2}^{2}}{2(2\alpha_{t}-1-\overline{\alpha}_{t})}\\
 & \quad=\frac{\|x_{t}-\sqrt{\overline{\alpha}_{t}}x_{0}\|_{2}^{2}}{2(1-\overline{\alpha}_{t})}+\frac{(1-\alpha_{t})\|x_{t}-\sqrt{\overline{\alpha}_{t}}x_{0}\|_{2}^{2}}{(2\alpha_{t}-1-\overline{\alpha}_{t})(1-\overline{\alpha}_{t})}+\frac{(1-\alpha_{t})s_{t}^{\star}(x_{t})^{\top}(x_{t}-\sqrt{\overline{\alpha}_{t}}x_{0})}{2\alpha_{t}-1-\overline{\alpha}_{t}}+\frac{(1-\alpha_{t})^{2}\|s_{t}^{\star}(x_{t})\|_{2}^{2}}{2(2\alpha_{t}-1-\overline{\alpha}_{t})}\\
 & \quad=\frac{\|x_{t}-\sqrt{\overline{\alpha}_{t}}x_{0}\|_{2}^{2}}{2(1-\overline{\alpha}_{t})}+\frac{1-\alpha_{t}}{(2\alpha_{t}-1-\overline{\alpha}_{t})(1-\overline{\alpha}_{t})}\int_{x_{0}}p_{X_{0}|X_{t}}(x_{0}\mymid x_{t})\|x_{t}-\sqrt{\overline{\alpha}_{t}}x_{0}\|_{2}^{2}\mathrm{d}x_{0}\\
 & \quad\qquad+\frac{1-\alpha_{t}}{2\alpha_{t}-1-\overline{\alpha}_{t}}s_{t}^{\star}(x_{t})^{\top}\int_{x_{0}}p_{X_{0}|X_{t}}(x_{0}\mymid x_{t})\big(x_{t}-\sqrt{\overline{\alpha}_{t}}x_{0}\big)\mathrm{d}x_{0}+\frac{(1-\alpha_{t})^{2}\|s_{t}^{\star}(x_{t})\|_{2}^{2}}{2(2\alpha_{t}-1-\overline{\alpha}_{t})}+\zeta_{t}(x_{t},x_{0}),
\end{align*}
where we let 
\begin{align}
\zeta_{t}(x_{t},x_{0}) & \coloneqq\frac{(1-\alpha_{t})\big(\|x_{t}-\sqrt{\overline{\alpha}_{t}}x_{0}\|_{2}^{2}-\int_{x_{0}}p_{X_{0}|X_{t}}(x_{0}\mymid x_{t})\|x_{t}-\sqrt{\overline{\alpha}_{t}}x_{0}\|_{2}^{2}\mathrm{d}x_{0}\big)}{(2\alpha_{t}-1-\overline{\alpha}_{t})(1-\overline{\alpha}_{t})}\nonumber \\
 & \qquad\qquad+\frac{(1-\alpha_{t})s_{t}^{\star}(x_{t})^{\top}\big[(x_{t}-\sqrt{\overline{\alpha}_{t}}x_{0})-\int_{x_{0}}p_{X_{0}|X_{t}}(x_{0}\mymid x_{t})\big(x_{t}-\sqrt{\overline{\alpha}_{t}}x_{0}\big)\mathrm{d}x_{0}\big]}{2\alpha_{t}-1-\overline{\alpha}_{t}}.\label{eq:zeta-defn}
\end{align}
In view of (\ref{eq:proof-lemma-1-3}) and (\ref{eq:proof-lemma-1-4}),
we can further derive 
\begin{align}
 & \frac{\|u_{t}-\sqrt{\overline{\alpha}_{t}}x_{0}\|_{2}^{2}}{2(2\alpha_{t}-1-\overline{\alpha}_{t})}=\frac{\|x_{t}-\sqrt{\overline{\alpha}_{t}}x_{0}\|_{2}^{2}}{2(1-\overline{\alpha}_{t})}+\frac{1-\alpha_{t}}{2\alpha_{t}-1-\overline{\alpha}_{t}}\mathsf{Tr}\left(I-J_{t}(x_{t})\right)+\frac{(1-\alpha_{t})^{2}\|s_{t}^{\star}(x_{t})\|_{2}^{2}}{2(2\alpha_{t}-1-\overline{\alpha}_{t})}+\zeta_{t}(x_{t},x_{0})\nonumber \\
 & \qquad\overset{\text{(i)}}{=}\frac{\|x_{t}-\sqrt{\overline{\alpha}_{t}}x_{0}\|_{2}^{2}}{2(1-\overline{\alpha}_{t})}+\left(1+O\Big(\frac{1-\alpha_{t}}{1-\overline{\alpha}_{t}}\Big)\right)\left(\frac{1-\alpha_{t}}{1-\overline{\alpha}_{t}}\mathsf{Tr}\left(I-J_{t}(x_{t})\right)+\frac{(1-\alpha_{t})^{2}\|s_{t}^{\star}(x_{t})\|_{2}^{2}}{2(1-\overline{\alpha}_{t})}\right)+\zeta_{t}(x_{t},x_{0})\nonumber \\
 & \qquad\overset{\text{(ii)}}{=}\frac{\|x_{t}-\sqrt{\overline{\alpha}_{t}}x_{0}\|_{2}^{2}}{2(1-\overline{\alpha}_{t})}+\frac{1-\alpha_{t}}{1-\overline{\alpha}_{t}}\mathsf{Tr}\left(I-J_{t}(x_{t})\right)+O\bigg(\Big(\frac{1-\alpha_{t}}{1-\overline{\alpha}_{t}}\Big)^{2}d\log T\bigg)+\zeta_{t}(x_{t},x_{0})\nonumber \\
 & \qquad\overset{\text{(iii)}}{=}\frac{\|x_{t}-\sqrt{\overline{\alpha}_{t}}x_{0}\|_{2}^{2}}{2(1-\overline{\alpha}_{t})}+\log\det\Big(I-\frac{1-\alpha_{t}}{1-\overline{\alpha}_{t}}J_{t}(x_{t})\Big)-\frac{d}{2}\log\frac{2\alpha_{t}-1-\overline{\alpha}_{t}}{1-\overline{\alpha}_{t}}\nonumber \\
 & \qquad\qquad+\zeta_{t}(x_{t},x_{0})+O\bigg(\Big(\frac{1-\alpha_{t}}{1-\overline{\alpha}_{t}}\Big)^{2}\big(d\log T+\|J_{t}(x_{t})\|_{\mathsf{F}}^{2}\big)\bigg).\label{eq:proof-lemma-2-2}
\end{align}
Here, step (i) utilizes an immediate consequence of Lemma~\ref{lemma:step-size}
\begin{equation}
\frac{1-\overline{\alpha}_{t}}{2\alpha_{t}-1-\overline{\alpha}_{t}}=1+\frac{2(1-\alpha_{t})/(1-\overline{\alpha}_{t})}{1-2(1-\alpha_{t})/(1-\overline{\alpha}_{t})}=1+O\left(\frac{1-\alpha_{t}}{1-\overline{\alpha}_{t}}\right)=1+O\left(\frac{\log T}{T}\right),\label{eq:proof-lemma-2-2.3}
\end{equation}
which holds provided that $T\gg c_{1}\log T$; step (ii) follows from
$x_{t}\in\mathcal{E}_{t,1}$ and Lemma~\ref{lem:typical}; whereas
step (iii) follows from the following two facts: 
\begin{align*}
\log\mathsf{det}\Big(I-\frac{1-\alpha_{t}}{1-\overline{\alpha}_{t}}J_{t}(x_{t})\Big) & =-\frac{1-\alpha_{t}}{1-\overline{\alpha}_{t}}\mathsf{Tr}\big(J_{t}(x_{t})\big)+O\bigg(\Big(\frac{1-\alpha_{t}}{1-\overline{\alpha}_{t}}\Big)^{2}\|J_{t}(x_{t})\|_{\mathsf{F}}^{2}\bigg),
\end{align*}
and 
\begin{equation}
\frac{d}{2}\log\frac{2\alpha_{t}-1-\overline{\alpha}_{t}}{1-\overline{\alpha}_{t}}=\frac{d(1-\alpha_{t})}{1-\overline{\alpha}_{t}}+O\bigg(\frac{d(1-\alpha_{t})^{2}}{(1-\overline{\alpha}_{t})^{2}}\bigg)=O\left(\frac{d\log T}{T}\right).\label{eq:proof-lemma-2-2.6}
\end{equation}
Then we can use (\ref{eq:proof-lemma-2-2}) to achieve 
\begin{align*}
 & \int_{x_{0}}p_{X_{0}}(x_{0})\exp\Big(-\frac{\|u_{t}-\sqrt{\overline{\alpha}_{t}}x_{0}\|_{2}^{2}}{2(2\alpha_{t}-1-\overline{\alpha}_{t})}\Big)\mathrm{d}x_{0}=\int_{x_{0}}p_{X_{0}}(x_{0})\exp\Big(-\frac{\|x_{t}-\sqrt{\overline{\alpha}_{t}}x_{0}\|_{2}^{2}}{2(1-\overline{\alpha}_{t})}-\zeta_{t}(x_{t},x_{0})\Big)\mathrm{d}x_{0}\\
 & \qquad\cdot\exp\bigg(-\log\det\Big(I-\frac{1-\alpha_{t}}{1-\overline{\alpha}_{t}}J_{t}(x_{t})\Big)+\frac{d}{2}\log\frac{2\alpha_{t}-1-\overline{\alpha}_{t}}{1-\overline{\alpha}_{t}}+O\bigg(\Big(\frac{1-\alpha_{t}}{1-\overline{\alpha}_{t}}\Big)^{2}\big(d\log T+\|J_{t}(x_{t})\|_{\mathsf{F}}^{2}\big)\bigg)\bigg).
\end{align*}
Define a function $\xi_{t}(\cdot)$ as follows 
\begin{equation}
\xi_{t}(x_{t}):=-\log\frac{\int_{x_{0}}p_{X_{0}}(x_{0})\exp\big(-\frac{\|x_{t}-\sqrt{\overline{\alpha}_{t}}x_{0}\|_{2}^{2}}{2(1-\overline{\alpha}_{t})}-\zeta_{t}(x_{t},x_{0})\big)\mathrm{d}x_{0}}{\int_{x_{0}}p_{X_{0}}(x_{0})\exp\big(-\frac{\|x_{t}-\sqrt{\overline{\alpha}_{t}}x_{0}\|_{2}^{2}}{2(1-\overline{\alpha}_{t})}\big)\mathrm{d}x_{0}}.\label{eq:defn-xi}
\end{equation}
Then we can write 
\begin{align}
 & \int_{x_{0}}p_{X_{0}}(x_{0})\exp\Big(-\frac{\|u_{t}-\sqrt{\overline{\alpha}_{t}}x_{0}\|_{2}^{2}}{2(2\alpha_{t}-1-\overline{\alpha}_{t})}\Big)\mathrm{d}x_{0}=\exp\bigg(-\xi_{t}(x_{t})+O\bigg(\Big(\frac{1-\alpha_{t}}{1-\overline{\alpha}_{t}}\Big)^{2}\big(d\log T+\|J_{t}(x_{t})\|_{\mathsf{F}}^{2}\big)\bigg)\bigg)\nonumber \\
 & \qquad\cdot\int_{x_{0}}p_{X_{0}}(x_{0})\exp\Big(-\frac{\|x_{t}-\sqrt{\overline{\alpha}_{t}}x_{0}\|_{2}^{2}}{2(1-\overline{\alpha}_{t})}-\log\det\Big(I-\frac{1-\alpha_{t}}{1-\overline{\alpha}_{t}}J_{t}(x_{t})\Big)+\frac{d}{2}\log\frac{2\alpha_{t}-1-\overline{\alpha}_{t}}{1-\overline{\alpha}_{t}}\Big)\mathrm{d}x_{0},\label{eq:proof-lemma-2-3}
\end{align}
and $\xi_{t}(x_{t})\leq0$ for any $x_{t}\in\mathcal{E}_{t,1}$ since
\begin{align*}
\exp(-\xi_{t}(x_{t})) & =\int_{x_{0}}p_{X_{0}|X_{t}}(x_{0}\mymid x_{t})\exp\big(-\zeta_{t}(x_{t},x_{0})\big)\mathrm{d}x_{0}\\
 & \ge1-\int_{x_{0}}p_{X_{0}|X_{t}}(x_{0}\mymid x_{t})\zeta_{t}(x_{t},x_{0})\mathrm{d}x_{0}=1,
\end{align*}
where we have used the fact that $e^{x}\geq1+x$ for any $x\in\mathbb{R}$.
Notice that 
\begin{equation}
p_{X_{t}}(x_{t})=\big(2\pi(1-\overline{\alpha}_{t})\big)^{-d/2}\int_{x_{0}}p_{X_{0}}(x_{0})\exp\Big(-\frac{\|x_{t}-\sqrt{\overline{\alpha}_{t}}x_{0}\|^{2}}{2(1-\overline{\alpha}_{t})}\Big)\mathrm{d}x_{0},\label{eq:proof-lemma-2-4}
\end{equation}
we can rearrange terms in (\ref{eq:proof-lemma-2-3}) to achieve 
\begin{align}
 & \det\Big(I-\frac{1-\alpha_{t}}{1-\overline{\alpha}_{t}}J_{t}(x_{t})\Big)^{-1}p_{X_{t}}(x_{t})\nonumber \\
 & \qquad=\big(2\pi(2\alpha_{t}-1-\overline{\alpha}_{t})\big)^{-d/2}\int_{x_{0}}p_{X_{0}}(x_{0})\exp\Big(-\frac{\|u_{t}-\sqrt{\overline{\alpha}_{t}}x_{0}\|^{2}}{2(2\alpha_{t}-1-\overline{\alpha}_{t})}\Big)\mathrm{d}x_{0}\nonumber \\
 & \qquad\qquad\qquad\cdot\exp\bigg(\xi_{t}(x_{t})+O\bigg(\Big(\frac{1-\alpha_{t}}{1-\overline{\alpha}_{t}}\Big)^{2}\big(d\log T+\|J_{t}(x_{t})\|_{\mathsf{F}}^{2}\big)\bigg)\bigg),\label{eq:proof-lemma-2-5}
\end{align}
which gives the desired decomposition (\ref{eq:lemma-det-1}).

To establish (\ref{eq:lemma-det-2}), we need the following result.

\begin{lemma} \label{lemma:small-prob}We have \begin{subequations}\label{eq:proof-lemma-2-6}
\begin{align}
\int_{x_{0}}\int_{x_{t}\notin\mathcal{E}_{t,1}}\hspace{-1ex}(2\pi(2\alpha_{t}-1-\overline{\alpha}_{t}))^{-d/2}p_{X_{0}}(x_{0})\exp\Big(-\frac{\|u_{t}-\sqrt{\overline{\alpha}_{t}}x_{0}\|_{2}^{2}}{2(2\alpha_{t}-1-\overline{\alpha}_{t})}\Big)\mathrm{d}x_{0}\mathrm{d}u_{t} & \leq T^{-4}\label{eq:proof-lemma-2-6A}
\end{align}
and 
\begin{equation}
\int_{x_{t}\in\mathcal{E}_{t,1}^{\mathrm{c}}}p_{X_{t}}(x_{t})\mathrm{d}x_{t}\leq T^{-4}.\label{eq:proof-lemma-2-6B}
\end{equation}
\end{subequations}\end{lemma}\begin{proof}See Appendix \ref{subsec:proof-lemma-small-prob}.
\end{proof}

Then we have 
\begin{align*}
1 & \overset{\text{(i)}}{\geq}\int_{x_{t}\in\mathcal{E}_{t,1}}\int_{x_{0}}(2\pi(2\alpha_{t}-1-\overline{\alpha}_{t}))^{-d/2}p_{X_{0}}(x_{0})\exp\Big(-\frac{\|u_{t}-\sqrt{\overline{\alpha}_{t}}x_{0}\|_{2}^{2}}{2(2\alpha_{t}-1-\overline{\alpha}_{t})}\Big)\mathrm{d}x_{0}\mathrm{d}u_{t}\\
 & \overset{\text{(ii)}}{=}\int_{x_{t}\in\mathcal{E}_{t,1}}\det\Big(I-\frac{1-\alpha_{t}}{1-\overline{\alpha}_{t}}J_{t}(x_{t})\Big)^{-1}p_{X_{t}}(x_{t})\exp\bigg(-\xi_{t}(x_{t})+O\bigg(\Big(\frac{1-\alpha_{t}}{1-\overline{\alpha}_{t}}\Big)^{2}\big(d\log T+\|J_{t}(x_{t})\|_{\mathsf{F}}^{2}\big)\bigg)\bigg)\mathrm{d}u_{t}\\
 & \overset{\text{(iii)}}{=}\int_{x_{t}\in\mathcal{E}_{t,1}}p_{X_{t}}(x_{t})\exp\bigg(-\xi_{t}(x_{t})+O\bigg(\Big(\frac{1-\alpha_{t}}{1-\overline{\alpha}_{t}}\Big)^{2}\big(d\log T+\|J_{t}(x_{t})\|_{\mathsf{F}}^{2}\big)\bigg)\bigg)\mathrm{d}x_{t}\\
 & \overset{\text{(iv)}}{\geq}\int_{x_{t}\in\mathcal{E}_{t,1}}\bigg(1-\xi_{t}(x_{t})+O\bigg(\Big(\frac{1-\alpha_{t}}{1-\overline{\alpha}_{t}}\Big)^{2}\big(d\log T+\|J_{t}(x_{t})\|_{\mathsf{F}}^{2}\big)\bigg)\bigg)p_{X_{t}}(x_{t})\mathrm{d}x_{t}.
\end{align*}
Here step (i) follows from (\ref{eq:proof-lemma-2-6A}); step (ii)
utilizes (\ref{eq:proof-lemma-2-5}); step (iii) holds since $u_{t}=x_{t}+(1-\alpha_{t})s_{t}^{\star}(x_{t})$,
namely 
\[
\mathrm{d}u_{t}=\mathsf{det}\Big(I-\frac{1-\alpha_{t}}{1-\overline{\alpha}_{t}}J_{t}(x_{t})\Big)\mathrm{d}x_{t};
\]
while step (iv) follows from the fact that $e^{x}\geq1+x$ for any
$x\in\mathbb{R}$. Recall that $\xi_{t}(x_{t})\leq0$ for any $x_{t}\in\mathcal{E}_{t,1}$.
By rearranging terms, we have 
\begin{align*}
 & \int_{x_{t}\in\mathcal{E}_{t,1}}|\xi_{t}(x_{t})|p_{X_{t}}(x_{t})\mathrm{d}x_{t}\\
 & \qquad\leq\int_{x_{t}\in\mathcal{E}_{t,1}^{\mathrm{c}}}p_{X_{t}}(x_{t})\mathrm{d}x_{t}+C_{3}\Big(\frac{1-\alpha_{t}}{1-\overline{\alpha}_{t}}\Big)^{2}\int_{x_{t}\in\mathcal{E}_{t,1}}\big(d\log T+\|J_{t}(x_{t})\|_{\mathsf{F}}^{2}\big)p_{X_{t}}(x_{t})\mathrm{d}x_{t}\\
 & \qquad\leq C_{3}\Big(\frac{1-\alpha_{t}}{1-\overline{\alpha}_{t}}\Big)^{2}\int_{x_{t}\in\mathcal{E}_{t,1}}\big(d\log T+\|J_{t}(x_{t})\|_{\mathsf{F}}^{2}\big)p_{X_{t}}(x_{t})\mathrm{d}x_{t}+T^{-4}
\end{align*}
for some universal constant $C_{3}>0$, where the last step follows
from (\ref{eq:proof-lemma-2-6B}).

\subsection{Proof of Lemma~\ref{lemma:u-x}\protect\protect\protect\label{sec:proof-lemma-u-x}}

Recall the definition of $\zeta_{t}(x_{t},x_{0})$ from (\ref{eq:zeta-defn})
in Appendix~\ref{sec:proof-lemma-det}. For any $x_{t}\in\mathcal{E}_{t,1}$,
we have 
\begin{align}
-\zeta_{t}(x_{t},x_{0}) & \overset{\text{(i)}}{\leq}2\frac{1-\alpha_{t}}{(1-\overline{\alpha}_{t})^{2}}\int_{x_{0}}p_{X_{0}|X_{t}}(x_{0}\mymid x_{t})\|x_{t}-\sqrt{\overline{\alpha}_{t}}x_{0}\|_{2}^{2}\mathrm{d}x_{0}+2\frac{1-\alpha_{t}}{1-\overline{\alpha}_{t}}\big|s_{t}^{\star}(x_{t})^{\top}(x_{t}-\sqrt{\overline{\alpha}_{t}}x_{0})\big|\nonumber \\
 & \overset{\text{(ii)}}{\leq}4\frac{1-\alpha_{t}}{1-\overline{\alpha}_{t}}(6C_{1}+3c_{0})d\log T+(1-\alpha_{t})\Vert s_{t}^{\star}(x_{t})\Vert_{2}^{2}+\frac{1-\alpha_{t}}{(1-\overline{\alpha}_{t})^{2}}\Vert x_{t}-\sqrt{\overline{\alpha}_{t}}x_{0}\Vert_{2}^{2}\nonumber \\
 & \overset{\text{(iii)}}{\leq}50\frac{1-\alpha_{t}}{1-\overline{\alpha}_{t}}(C_{1}+c_{0})d\log T+\frac{1-\alpha_{t}}{(1-\overline{\alpha}_{t})^{2}}\Vert x_{t}-\sqrt{\overline{\alpha}_{t}}x_{0}\Vert_{2}^{2}\nonumber \\
 & \overset{\text{(iv)}}{\leq}1+\frac{1-\alpha_{t}}{(1-\overline{\alpha}_{t})^{2}}\Vert x_{t}-\sqrt{\overline{\alpha}_{t}}x_{0}\Vert_{2}^{2}.\label{eq:proof-lemma-3-1}
\end{align}
Here step (i) utilizes (\ref{eq:proof-lemma-1-3}), (\ref{eq:zeta-defn})
and (\ref{eq:proof-lemma-2-2.3}); step (ii) follows from the AM-GM
inequality and an intermediate step in (\ref{eq:proof-lemma-1-6}):
\[
\frac{1}{1-\overline{\alpha}_{t}}\int_{x_{0}}p_{X_{0}|X_{t}}(x_{0}\mymid x_{t})\Vert x_{t}-\sqrt{\overline{\alpha}_{t}}x_{0}\Vert_{2}^{2}\mathrm{d}x_{0}\leq2(6C_{1}+3c_{0})d\log T,
\]
where we also use the fact that $-\log p_{X_{t}}(x_{t})\le C_{1}d\log T$
for $x_{t}\in\mathcal{E}_{t,1}$; step (iii) follows from Lemma~\ref{lem:typical};
while step (iv) follows from Lemma~\ref{lemma:step-size} and holds
provided that $T\gg c_{1}(C_{1}+c_{0})$. In addition, we also have
\begin{align}
\|J_{t}(x_{t})\|_{\mathsf{F}}^{2} & \leq2\|I_{d}-J_{t}(x_{t})\|_{\mathsf{F}}^{2}+2\|I_{d}\|_{\mathsf{F}}^{2}\overset{\text{(a)}}{\leq}2\big[\mathsf{Tr}\big(I_{d}-J_{t}(x_{t})\big)\big]^{2}+2d\nonumber \\
 & \overset{\text{(b)}}{\leq}288(C_{1}+c_{0})^{2}d^{2}\log^{2}T+2d,\label{eq:proof-lemma-3-2}
\end{align}
for $x_{t}\in\mathcal{E}_{t,1}$, where step (a) holds since $I_{d}-J_{t}(x_{t})\succeq0$
and step (b) follows from Lemma~\ref{lem:typical}. Substituting
the bounds (\ref{eq:proof-lemma-3-1}), (\ref{eq:proof-lemma-3-2})
and (\ref{eq:proof-lemma-2-2.6}) into (\ref{eq:proof-lemma-2-2})
gives 
\begin{align}
-\frac{\|u_{t}-\sqrt{\overline{\alpha}_{t}}x_{0}\|_{2}^{2}}{2(2\alpha_{t}-1-\overline{\alpha}_{t})} & \leq-\frac{\|x_{t}-\sqrt{\overline{\alpha}_{t}}x_{0}\|_{2}^{2}}{2(1-\overline{\alpha}_{t})}-\log\det\Big(I-\frac{1-\alpha_{t}}{1-\overline{\alpha}_{t}}J_{t}(x_{t})\Big)+\frac{1-\alpha_{t}}{(1-\overline{\alpha}_{t})^{2}}\Vert x_{t}-\sqrt{\overline{\alpha}_{t}}x_{0}\Vert_{2}^{2}+2,\label{eq:proof-lemma-3-3}
\end{align}
provided that $T\gg c_{1}(C_{1}+c_{0})d\log^{2}T$. Taking (\ref{eq:proof-lemma-3-3})
and (\ref{eq:proof-lemma-2-2.6}) collectively yields 
\begin{align}
 & \det\Big(I-\frac{1-\alpha_{t}}{1-\overline{\alpha}_{t}}J_{t}(x_{t})\Big)\int_{x_{0}}p_{X_{0}}(x_{0})\big(2\pi(2\alpha_{t}-1-\overline{\alpha}_{t})\big)^{-d/2}\exp\Big(-\frac{\|u_{t}-\sqrt{\overline{\alpha}_{t}}x_{0}\|^{2}}{2(2\alpha_{t}-1-\overline{\alpha}_{t})}\Big)\mathrm{d}x_{0}\nonumber \\
 & \qquad\leq10\int_{x_{0}}p_{X_{0}}(x_{0})\big(2\pi(1-\overline{\alpha}_{t})\big)^{-d/2}\exp\Big(-\frac{\|x_{t}-\sqrt{\overline{\alpha}_{t}}x_{0}\|_{2}^{2}}{2(1-\overline{\alpha}_{t})}+\frac{1-\alpha_{t}}{(1-\overline{\alpha}_{t})^{2}}\Vert x_{t}-\sqrt{\overline{\alpha}_{t}}x_{0}\Vert_{2}^{2}\Big)\mathrm{d}x_{0}.\label{eq:proof-lemma-3-4}
\end{align}
provided that $T\gg d\log T$. To achieve the desired result, it suffices
to connect the above expression with 
\[
p_{X_{t}}(x_{t})=\int_{x_{0}}p_{X_{0}}(x_{0})\big(2\pi(1-\overline{\alpha}_{t})\big)^{-d/2}\exp\Big(-\frac{\|x_{t}-\sqrt{\overline{\alpha}_{t}}x_{0}\|_{2}^{2}}{2(1-\overline{\alpha}_{t})}\Big)\mathrm{d}x_{0}.
\]
For any $x_{t}\in\mathcal{E}_{t,1}$, define a set 
\[
\mathcal{A}(x_{t})\coloneqq\Big\{ x_{0}:\frac{1-\alpha_{t}}{(1-\overline{\alpha}_{t})^{2}}\Vert x_{t}-\sqrt{\overline{\alpha}_{t}}x_{0}\Vert_{2}^{2}>(6C_{1}+3c_{0})\frac{1-\alpha_{t}}{1-\overline{\alpha}_{t}}d\log T\Big\}.
\]
We have 
\begin{align}
 & \int_{x_{0}\in\mathcal{A}(x_{t})}p_{X_{0}}(x_{0})\big(2\pi(1-\overline{\alpha}_{t})\big)^{-d/2}\exp\Big(-\frac{\|x_{t}-\sqrt{\overline{\alpha}_{t}}x_{0}\|_{2}^{2}}{2(1-\overline{\alpha}_{t})}+\frac{1-\alpha_{t}}{(1-\overline{\alpha}_{t})^{2}}\Vert x_{t}-\sqrt{\overline{\alpha}_{t}}x_{0}\Vert_{2}^{2}\Big)\mathrm{d}x_{0}\nonumber \\
 & \qquad=p_{X_{t}}(x_{t})\int_{x_{0}\in\mathcal{A}(x_{t})}p_{X_{0}|X_{t}}(x_{0}\mymid x_{t})\exp\Big(\frac{1-\alpha_{t}}{(1-\overline{\alpha}_{t})^{2}}\Vert x_{t}-\sqrt{\overline{\alpha}_{t}}x_{0}\Vert_{2}^{2}\Big)\mathrm{d}x_{0}\nonumber \\
 & \qquad\overset{\text{(i)}}{\leq}p_{X_{t}}(x_{t})\int_{x_{0}\in\mathcal{A}(x_{t})}p_{X_{0}}(x_{0})\exp\Big(-\frac{\|x-\sqrt{\overline{\alpha}_{t}}x_{0}\|_{2}^{2}}{3(1-\overline{\alpha}_{t})}+\frac{1-\alpha_{t}}{(1-\overline{\alpha}_{t})^{2}}\Vert x_{t}-\sqrt{\overline{\alpha}_{t}}x_{0}\Vert_{2}^{2}\Big)\mathrm{d}x_{0}\nonumber \\
 & \qquad\overset{\text{(ii)}}{\leq}p_{X_{t}}(x_{t})\int_{x_{0}\in\mathcal{A}(x_{t})}p_{X_{0}}(x_{0})\exp\Big(-\frac{\|x-\sqrt{\overline{\alpha}_{t}}x_{0}\|_{2}^{2}}{4(1-\overline{\alpha}_{t})}\Big)\mathrm{d}x_{0}\nonumber \\
 & \qquad\overset{\text{(iii)}}{\leq}p_{X_{t}}(x_{t})\exp\Big(-\frac{(6C_{1}+3c_{0})d\log T}{4}\Big)\int_{x_{0}\in\mathcal{A}(x_{t})}p_{X_{0}}(x_{0})\mathrm{d}x_{0}\overset{\text{(iv)}}{\leq}\frac{1}{2}p_{X_{t}}(x_{t}).\label{eq:proof-lemma-3-5}
\end{align}
Here step (i) follows from (\ref{eq:proof-lemma-1-2}); step (ii)
utilizes Lemma~\ref{lemma:step-size} and holds provided that $T\gg c_{1}\log T$;
step (iii) follows from the definition of $\mathcal{A}(x_{t})$; while
step (iv) holds provided that $C_{1}$ is sufficiently large. On the
other hand, we have 
\begin{align}
 & \int_{x_{0}\in\mathcal{A}(x_{t})^{\mathrm{c}}}p_{X_{0}}(x_{0})\big(2\pi(1-\overline{\alpha}_{t})\big)^{-d/2}\exp\Big(-\frac{\|x_{t}-\sqrt{\overline{\alpha}_{t}}x_{0}\|_{2}^{2}}{2(1-\overline{\alpha}_{t})}+\frac{1-\alpha_{t}}{(1-\overline{\alpha}_{t})^{2}}\Vert x_{t}-\sqrt{\overline{\alpha}_{t}}x_{0}\Vert_{2}^{2}\Big)\mathrm{d}x_{0}\nonumber \\
 & \qquad\overset{\text{(a})}{\leq}\exp\Big((6C_{1}+3c_{0})\frac{1-\alpha_{t}}{1-\overline{\alpha}_{t}}d\log T\Big)\int_{x_{0}}p_{X_{0}}(x_{0})\big(2\pi(1-\overline{\alpha}_{t})\big)^{-d/2}\exp\Big(-\frac{\|x_{t}-\sqrt{\overline{\alpha}_{t}}x_{0}\|_{2}^{2}}{2(1-\overline{\alpha}_{t})}\Big)\mathrm{d}x_{0}\nonumber \\
 & \qquad\overset{\text{(b)}}{\leq}\exp\Big((6C_{1}+3c_{0})\frac{8c_{1}d\log^{2}T}{T}\Big)p_{X_{t}}(x_{t})\overset{\text{(c)}}{\leq}\frac{3}{2}p_{X_{t}}(x_{t}).\label{eq:proof-lemma-3-6}
\end{align}
Here step (a) follows from the definition of $\mathcal{A}(x_{t})$;
step (b) utilizes Lemma~\ref{lemma:step-size}; whereas step (c)
holds provided that $T\gg c_{1}(C_{1}+c_{0})d\log^{2}T$. Taking (\ref{eq:proof-lemma-3-4}),
(\ref{eq:proof-lemma-3-5}) and (\ref{eq:proof-lemma-3-6}) collectively
gives 
\[
\det\Big(I-\frac{1-\alpha_{t}}{1-\overline{\alpha}_{t}}J_{t}(x_{t})\Big)\int_{x_{0}}p_{X_{0}}(x_{0})\big(2\pi(2\alpha_{t}-1-\overline{\alpha}_{t})\big)^{-d/2}\exp\Big(-\frac{\|u_{t}-\sqrt{\overline{\alpha}_{t}}x_{0}\|^{2}}{2(2\alpha_{t}-1-\overline{\alpha}_{t})}\Big)\mathrm{d}x_{0}\leq20p_{X_{t}}(x_{t}).
\]
Rearrange terms to achieve the desired result.

\subsection{Proof of Lemma~\ref{lemma:delta}\protect\protect\label{sec:proof-lemma-delta}}

By definition of $\delta_{t-1}(x_{t-1})$ in (\ref{eq:delta-defn}),
we have 
\begin{align}
\int_{x_{t-1}}\delta_{t-1}(x_{t-1})\mathrm{d}x_{t-1} & =\int_{x_{0}}\int_{x_{t-1}}\int_{x_{t}\notin\mathcal{E}_{t,1}}p_{X_{0}}(x_{0})\Big(\frac{\alpha_{t}}{4\pi^{2}(1-\alpha_{t})(2\alpha_{t}-1-\overline{\alpha}_{t})}\Big)^{d/2}\nonumber \\
 & \qquad\qquad\cdot\exp\Big(-\frac{\|u_{t}-\sqrt{\overline{\alpha}_{t}}x_{0}\|_{2}^{2}}{2(2\alpha_{t}-1-\overline{\alpha}_{t})}\Big)\exp\Big(-\frac{\big\|\sqrt{\alpha_{t}}x_{t-1}-u_{t}\big\|_{2}^{2}}{2(1-\alpha_{t})}\Big)\mathrm{d}x_{t-1}\mathrm{d}u_{t}\mathrm{d}x_{0}\nonumber \\
 & \overset{\text{(i)}}{=}\int_{x_{0}}\int_{x_{t}\notin\mathcal{E}_{t,1}}(2\pi(2\alpha_{t}-1-\overline{\alpha}_{t}))^{-d/2}p_{X_{0}}(x_{0})\exp\Big(-\frac{\|u_{t}-\sqrt{\overline{\alpha}_{t}}x_{0}\|_{2}^{2}}{2(2\alpha_{t}-1-\overline{\alpha}_{t})}\Big)\mathrm{d}x_{0}\mathrm{d}u_{t}\nonumber \\
 & \overset{\text{(ii)}}{\leq}T^{-4}.\label{eq:proof-lemma-delta-1}
\end{align}
Here step (i) holds since for fixed $u_{t}$, the following function
\[
\Big(2\pi\frac{1-\alpha_{t}}{\alpha_{t}}\Big)^{-d/2}\exp\Big(-\frac{\big\|\sqrt{\alpha_{t}}x_{t-1}-u_{t}\big\|_{2}^{2}}{2(1-\alpha_{t})}\Big)
\]
is a density function w.r.t.~$x_{t-1}$, while step (ii) was established
in (\ref{eq:proof-lemma-2-6A}). 

\subsection{Proof of Lemma~\ref{lemma:small-prob} \protect\label{subsec:proof-lemma-small-prob}}

\paragraph{Proof of (\ref{eq:proof-lemma-2-6}). }

We first prove (\ref{eq:proof-lemma-2-6B}). Recall that 
\begin{align*}
\mathcal{E}_{t,1} & =\big\{ x_{t}:-\log p_{X_{t}}(x_{t})\leq C_{1}d\log T,\|x_{t}\|_{2}\leq\sqrt{\overline{\alpha}_{t}}T^{2c_{R}}+C_{2}\sqrt{d(1-\overline{\alpha}_{t})\log T}\big\}.
\end{align*}
Then we can decompose 
\begin{align*}
\int_{x_{t}\in\mathcal{E}_{t,1}^{\mathrm{c}}}p_{X_{t}}(x_{t})\mathrm{d}x_{t} & =\int p_{X_{t}}(x_{t})\ind\big\{-\log p_{X_{t}}(x_{t})>C_{1}d\log T,\|x_{t}\|_{2}\leq\sqrt{\overline{\alpha}_{t}}T^{2c_{R}}+C_{2}\sqrt{d(1-\overline{\alpha}_{t})\log T}\big\}\mathrm{d}x_{t}\\
 & \qquad+\int p_{X_{t}}(x_{t})\ind\big\{\|x_{t}\|_{2}>\sqrt{\overline{\alpha}_{t}}T^{2c_{R}}+C_{2}\sqrt{d(1-\overline{\alpha}_{t})\log T}\big\}\mathrm{d}x_{t}\\
 & \overset{\text{(i)}}{\leq}\exp\Big(-\frac{C_{1}}{2}d\log T\Big)+\mathbb{P}\big(\|X_{t}\|_{2}>\sqrt{\overline{\alpha}_{t}}T^{2c_{R}}+C_{2}\sqrt{d(1-\overline{\alpha}_{t})\log T}\big)\\
 & \overset{\text{(ii)}}{\leq}\exp\Big(-\frac{C_{1}}{2}d\log T\Big)+\mathbb{P}\big(\Vert X_{0}\Vert_{2}>T^{2c_{R}}\big)+\mathbb{P}\big(\Vert\overline{W}_{t}\Vert_{2}>C_{2}\sqrt{d\log T}\big)\\
 & \overset{\text{(iii)}}{\leq}\exp\Big(-\frac{C_{1}}{2}d\log T\Big)+\frac{\mathbb{E}[\Vert X_{0}\Vert_{2}]}{T^{2c_{R}}}+\mathbb{P}\big(\Vert\overline{W}_{t}\Vert_{2}>C_{2}\sqrt{d\log T}\big)\overset{\text{(iv)}}{\leq}T^{-4}.
\end{align*}
Here step (i) follows from a simple volume argument 
\begin{align*}
 & \int p_{X_{t}}(x_{t})\ind\big\{-\log p_{X_{t}}(x_{t})>C_{1}d\log T,\|x_{t}\|_{2}\leq\sqrt{\overline{\alpha}_{t}}T^{2c_{R}}+C_{2}\sqrt{d(1-\overline{\alpha}_{t})\log T}\big\}\mathrm{d}x_{t}\\
 & \qquad\leq\big(2\sqrt{\overline{\alpha}_{t}}T^{2c_{R}}+2C_{2}\sqrt{d(1-\overline{\alpha}_{t})\log T}\big)^{d}\exp\left(-C_{1}d\log T\right)\leq\exp\Big(-\frac{C_{1}}{2}d\log T\Big),
\end{align*}
provided that $C_{1}\gg c_{R}$ and $T\gg d\log T$; step (ii) follows
from $X_{t}=\sqrt{\overline{\alpha}_{t}}X_{0}+\sqrt{1-\overline{\alpha}_{t}}\,\overline{W}_{t}$;
step (iii) utilizes Markov's inequality; while step (iv) holds provided
that $C_{1},C_{2},c_{R}>0$ are large enough. This establishes (\ref{eq:proof-lemma-2-6B}).

Then we prove (\ref{eq:proof-lemma-2-6A}). Define 
\[
\mathcal{B}_{t}:=\big\{ x:\|x\|_{2}\leq\sqrt{\overline{\alpha}_{t}}T^{2c_{R}}+C_{2}\sqrt{d(2\alpha_{t}-1-\overline{\alpha}_{t})\log T}\big\},
\]
and for each $k\geq1$, 
\[
\mathcal{L}_{t,k}:=\big\{ x_{t}:2^{k-1}C_{1}d\log T<-\log p_{X_{t}}(x_{t})\leq2^{k}C_{1}d\log T\big\}.
\]
We first decompose 
\begin{align*}
I & \coloneqq\int_{x_{0}}\int_{x_{t}\notin\mathcal{E}_{t,1}}(2\pi(2\alpha_{t}-1-\overline{\alpha}_{t}))^{-d/2}p_{X_{0}}(x_{0})\exp\Big(-\frac{\|u_{t}-\sqrt{\overline{\alpha}_{t}}x_{0}\|_{2}^{2}}{2(2\alpha_{t}-1-\overline{\alpha}_{t})}\Big)\mathrm{d}x_{0}\mathrm{d}u_{t}\\
 & \overset{\text{(a)}}{\leq}\underbrace{\int_{x_{0}}\int_{u_{t}\notin\mathcal{B}_{t}}p_{X_{0}}(x_{0})\big(2\pi(2\alpha_{t}-1-\overline{\alpha}_{t})\big)^{-d/2}\exp\Big(-\frac{\|u_{t}-\sqrt{\overline{\alpha}_{t}}x_{0}\|_{2}^{2}}{2(2\alpha_{t}-1-\overline{\alpha}_{t})}\Big)\mathrm{d}u_{t}\mathrm{d}x_{0}}_{\eqqcolon I_{0}}\\
 & \qquad+\sum_{k=1}^{\infty}\underbrace{\int_{x_{0}}\int_{x_{t}\in\mathcal{L}_{t,k},u_{t}\in\mathcal{B}_{t}}p_{X_{0}}(x_{0})\big(2\pi(2\alpha_{t}-1-\overline{\alpha}_{t})\big)^{-d/2}\exp\Big(-\frac{\|u_{t}-\sqrt{\overline{\alpha}_{t}}x_{0}\|_{2}^{2}}{2(2\alpha_{t}-1-\overline{\alpha}_{t})}\Big)\mathrm{d}x_{0}\mathrm{d}u_{t}}_{\eqqcolon I_{k}},
\end{align*}
where step (a) holds since $\mathcal{E}_{t,1}^{\mathrm{c}}=\cup_{k=1}^{\infty}\mathcal{L}_{t,k}$.
The first term $I_{0}$ can be upper bounded as follows: 
\begin{align}
I_{0} & \le\Big(\int_{\|x_{0}\|_{2}\geq T^{2c_{R}}}\int_{u_{t}}+\int_{\|u_{t}-\sqrt{\overline{\alpha}_{t}}x_{0}\|_{2}\geq C_{2}\sqrt{d(2\alpha_{t}-1-\overline{\alpha}_{t})\log T}}\int_{x_{0}}\Big)p_{X_{0}}(x_{0})\nonumber \\
 & \qquad\qquad\cdot\Big(\frac{1}{2\pi(2\alpha_{t}-1-\overline{\alpha}_{t})}\Big)^{d/2}\exp\Big(-\frac{\|u_{t}-\sqrt{\overline{\alpha}_{t}}x_{0}\|_{2}^{2}}{2(2\alpha_{t}-1-\overline{\alpha}_{t})}\Big)\mathrm{d}u_{t}\mathrm{d}x_{0}\nonumber \\
 & \overset{\text{(i)}}{\leq}\mathbb{P}\left(\Vert X_{0}\Vert_{2}\geq T^{2c_{R}}\right)+\mathbb{P}\left(\|Z\|_{2}\geq C_{2}\sqrt{d\log T}\right)\nonumber \\
 & \overset{\text{(ii)}}{\leq}\frac{\mathbb{E}[\Vert X_{0}\Vert_{2}]}{T^{2c_{R}}}+\mathbb{P}\left(\|Z\|_{2}\geq C_{2}\sqrt{d\log T}\right)\overset{\text{(iii)}}{\leq}T^{-5}.\label{eq:proof-eq-1}
\end{align}
Here step (i) holds since 
\[
(2\pi(2\alpha_{t}-1-\overline{\alpha}_{t}))^{-d/2}p_{X_{0}}(x_{0})\exp\Big(-\frac{\|u_{t}-\sqrt{\overline{\alpha}_{t}}x_{0}\|_{2}^{2}}{2(2\alpha_{t}-1-\overline{\alpha}_{t})}\Big)
\]
is the joint density of $(X_{0},\sqrt{\overline{\alpha}_{t}}X_{0}+\sqrt{2\alpha_{t}-1-\overline{\alpha}_{t}}Z)$
where $Z\sim\mathcal{N}(0,I_{d})$ is independent of $X_{0}$; step
(ii) follows from Markov's inequality; whereas step (iii) holds provided
that $c_{R}$ and $C_{2}$ are sufficiently large. Regarding $I_{k}$,
we first show that 
\begin{align}
-\frac{\|u_{t}-\sqrt{\overline{\alpha}_{t}}x_{0}\|_{2}^{2}}{2(2\alpha_{t}-1-\overline{\alpha}_{t})} & \overset{\text{(a)}}{\leq}-\frac{(\|x_{t}-\sqrt{\overline{\alpha}_{t}}x_{0}\|_{2}-(1-\alpha_{t})\|s_{t}^{\star}(x_{t})\|_{2})^{2}}{2(2\alpha_{t}-1-\overline{\alpha}_{t})}\nonumber \\
 & \leq-\frac{\|x_{t}-\sqrt{\overline{\alpha}_{t}}x_{0}\|_{2}^{2}}{2(2\alpha_{t}-1-\overline{\alpha}_{t})}+\frac{1-\alpha_{t}}{2\alpha_{t}-1-\overline{\alpha}_{t}}\|x_{t}-\sqrt{\overline{\alpha}_{t}}x_{0}\|_{2}\|s_{t}^{\star}(x_{t})\|_{2}\nonumber \\
 & \overset{\text{(b)}}{\leq}-\frac{\|x_{t}-\sqrt{\overline{\alpha}_{t}}x_{0}\|_{2}^{2}}{2(1-\overline{\alpha}_{t})}-\frac{1-\alpha_{t}}{(1-\overline{\alpha}_{t})(2\alpha_{t}-1-\overline{\alpha}_{t})}\|x_{t}-\sqrt{\overline{\alpha}_{t}}x_{0}\|_{2}^{2}\nonumber \\
 & \qquad+\frac{1-\alpha_{t}}{(1-\overline{\alpha}_{t})(2\alpha_{t}-1-\overline{\alpha}_{t})}\|x_{t}-\sqrt{\overline{\alpha}_{t}}x_{0}\|_{2}^{2}+\frac{(1-\alpha_{t})(1-\overline{\alpha}_{t})}{4(2\alpha_{t}-1-\overline{\alpha}_{t})}\|s_{t}^{\star}(x_{t})\|_{2}^{2}\nonumber \\
 & \overset{\text{(c)}}{\leq}-\frac{\|x_{t}-\sqrt{\overline{\alpha}_{t}}x_{0}\|_{2}^{2}}{2(1-\overline{\alpha}_{t})}+\left(1-\alpha_{t}\right)\|s_{t}^{\star}(x_{t})\|_{2}^{2}.\label{eq:proof-eq-2}
\end{align}
Here step (a) utilizes the triangle inequality and $u_{t}=x_{t}+(1-\alpha_{t})s_{t}^{\star}(x_{t})$;
step (b) invokes the AM-GM inequality; whereas step (c) follows from
(\ref{eq:proof-lemma-2-2.3}). Therefore we have 
\begin{align}
I_{k} & \overset{\text{(i)}}{\leq}\int_{x_{t}\in\mathcal{L}_{t,k},u_{t}\in\mathcal{B}_{t}}\int_{x_{0}}p_{X_{0}}(x_{0})\Big(\frac{1}{2\pi(1-\overline{\alpha}_{t})}\Big)^{d/2}\exp\Big(-\frac{\|x_{t}-\sqrt{\overline{\alpha}_{t}}x_{0}\|_{2}^{2}}{2(1-\overline{\alpha}_{t})}+\left(1-\alpha_{t}\right)\|s_{t}^{\star}(x_{t})\|_{2}^{2}\Big)\mathrm{d}x_{0}\mathrm{d}u_{t}\nonumber \\
 & =\int_{x_{t}\in\mathcal{L}_{t,k},u_{t}\in\mathcal{B}_{t}}\int_{x_{0}}p_{X_{0},X_{t}}(x_{0},x_{t})\exp\Big(\left(1-\alpha_{t}\right)\|s_{t}^{\star}(x_{t})\|_{2}^{2}\Big)\mathrm{d}x_{0}\mathrm{d}u_{t}\nonumber \\
 & \overset{\text{(ii)}}{=}\exp\Big(200c_{1}(2^{k}C_{1}+c_{0})\frac{d\log^{2}T}{T}\Big)\int_{x_{t}\in\mathcal{L}_{t,k},u_{t}\in\mathcal{B}_{t}}p_{X_{t}}(x_{t})\mathrm{d}u_{t}\nonumber \\
 & \overset{\text{(iii)}}{\leq}\exp\Big(200c_{1}(2^{k}C_{1}+c_{0})\frac{d\log^{2}T}{T}\Big)\int_{u_{t}\in\mathcal{B}_{t}}\exp\left(-2^{k-1}C_{1}d\log T\right)\mathrm{d}u_{t}\nonumber \\
 & \overset{\text{(iv)}}{\leq}\exp\Big(200c_{1}(2^{k}C_{1}+c_{0})\frac{d\log^{2}T}{T}-2^{k-1}C_{1}d\log T+4dc_{R}\log T+4d\log(C_{2}d)\Big)\nonumber \\
 & \overset{\text{(v)}}{\leq}\exp\Big(-\frac{C_{1}}{4}2^{k}d\log T\Big)=T^{-(C_{1}/4)2^{k}d}.\label{eq:proof-eq-3}
\end{align}
Here step (i) follows from (\ref{eq:proof-eq-2}); step (ii) uses
a consequence of Lemma~\ref{lem:typical} and Lemma~\ref{lemma:step-size}:
for $x_{t}\in\mathcal{L}_{t,k}$, 
\[
\left(1-\alpha_{t}\right)\|s_{t}^{\star}(x_{t})\|_{2}^{2}\leq25\frac{1-\alpha_{t}}{1-\overline{\alpha}_{t}}(2^{k}C_{1}+c_{0})d\log T\leq200c_{1}(2^{k}C_{1}+c_{0})\frac{d\log^{2}T}{T};
\]
step (iii) follows from the definition of $\mathcal{L}_{t,k}$, which
ensures tht $p_{X_{t}}(x_{t})\leq\exp(-2^{k-1}C_{1}d\log T)$ for
any $x_{t}\in\mathcal{L}_{t,k}$; step (iv) follows from 
\begin{align*}
\log\mathsf{vol}(\mathcal{B}_{t}) & \leq d\log\big(2\sqrt{\overline{\alpha}_{t}}T^{2c_{R}}+2C_{2}\sqrt{d(2\alpha_{t}-1-\overline{\alpha}_{t})\log T}\big)\\
 & \leq4c_{R}d\log T+4d\log(C_{2}d);
\end{align*}
and finally, step (v) holds provided that $C_{1}\gg c_{R}+c_{0}$
and $T\gg d\log^{2}T$. Taking (\ref{eq:proof-eq-2}) and (\ref{eq:proof-eq-3})
collectively yields 
\[
I\leq I_{0}+\sum_{k=1}^{\infty}I_{k}\leq T^{-5}+\sum_{k=1}^{\infty}T^{-(C_{1}/4)2^{k}d}\leq T^{-4},
\]
provided that $C_{1}$ is sufficiently large.

%% file: appendix_low_d.tex
\section{Proof of auxiliary lemmas in Section~\ref{sec:proof-main-low-d}}

\subsection{Proof of Lemma~\ref{lem:Jt}\protect\label{subsec:proof-lem-Jt}}

We start with the following decomposition
\begin{align*}
 & \mathsf{Tr}\left(I-J_{t}(x_{t})\right)\overset{\text{(i)}}{=}\frac{1}{1-\overline{\alpha}_{t}}\Big(\int_{x_{0}}p_{X_{0}|X_{t}}(x_{0}\mymid x_{t})\Vert x_{t}-\sqrt{\overline{\alpha}_{t}}x_{0}\Vert_{2}^{2}\mathrm{d}x_{0}-\big\Vert\int_{x_{0}}p_{X_{0}|X_{t}}(x_{0}\mymid x_{t})\big(x_{t}-\sqrt{\overline{\alpha}_{t}}x_{0}\big)\mathrm{d}x_{0}\big\Vert_{2}^{2}\Big)\\
 & \quad\overset{\text{(ii)}}{=}\frac{\overline{\alpha}_{t}}{1-\overline{\alpha}_{t}}\int_{x_{0}}p_{X_{0}|X_{t}}(x_{0}\mymid x_{t})\Vert x_{0}(x_{t})-x_{0}\Vert_{2}^{2}\mathrm{d}x_{0}+2\sqrt{\frac{\overline{\alpha}_{t}}{1-\overline{\alpha}_{t}}}\int_{x_{0}}p_{X_{0}|X_{t}}(x_{0}\mymid x_{t})\omega^{\top}\big(x_{0}(x_{t})-x_{0}\big)\mathrm{d}x_{0}\\
 & \quad\quad-\frac{\overline{\alpha}_{t}}{1-\overline{\alpha}_{t}}\big\Vert\int_{x_{0}}p_{X_{0}|X_{t}}(x_{0}\mymid x_{t})\big(x_{0}(x_{t})-x_{0}\big)\mathrm{d}x_{0}\big\Vert_{2}^{2}-2\sqrt{\frac{\overline{\alpha}_{t}}{1-\overline{\alpha}_{t}}}\omega^{\top}\int_{x_{0}}p_{X_{0}|X_{t}}(x_{0}\mymid x_{t})\big(x_{0}(x_{t})-x_{0}\big)\mathrm{d}x_{0}\\
 & \quad\leq\underbrace{\frac{\overline{\alpha}_{t}}{1-\overline{\alpha}_{t}}\int_{x_{0}}p_{X_{0}|X_{t}}(x_{0}\mymid x_{t})\Vert x_{0}(x_{t})-x_{0}\Vert_{2}^{2}\mathrm{d}x_{0}}_{\eqqcolon\xi}+4\underbrace{\sqrt{\frac{\overline{\alpha}_{t}}{1-\overline{\alpha}_{t}}}\int p_{X_{0}|X_{t}}(x_{0}\mymid x_{t})\big|\omega^{\top}\big(x_{0}(x_{t})-x_{0}\big)\big|\mathrm{d}x_{0}}_{\eqqcolon\zeta}.
\end{align*}
Here step (i) follows from the definition of $J_{t}(\cdot)$ in (\ref{eq:Jt-defn}),
while step (ii) utilizes the decomposition (\ref{eq:xt-decom}). Then
we bound $\xi$ and $\zeta$ respectively. 
\begin{itemize}
\item Regarding $\xi$, we have
\begin{align*}
\xi & \leq\sum_{i=1}^{N_{\varepsilon}}\frac{\overline{\alpha}_{t}}{1-\overline{\alpha}_{t}}\sup_{x_{0}\in\mathcal{B}_{i}}\Vert x_{0}(x_{t})-x_{0}\Vert_{2}^{2}\mathbb{P}\left(X_{0}\in\mathcal{B}_{i}\mymid X_{t}=x_{t}\right)\\
 & \overset{\text{(a)}}{\leq}\sum_{i=1}^{N_{\varepsilon}}\frac{\overline{\alpha}_{t}}{1-\overline{\alpha}_{t}}\big(\Vert x_{i(x_{t})}^{\star}-x_{i}^{\star}\Vert_{2}+2\varepsilon\big)^{2}\mathbb{P}\left(X_{0}\in\mathcal{B}_{i}\mymid X_{t}=x_{t}\right)\\
 & \leq2\underbrace{\sum_{i\in\mathcal{I}(x_{t};C_{1}\theta)}\frac{\overline{\alpha}_{t}}{1-\overline{\alpha}_{t}}\Vert x_{i(x_{t})}^{\star}-x_{i}^{\star}\Vert_{2}^{2}\mathbb{P}\left(X_{0}\in\mathcal{B}_{i}\mymid X_{t}=x_{t}\right)}_{\eqqcolon\xi_{1}}\\
 & \qquad+2\underbrace{\sum_{i\notin\mathcal{I}(x_{t};C_{1}\theta)}\frac{\overline{\alpha}_{t}}{1-\overline{\alpha}_{t}}\Vert x_{i(x_{t})}^{\star}-x_{i}^{\star}\Vert_{2}^{2}\mathbb{P}\left(X_{0}\in\mathcal{B}_{i}\mymid X_{t}=x_{t}\right)}_{\eqqcolon\xi_{2}}+4\underbrace{\frac{\overline{\alpha}_{t}}{1-\overline{\alpha}_{t}}\varepsilon^{2}}_{\eqqcolon\xi_{3}},
\end{align*}
where the constant $C_{1}$ was specified in Lemma~\ref{lem:cond-low-dim}.
Here step (a) follows from the fact that, for $x_{0}\in\mathcal{B}_{i}$,
we have 
\begin{equation}
\Vert x_{0}(x_{t})-x_{0}\Vert_{2}\leq\Vert x_{0}(x_{t})-x_{i(x_{t})}^{\star}\Vert_{2}+\Vert x_{i(x_{t})}^{\star}-x_{i}^{\star}\Vert_{2}+\Vert x_{i}^{\star}-x_{0}\Vert_{2}\leq\Vert x_{i(x_{t})}^{\star}-x_{i}^{\star}\Vert_{2}+2\varepsilon;\label{eq:x0(xt)-x0}
\end{equation}
In view of the definition of $\mathcal{I}(x_{t};C_{1}\theta)$, we
have
\begin{align*}
\xi_{1} & \leq C_{1}\theta k\log T\sum_{i\in\mathcal{I}(x_{t};C_{1}\theta)}\mathbb{P}\left(X_{0}\in\mathcal{B}_{i}\mymid X_{t}=x_{t}\right)\leq C_{1}\theta k\log T.
\end{align*}
To bound $\xi_{2}$, we have
\begin{align*}
\xi_{2} & \overset{\text{(i)}}{\leq}\sum_{i\notin\mathcal{I}(x_{t};C_{1}\theta)}\frac{\overline{\alpha}_{t}}{1-\overline{\alpha}_{t}}\Vert x_{i(x_{t})}^{\star}-x_{i}^{\star}\Vert_{2}^{2}\exp\left(-\frac{\overline{\alpha}_{t}}{16\left(1-\overline{\alpha}_{t}\right)}\Vert x_{i(x_{t})}^{\star}-x_{i}^{\star}\Vert_{2}^{2}\right)\mathbb{P}\left(X_{0}\in\mathcal{B}_{i}\right)\\
 & \overset{\text{(ii)}}{\leq}\sum_{i\notin\mathcal{I}(x_{t};C_{1}\theta)}\exp\left(-\frac{\overline{\alpha}_{t}}{32\left(1-\overline{\alpha}_{t}\right)}\Vert x_{i(x_{t})}^{\star}-x_{i}^{\star}\Vert_{2}^{2}\right)\mathbb{P}\left(X_{0}\in\mathcal{B}_{i}\right)\\
 & \overset{\text{(iii)}}{\leq}\sum_{i\notin\mathcal{I}(x_{t};C_{1}\theta)}\exp\left(-\frac{1}{32}C_{1}\theta k\log T\right)\mathbb{P}\left(X_{0}\in\mathcal{B}_{i}\right)\leq\exp\left(-\frac{1}{32}C_{1}\theta k\log T\right).
\end{align*}
Here step (i) follows from Lemma~\ref{lem:cond-low-dim}, while step
(ii) holds when $\frac{\overline{\alpha}_{t}}{1-\overline{\alpha}_{t}}\Vert x_{i(x_{t})}^{\star}-x_{i}^{\star}\Vert_{2}^{2}$
is large enough, which can be guaranteed by taking $C_{1}>0$ to be
sufficiently large; step (iii) follows from the definition of $\mathcal{I}(x_{t};C_{1}\theta)$.
In addition, $\xi_{3}\ll1$ as long as $\varepsilon$ is sufficiently
small (see (\ref{eq:eps-condition}). Therefore we have
\begin{equation}
\xi\leq2\xi_{1}+2\xi_{2}+\xi_{3}\leq3C_{1}\theta k\log T\label{eq:xi-bound}
\end{equation}
provided that $C_{1}>0$ is sufficiently large.
\item Regarding $\zeta$, we have
\begin{align*}
\zeta & \leq\sqrt{\frac{\overline{\alpha}_{t}}{1-\overline{\alpha}_{t}}}\sum_{i=1}^{N_{\varepsilon}}\sup_{x_{0}\in\mathcal{B}_{i}}\big|\omega^{\top}\big(x_{0}(x_{t})-x_{0}\big)\big|\mathbb{P}\left(X_{0}\in\mathcal{B}_{i}\mymid X_{t}=x_{t}\right)\\
 & \overset{\text{(a)}}{\leq}\sqrt{\frac{\overline{\alpha}_{t}}{1-\overline{\alpha}_{t}}}\sum_{i=1}^{N_{\varepsilon}}\big(\big|\omega^{\top}\big(x_{i}^{\star}-x_{i(x_{t})}^{\star}\big)\big|+\varepsilon\Vert\omega\Vert_{2}\big)\mathbb{P}\left(X_{0}\in\mathcal{B}_{i}\mymid X_{t}=x_{t}\right)\\
 & \overset{\text{(b)}}{\leq}\underbrace{\sqrt{\frac{\overline{\alpha}_{t}}{1-\overline{\alpha}_{t}}}\sum_{i\in\mathcal{I}(x_{t};C_{1}\theta)}\sqrt{\theta k\log T}\Vert x_{i}^{\star}-x_{i(x_{t})}^{\star}\Vert_{2}\mathbb{P}\left(X_{0}\in\mathcal{B}_{i}\mymid X_{t}=x_{t}\right)}_{\eqqcolon\zeta_{1}}\\
 & \qquad+\underbrace{\sqrt{\frac{\overline{\alpha}_{t}}{1-\overline{\alpha}_{t}}}\sum_{i\notin\mathcal{I}(x_{t};C_{1}\theta)}\sqrt{\theta k\log T}\Vert x_{i}^{\star}-x_{i(x_{t})}^{\star}\Vert_{2}\mathbb{P}\left(X_{0}\in\mathcal{B}_{i}\mymid X_{t}=x_{t}\right)}_{\eqqcolon\zeta_{2}}+\underbrace{\sqrt{\frac{\overline{\alpha}_{t}}{1-\overline{\alpha}_{t}}}\varepsilon(2\sqrt{d}+\sqrt{\theta k\log T})}_{\eqqcolon\zeta_{3}}.
\end{align*}
Here step (a) uses Cauchy-Schwarz inequality, while step (b) follows
from the definition of $\mathcal{G}$. By the definition of $\mathcal{I}(x_{t};C_{1}\theta)$,
we have
\[
\zeta_{1}\leq\sum_{i\in\mathcal{I}(x_{t};C_{1}\theta)}\sqrt{C_{1}}\theta k\log T\mathbb{P}\left(X_{0}\in\mathcal{B}_{i}\mymid X_{t}=x_{t}\right)\leq\sqrt{C_{1}}\theta k\log T.
\]
To bound $\zeta_{2}$, we have
\begin{align*}
\zeta_{2} & \overset{\text{(i)}}{\leq}\sum_{i\notin\mathcal{I}(x_{t};C_{1}\theta)}\sqrt{\theta k\log T}\exp\left(-\frac{\overline{\alpha}_{t}}{32\left(1-\overline{\alpha}_{t}\right)}\Vert x_{i(x_{t})}^{\star}-x_{i}^{\star}\Vert_{2}^{2}\right)\mathbb{P}\left(X_{0}\in\mathcal{B}_{i}\right)\\
 & \overset{\text{(ii)}}{\leq}\sqrt{\theta k\log T}\exp\left(-\frac{1}{32}C_{1}\theta k\log T\right)\overset{\text{(iii)}}{\leq}\exp\left(-\frac{1}{64}C_{1}\theta k\log T\right).
\end{align*}
Here step (i) holds when $\frac{\overline{\alpha}_{t}}{1-\overline{\alpha}_{t}}\Vert x_{i(x_{t})}^{\star}-x_{i}^{\star}\Vert_{2}^{2}$
is large enough, which can be guaranteed by taking $C_{1}>0$ to be
sufficiently large; step (ii) follows from the definition of $\mathcal{I}(x_{t};C_{1}\theta)$;
and step (iii) holds when $C_{1}$ is large enough. In addition, we
have $\xi_{3}\ll1$ as long as $\varepsilon$ is sufficiently small
(see (\ref{eq:eps-condition}). Hence we have
\begin{equation}
\zeta\leq2\sqrt{C_{1}}\theta k\log T\label{eq:zeta-bound}
\end{equation}
provided that $C_{1}>0$ is sufficiently large. 
\end{itemize}
Taking the bounds on $\xi$ and $\zeta$ collectively leads to 
\[
\mathsf{Tr}\left(I-J_{t}(x_{t})\right)\leq\xi+4\zeta\leq4C_{1}\theta k\log T
\]
provided that $C_{1}>0$ is large enough. In addition, since $I-J_{t}(x_{t})\succeq0$,
we have
\[
\Vert I-J_{t}(x_{t})\Vert_{\mathrm{F}}^{2}\leq\mathsf{Tr}\left(I-J_{t}(x_{t})\right)^{2},
\]
hence we have
\[
\Vert I-J_{t}(x_{t})\Vert\leq\Vert I-J_{t}(x_{t})\Vert_{\mathrm{F}}\leq\mathsf{Tr}\left(I-J_{t}(x_{t})\right)\leq C_{2}\theta k\log T
\]
provided that $C_{2}\geq4C_{1}$. This finishes the proof of the first
relation (\ref{eq:I-Jt-bound-low-d}). 

Finally, we invoke Lemma~\ref{lemma:jacob-sum-low-d} to obtain
\begin{equation}
\sum_{t=2}^{T}\frac{1-\alpha_{t}}{1-\overline{\alpha}_{t}}\mathsf{Tr}\big(\mathbb{E}\big[\big(\Sigma_{\overline{\alpha}_{t}}(X_{t})\big)^{2}\big]\big)\leq C_{J}k\log T,\label{eq:proof-lemma-1-7-1}
\end{equation}
where the matrix function 
\[
\Sigma_{\overline{\alpha}_{t}}(x)=\mathsf{Cov}\big(Z\mymid\sqrt{\overline{\alpha}_{t}}X_{0}+\sqrt{1-\overline{\alpha}_{t}}Z=x\big)=I_{d}-J_{t}(x).
\]
Here $Z\sim\mathcal{N}(0,I_{d})$ is independent of $X_{0}$. By noticing
that
\[
\mathsf{Tr}\big(\mathbb{E}\big[\big(\Sigma_{\overline{\alpha}_{t}}(X_{t})\big)^{2}\big]\big)=\mathsf{Tr}\big(\mathbb{E}\big[\big(I_{d}-J_{t}(X_{t})\big)^{2}\big]\big)=\mathbb{E}\big[\Vert I-J_{t}(X_{t})\Vert_{\mathrm{F}}^{2}\big]=\int_{x_{t}}\Vert I-J_{t}(x_{t})\Vert_{\mathrm{F}}^{2}\,p_{X_{t}}(x_{t})\mathrm{d}x_{t},
\]
we finish the proof of the second relation (\ref{eq:Jt_sum-low-d}). 

\subsection{Proof of Lemma~\ref{lemma:recursion-low-d}\protect\label{sec:proof-lemma-recursion-low-d}}

The proof of Lemma~\ref{lemma:recursion-low-d} is similar to the
proof of Lemma~\ref{lemma:recursion} in Appendix~\ref{sec:proof-lemma-recursion}.
We will only highlight the differences due to the different update
rule (\ref{eq:coef-design}). Equation (\ref{eq:proof-main-1}) should
be changed to
\begin{align}
 & p_{X_{t-1}}(x_{t-1})-\Delta_{t-1}(x_{t-1})+\Delta_{t\to t-1}(x_{t-1})\label{eq:proof-main-1-low-d}\\
 & \ge\int_{x_{t}\in\mathcal{E}_{t,1}}\mathsf{det}\Big(I-\frac{1-\alpha_{t}}{1-\overline{\alpha}_{t}}J_{t}(x_{t})\Big)^{-1}p_{X_{t}}(x_{t})\Big(\frac{\alpha_{t}(1-\overline{\alpha}_{t})}{2\pi(1-\alpha_{t})(\alpha_{t}-\overline{\alpha}_{t})}\Big)^{d/2}\exp\Big(-\frac{(1-\overline{\alpha}_{t})\big\|\sqrt{\alpha_{t}}x_{t-1}-u_{t}\big\|^{2}}{2(1-\alpha_{t})(\alpha_{t}-\overline{\alpha}_{t})}\Big)\mathrm{d}u_{t}.\nonumber 
\end{align}
Lemma~\ref{lemma:det} need to be changed to the following version.

\begin{lemma}\label{lemma:det-low-d}For any $x_{t}\in\mathcal{E}_{t,1}$,
we have 
\begin{align}
 & \mathsf{det}\Big(I-\frac{1-\alpha_{t}}{1-\overline{\alpha}_{t}}J_{t}(x_{t})\Big)^{-1}p_{X_{t}}(x_{t})\nonumber \\
 & \qquad=\big(2\pi(\alpha_{t}-\overline{\alpha}_{t})\big)^{-d/2}\int_{x_{0}}p_{X_{0}}(x_{0})\exp\Big(-\frac{(1-\overline{\alpha}_{t})\|u_{t}-\sqrt{\overline{\alpha}_{t}}x_{0}\|^{2}}{2(\alpha_{t}-\overline{\alpha}_{t})^{2}}\Big)\mathrm{d}x_{0}\nonumber \\
 & \qquad\qquad\cdot\exp\Big(\xi_{t}(x_{t})+O\Big(\Big(\frac{1-\alpha_{t}}{1-\overline{\alpha}_{t}}\Big)^{2}\big(\vert\mathsf{Tr}(I-J_{t}(x_{t}))\vert+\|I-J_{t}(x_{t})\|_{\mathrm{F}}^{2}\big)\Big)\Big),\label{eq:lemma-det-1-low-d}
\end{align}
where $\xi_{t}(x_{t})\le0$ satisfies 
\begin{equation}
\int_{x_{t}\in\mathcal{E}_{t,1}}|\xi_{t}(x_{t})|p_{X_{t}}(x_{t})\mathrm{d}x_{t}\leq C_{3}\Big(\frac{1-\alpha_{t}}{1-\overline{\alpha}_{t}}\Big)^{2}\int_{x_{t}\in\mathcal{E}_{t,1}}\big(\vert\mathsf{Tr}(I-J_{t}(x_{t}))\vert+\|I-J_{t}(x_{t})\|_{\mathrm{F}}^{2}\big)p_{X_{t}}(x_{t})\mathrm{d}x_{t}+T^{-4}\label{eq:lemma-det-2-low-d}
\end{equation}
for some universal constant $C_{5}>0$. \end{lemma}\begin{proof}See
Appendix~\ref{subsec:proof-lemma-det-low-d}.\end{proof}

Taking the decomposition (\ref{eq:lemma-det-1-low-d}) and (\ref{eq:proof-main-1-low-d})
collectively, we have 
\begin{align}
 & p_{X_{t-1}}(x_{t-1})-\Delta_{t-1}(x_{t-1})+\Delta_{t\to t-1}(x_{t-1})+\delta_{t-1}(x_{t-1})\label{eq:proof-main-2-low-d}\\
 & \quad\geq\int_{x_{0}}\int_{x_{t}}\exp\bigg(\bigg[\xi_{t}(x_{t})+O\Big(\Big(\frac{1-\alpha_{t}}{1-\overline{\alpha}_{t}}\Big)^{2}\big(\vert\mathsf{Tr}(I-J_{t}(x_{t}))\vert+\|I-J_{t}(x_{t})\|_{\mathrm{F}}^{2}\big)\Big)\bigg]\ind\left\{ x_{t}\in\mathcal{E}_{t,1}\right\} \bigg)p_{X_{0}}(x_{0})\nonumber \\
 & \quad\quad\cdot\Big(\frac{\alpha_{t}(1-\overline{\alpha}_{t})^{2}}{4\pi^{2}(1-\alpha_{t})(\alpha_{t}-\overline{\alpha}_{t})^{3}}\Big)^{d/2}\exp\Big(-\frac{(1-\overline{\alpha}_{t})\|u_{t}-\sqrt{\overline{\alpha}_{t}}x_{0}\|^{2}}{2(\alpha_{t}-\overline{\alpha}_{t})^{2}}-\frac{(1-\overline{\alpha}_{t})\big\|\sqrt{\alpha_{t}}x_{t-1}-u_{t}\big\|^{2}}{2(1-\alpha_{t})(\alpha_{t}-\overline{\alpha}_{t})}\Big)\mathrm{d}u_{t}\mathrm{d}x_{0},\nonumber 
\end{align}
where we define 
\begin{align}
\delta_{t-1}(x_{t-1}) & :=\int_{x_{0}}\int_{x_{t}\notin\mathcal{E}_{t,1}}p_{X_{0}}(x_{0})\Big(\frac{\alpha_{t}(1-\overline{\alpha}_{t})^{2}}{4\pi^{2}(1-\alpha_{t})(\alpha_{t}-\overline{\alpha}_{t})^{3}}\Big)^{d/2}\nonumber \\
 & \qquad\qquad\cdot\exp\Big(-\frac{(1-\overline{\alpha}_{t})\|u_{t}-\sqrt{\overline{\alpha}_{t}}x_{0}\|^{2}}{2(\alpha_{t}-\overline{\alpha}_{t})^{2}}\Big)\exp\Big(-\frac{(1-\overline{\alpha}_{t})\big\|\sqrt{\alpha_{t}}x_{t-1}-u_{t}\big\|^{2}}{2(1-\alpha_{t})(\alpha_{t}-\overline{\alpha}_{t})}\Big)\mathrm{d}u_{t}\mathrm{d}x_{0}.\label{eq:delta-defn-low-d}
\end{align}
Moreover, it is straightforward to check that 
\begin{align}
 & \int_{x_{0}}\int_{x_{t}}p_{X_{0}}(x_{0})\Big(\frac{\alpha_{t}(1-\overline{\alpha}_{t})^{2}}{4\pi^{2}(1-\alpha_{t})(\alpha_{t}-\overline{\alpha}_{t})^{3}}\Big)^{d/2}\exp\Big(-\frac{(1-\overline{\alpha}_{t})\|u_{t}-\sqrt{\overline{\alpha}_{t}}x_{0}\|^{2}}{2(\alpha_{t}-\overline{\alpha}_{t})^{2}}\Big)\nonumber \\
 & \qquad\qquad\qquad\qquad\cdot\exp\Big(-\frac{(1-\overline{\alpha}_{t})\big\|\sqrt{\alpha_{t}}x_{t-1}-u_{t}\big\|^{2}}{2(1-\alpha_{t})(\alpha_{t}-\overline{\alpha}_{t})}\Big)\mathrm{d}u_{t}\mathrm{d}x_{0}=p_{X_{t-1}}(x_{t-1}).\label{eq:proof-main-3-low-d}
\end{align}
Then we can continue the derivation in (\ref{eq:proof-main-2-low-d}):
\begin{align*}
 & p_{X_{t-1}}(x_{t-1})-\Delta_{t-1}(x_{t-1})+\Delta_{t\to t-1}(x_{t-1})+\delta_{t-1}(x_{t-1})\\
 & \quad\overset{\text{(i)}}{\geq}\int_{x_{0}}\int_{x_{t}}\bigg(1+\bigg[\xi_{t}(x_{t})+O\Big(\Big(\frac{1-\alpha_{t}}{1-\overline{\alpha}_{t}}\Big)^{2}\big(\vert\mathsf{Tr}(I-J_{t}(x_{t}))\vert+\|I-J_{t}(x_{t})\|_{\mathrm{F}}^{2}\big)\Big)\bigg]\ind\left\{ x_{t}\in\mathcal{E}_{t,1}\right\} \bigg)p_{X_{0}}(x_{0})\\
 & \quad\qquad\cdot\Big(\frac{\alpha_{t}(1-\overline{\alpha}_{t})^{2}}{4\pi^{2}(1-\alpha_{t})(\alpha_{t}-\overline{\alpha}_{t})^{3}}\Big)^{d/2}\exp\Big(-\frac{(1-\overline{\alpha}_{t})\|u_{t}-\sqrt{\overline{\alpha}_{t}}x_{0}\|^{2}}{2(\alpha_{t}-\overline{\alpha}_{t})^{2}}-\frac{(1-\overline{\alpha}_{t})\big\|\sqrt{\alpha_{t}}x_{t-1}-u_{t}\big\|^{2}}{2(1-\alpha_{t})(\alpha_{t}-\overline{\alpha}_{t})}\Big)\mathrm{d}u_{t}\mathrm{d}x_{0}\\
 & \quad\overset{\text{(ii)}}{=}p_{X_{t-1}}(x_{t-1})+\int_{x_{0}}\int_{x_{t}\in\mathcal{E}_{t,1}}\bigg[\xi_{t}(x_{t})+O\Big(\Big(\frac{1-\alpha_{t}}{1-\overline{\alpha}_{t}}\Big)^{2}\big(\vert\mathsf{Tr}(I-J_{t}(x_{t}))\vert+\|I-J_{t}(x_{t})\|_{\mathrm{F}}^{2}\big)\Big)\bigg]p_{X_{0}}(x_{0})\\
 & \quad\qquad\cdot\Big(\frac{\alpha_{t}(1-\overline{\alpha}_{t})^{2}}{4\pi^{2}(1-\alpha_{t})(\alpha_{t}-\overline{\alpha}_{t})^{3}}\Big)^{d/2}\exp\Big(-\frac{(1-\overline{\alpha}_{t})\|u_{t}-\sqrt{\overline{\alpha}_{t}}x_{0}\|^{2}}{2(\alpha_{t}-\overline{\alpha}_{t})^{2}}-\frac{(1-\overline{\alpha}_{t})\big\|\sqrt{\alpha_{t}}x_{t-1}-u_{t}\big\|^{2}}{2(1-\alpha_{t})(\alpha_{t}-\overline{\alpha}_{t})}\Big)\mathrm{d}u_{t}\mathrm{d}x_{0}.
\end{align*}
By rearranging terms and integrate over the variable $x_{t-1}$, we
arrive at 
\begin{align}
 & \int_{x_{t-1}}\Delta_{t-1}(x_{t-1})\mathrm{d}x_{t-1}\leq\int_{x_{t-1}}\big(\Delta_{t}(x_{t-1})+\delta_{t-1}(x_{t-1})\big)\mathrm{d}x_{t-1}\nonumber \\
 & \qquad+\int_{x_{0}}\int_{x_{t}\in\mathcal{E}_{t,1}}\bigg(|\xi_{t}(x_{t})|+O\Big(\Big(\frac{1-\alpha_{t}}{1-\overline{\alpha}_{t}}\Big)^{2}\big(\vert\mathsf{Tr}(I-J_{t}(x_{t}))\vert+\|I-J_{t}(x_{t})\|_{\mathsf{F}}^{2}\big)\Big)\bigg)p_{X_{0}}(x_{0})\nonumber \\
 & \qquad\qquad\cdot\Big(\frac{1-\overline{\alpha}_{t}}{2\pi(\alpha_{t}-\overline{\alpha}_{t})^{2}}\Big)^{d/2}\exp\Big(-\frac{(1-\overline{\alpha}_{t})\|u_{t}-\sqrt{\overline{\alpha}_{t}}x_{0}\|_{2}^{2}}{2(\alpha_{t}-\overline{\alpha}_{t})^{2}}\Big)\mathrm{d}u_{t}\mathrm{d}x_{0},\label{eq:proof-main-5-low-d}
\end{align}
where we used (\ref{eq:proof-main-0.5}) and for any fixed $u_{t}$,
the function 
\[
\left(2\pi\frac{(1-\alpha_{t})(\alpha_{t}-\overline{\alpha}_{t})}{(1-\overline{\alpha}_{t})\alpha_{t}}\right)^{-d/2}\exp\Big(-\frac{(1-\overline{\alpha}_{t})\big\|\sqrt{\alpha_{t}}x_{t-1}-u_{t}\big\|^{2}}{2(1-\alpha_{t})(\alpha_{t}-\overline{\alpha}_{t})}\Big)
\]
is a density function of $x_{t-1}$. To establish the desired result,
we need the following two lemmas.

\begin{lemma}\label{lemma:u-x-low-d} Suppose that $T\gg\theta k\log^{2}T$.
For any $x_{t}\in\mathcal{E}_{t,1}$, we have 
\[
\int_{x_{0}}p_{X_{0}}(x_{0})\Big(\frac{1-\overline{\alpha}_{t}}{2\pi(\alpha_{t}-\overline{\alpha}_{t})^{2}}\Big)^{d/2}\exp\Big(-\frac{(1-\overline{\alpha}_{t})\|u_{t}-\sqrt{\overline{\alpha}_{t}}x_{0}\|^{2}}{2(\alpha_{t}-\overline{\alpha}_{t})^{2}}\Big)\mathrm{d}x_{0}\leq20\det\Big(I-\frac{1-\alpha_{t}}{1-\overline{\alpha}_{t}}J_{t}(x_{t})\Big)^{-1}p_{X_{t}}(x_{t}).
\]
\end{lemma}\begin{proof}See Appendix~\ref{sec:proof-lemma-u-x-low-d}.\end{proof}

\begin{lemma}\label{lemma:delta-low-d} For the function $\delta_{t-1}(\cdot)$
defined in (\ref{eq:delta-defn}), we have 
\[
\int_{x_{t-1}}\delta_{t-1}(x_{t-1})\mathrm{d}x_{t-1}\leq T^{-4}.
\]
\end{lemma}\begin{proof}The proof is the same as that of Lemma~\ref{lemma:delta},
and is hence omitted.\end{proof}

Equipped with Lemmas~\ref{lemma:u-x-low-d}~and~\ref{lemma:delta-low-d},
we can continue the derivation in (\ref{eq:proof-main-5-low-d}) as
follows: 
\begin{align*}
 & \int_{x_{t-1}}\Delta_{t-1}(x_{t-1})\mathrm{d}x_{t-1}\\
 & \qquad\overset{\text{(a)}}{\leq}\int_{x_{t}}\Delta_{t}(x_{t})\mathrm{d}x_{t}+20\int_{x_{t}\in\mathcal{E}_{t,1}}\bigg(|\xi_{t}(x_{t})|+O\Big(\Big(\frac{1-\alpha_{t}}{1-\overline{\alpha}_{t}}\Big)^{2}\big(\vert\mathsf{Tr}(I-J_{t}(x_{t}))\vert+\|I-J_{t}(x_{t})\|_{\mathsf{F}}^{2}\big)\Big)\bigg)\\
 & \qquad\qquad\qquad\qquad\qquad\qquad\qquad\qquad\cdot\det\Big(I-\frac{1-\alpha_{t}}{1-\overline{\alpha}_{t}}J_{t}(x_{t})\Big)^{-1}p_{X_{t}}(x_{t})\mathrm{d}u_{t}+T^{-4}\\
 & \qquad\overset{\text{(b)}}{=}\int_{x_{t}}\Delta_{t}(x_{t})\mathrm{d}x_{t}+T^{-4}+20\int_{x_{t}\in\mathcal{E}_{t,1}}\bigg(|\xi_{t}(x_{t})|+O\Big(\Big(\frac{1-\alpha_{t}}{1-\overline{\alpha}_{t}}\Big)^{2}\big(\vert\mathsf{Tr}(I-J_{t}(x_{t}))\vert+\|I-J_{t}(x_{t})\|_{\mathsf{F}}^{2}\big)\Big)\bigg)p_{X_{t}}(x_{t})\mathrm{d}x_{t}\\
 & \qquad\overset{\text{(c)}}{\leq}\int_{x_{t}}\Delta_{t}(x_{t})\mathrm{d}x_{t}+T^{-3}+C_{4}\Big(\frac{1-\alpha_{t}}{1-\overline{\alpha}_{t}}\Big)^{2}\int_{x_{t}\in\mathcal{E}_{t,1}}\big(d\log T+\|J_{t}(x_{t})\|_{\mathsf{F}}^{2}\big)p_{X_{t}}(x_{t})\mathrm{d}x_{t},
\end{align*}
which establishes the desired recursive relation. Here step (a) follows
from Lemmas~\ref{lemma:u-x-low-d}~and~\ref{lemma:delta-low-d};
step (b) follows from $u_{t}=x_{t}+(1-\alpha_{t})s_{t}^{\star}(x_{t})$,
hence 
\[
\mathrm{d}u_{t}=\mathsf{det}\Big(I-\frac{1-\alpha_{t}}{1-\overline{\alpha}_{t}}J_{t}(x_{t})\Big)\mathrm{d}x_{t};
\]
whereas step (c) uses (\ref{eq:lemma-det-2-low-d}) in Lemma~\ref{lemma:det-low-d},
and holds provided that $C_{4}\gg C_{3}$ is sufficiently large. In
addition, the relation $\int\Delta_{T}(x)\mathrm{d}x\leq T^{-4}$
can be established in the same way as the proof of Lemma~\ref{lemma:recursion},
and is hence omitted here.

\subsection{Proof of Lemma~\ref{lemma:det-low-d}\protect\label{subsec:proof-lemma-det-low-d}}

The proof is similar to that of Lemma~\ref{lemma:det}. Recall that
$u_{t}=x_{t}+(1-\alpha_{t})s_{t}^{\star}(x_{t})$, we start with the
following decomposition 
\begin{align*}
 & \frac{(1-\overline{\alpha}_{t})\|u_{t}-\sqrt{\overline{\alpha}_{t}}x_{0}\|^{2}}{2(\alpha_{t}-\overline{\alpha}_{t})^{2}}=\frac{\|x_{t}-\sqrt{\overline{\alpha}_{t}}x_{0}\|_{2}^{2}}{2(1-\overline{\alpha}_{t})}+\frac{(1-\alpha_{t})(1+\alpha_{t}-2\overline{\alpha}_{t})\|x_{t}-\sqrt{\overline{\alpha}_{t}}x_{0}\|_{2}^{2}}{2(\alpha_{t}-\overline{\alpha}_{t})^{2}(1-\overline{\alpha}_{t})}\\
 & \quad\qquad+\frac{(1-\alpha_{t})(1-\overline{\alpha}_{t})s_{t}^{\star}(x_{t})^{\top}(x_{t}-\sqrt{\overline{\alpha}_{t}}x_{0})}{(\alpha_{t}-\overline{\alpha}_{t})^{2}}+\frac{(1-\alpha_{t})^{2}(1-\overline{\alpha}_{t})\|s_{t}^{\star}(x_{t})\|_{2}^{2}}{2(\alpha_{t}-\overline{\alpha}_{t})^{2}}\\
 & \quad=\frac{\|x_{t}-\sqrt{\overline{\alpha}_{t}}x_{0}\|_{2}^{2}}{2(1-\overline{\alpha}_{t})}+\frac{(1-\alpha_{t})(1+\alpha_{t}-2\overline{\alpha}_{t})}{2(\alpha_{t}-\overline{\alpha}_{t})^{2}(1-\overline{\alpha}_{t})}\int_{x_{0}}p_{X_{0}|X_{t}}(x_{0}\mymid x_{t})\|x_{t}-\sqrt{\overline{\alpha}_{t}}x_{0}\|_{2}^{2}\mathrm{d}x_{0}\\
 & \quad\qquad-\frac{(1-\alpha_{t})(1+\alpha_{t}-2\overline{\alpha}_{t})}{2(\alpha_{t}-\overline{\alpha}_{t})^{2}(1-\overline{\alpha}_{t})}\bigg\|\int_{x_{0}}p_{X_{0}|X_{t}}(x_{0}\mymid x_{t})\big(x_{t}-\sqrt{\overline{\alpha}_{t}}x_{0}\big)\mathrm{d}x_{0}\bigg\|_{2}^{2}+\zeta_{t}(x_{t},x_{0}),
\end{align*}
where we let
\begin{align}
\zeta_{t}(x_{t},x_{0}) & \coloneqq\frac{(1-\alpha_{t})(1+\alpha_{t}-2\overline{\alpha}_{t})\big(\|x_{t}-\sqrt{\overline{\alpha}_{t}}x_{0}\|_{2}^{2}-\int_{x_{0}}p_{X_{0}|X_{t}}(x_{0}\mymid x_{t})\|x_{t}-\sqrt{\overline{\alpha}_{t}}x_{0}\|_{2}^{2}\mathrm{d}x_{0}\big)}{2(\alpha_{t}-\overline{\alpha}_{t})^{2}(1-\overline{\alpha}_{t})}\label{eq:zeta-defn-low-d}\\
 & \qquad+\frac{(1-\alpha_{t})\big[\int_{x_{0}}p_{X_{0}|X_{t}}(x_{0}\mymid x_{t})\big(x_{t}-\sqrt{\overline{\alpha}_{t}}x_{0}\big)\mathrm{d}x_{0}\big]^{\top}\sqrt{\overline{\alpha}_{t}}\big(x_{0}-\int_{x_{0}}p_{X_{0}|X_{t}}(x_{0}\mymid x_{t})x_{0}\mathrm{d}x_{0}\big)}{(\alpha_{t}-\overline{\alpha}_{t})^{2}}.\nonumber 
\end{align}
We can further derive 
\begin{align}
 & \frac{(1-\overline{\alpha}_{t})\|u_{t}-\sqrt{\overline{\alpha}_{t}}x_{0}\|^{2}}{2(\alpha_{t}-\overline{\alpha}_{t})^{2}}\overset{\text{(i)}}{=}\frac{\|x_{t}-\sqrt{\overline{\alpha}_{t}}x_{0}\|_{2}^{2}}{2(1-\overline{\alpha}_{t})}+\frac{(1-\alpha_{t})(1+\alpha_{t}-2\overline{\alpha}_{t})}{2(\alpha_{t}-\overline{\alpha}_{t})^{2}}\mathsf{Tr}\left(I-J_{t}(x_{t})\right)+\zeta_{t}(x_{t},x_{0})\nonumber \\
 & \qquad\overset{\text{(ii)}}{=}\frac{\|x_{t}-\sqrt{\overline{\alpha}_{t}}x_{0}\|_{2}^{2}}{2(1-\overline{\alpha}_{t})}+\frac{1-\alpha_{t}}{\alpha_{t}-\overline{\alpha}_{t}}\mathsf{Tr}\left(I-J_{t}(x_{t})\right)+\zeta_{t}(x_{t},x_{0})+O\Big(\Big(\frac{1-\alpha_{t}}{\alpha_{t}-\overline{\alpha}_{t}}\Big)^{2}\vert\mathsf{Tr}(I-J_{t}(x_{t}))\vert\Big)\nonumber \\
 & \qquad\overset{\text{(iii)}}{=}\frac{\|x_{t}-\sqrt{\overline{\alpha}_{t}}x_{0}\|_{2}^{2}}{2(1-\overline{\alpha}_{t})}+\log\det\Big(I+\frac{1-\alpha_{t}}{\alpha_{t}-\overline{\alpha}_{t}}(I-J_{t}(x_{t}))\Big)\nonumber \\
 & \qquad\qquad+\zeta_{t}(x_{t},x_{0})+O\Big(\Big(\frac{1-\alpha_{t}}{\alpha_{t}-\overline{\alpha}_{t}}\Big)^{2}\big(\vert\mathsf{Tr}(I-J_{t}(x_{t}))\vert+\|I-J_{t}(x_{t})\|_{\mathrm{F}}^{2}\big)\Big).\label{eq:zeta-decom-low-d}
\end{align}
 Here step (i) follows from (\ref{eq:proof-lemma-1-3}) and (\ref{eq:proof-lemma-1-4});
step (ii) holds since
\begin{align*}
\frac{(1-\alpha_{t})(1+\alpha_{t}-2\overline{\alpha}_{t})}{2(\alpha_{t}-\overline{\alpha}_{t})^{2}} & =\frac{1-\alpha_{t}}{\alpha_{t}-\overline{\alpha}_{t}}\left(1+\frac{1-\alpha_{t}}{2(\alpha_{t}-\overline{\alpha}_{t})}\right),
\end{align*}
while step (iii) uses the fact that
\begin{align*}
\log\mathsf{det}\Big(I+\frac{1-\alpha_{t}}{\alpha_{t}-\overline{\alpha}_{t}}(I-J_{t}(x_{t}))\Big) & =\frac{1-\alpha_{t}}{\alpha_{t}-\overline{\alpha}_{t}}\mathsf{Tr}\big(I-J_{t}(x_{t})\big)+O\bigg(\Big(\frac{1-\alpha_{t}}{\alpha_{t}-\overline{\alpha}_{t}}\Big)^{2}\|I-J_{t}(x_{t})\|_{\mathrm{F}}^{2}\bigg).
\end{align*}
Then we have
\begin{align*}
 & \int_{x_{0}}p_{X_{0}}(x_{0})\exp\Big(-\frac{(1-\overline{\alpha}_{t})\|u_{t}-\sqrt{\overline{\alpha}_{t}}x_{0}\|^{2}}{2(\alpha_{t}-\overline{\alpha}_{t})^{2}}\Big)\mathrm{d}x_{0}=\int_{x_{0}}p_{X_{0}}(x_{0})\exp\Big(-\frac{\|x_{t}-\sqrt{\overline{\alpha}_{t}}x_{0}\|_{2}^{2}}{2(1-\overline{\alpha}_{t})}-\zeta_{t}(x_{t},x_{0})\Big)\mathrm{d}x_{0}\\
 & \qquad\cdot\exp\bigg(-\log\det\Big(I+\frac{1-\alpha_{t}}{\alpha_{t}-\overline{\alpha}_{t}}(I-J_{t}(x_{t}))\Big)+O\Big(\Big(\frac{1-\alpha_{t}}{\alpha_{t}-\overline{\alpha}_{t}}\Big)^{2}\big(\vert\mathsf{Tr}(I-J_{t}(x_{t}))\vert+\|I-J_{t}(x_{t})\|_{\mathrm{F}}^{2}\big)\Big)\bigg).
\end{align*}
Recall the definition of $\xi_{t}(x_{t})$ in (\ref{eq:defn-xi})
and (\ref{eq:proof-lemma-2-4}), which allows us to write
\begin{align*}
 & \int_{x_{0}}p_{X_{0}}(x_{0})\exp\Big(-\frac{(1-\overline{\alpha}_{t})\|u_{t}-\sqrt{\overline{\alpha}_{t}}x_{0}\|^{2}}{2(\alpha_{t}-\overline{\alpha}_{t})^{2}}\Big)\mathrm{d}x_{0}=\det\Big(I+\frac{1-\alpha_{t}}{\alpha_{t}-\overline{\alpha}_{t}}(I-J_{t}(x_{t}))\Big)^{-1}p_{X_{t}}(x_{t})\big(2\pi(1-\overline{\alpha}_{t})\big)^{d/2}\\
 & \qquad\qquad\qquad\cdot\exp\bigg(-\xi_{t}(x_{t})+O\Big(\Big(\frac{1-\alpha_{t}}{\alpha_{t}-\overline{\alpha}_{t}}\Big)^{2}\big(\vert\mathsf{Tr}(I-J_{t}(x_{t}))\vert+\|I-J_{t}(x_{t})\|_{\mathrm{F}}^{2}\big)\Big)\bigg).
\end{align*}
Using the fact that
\begin{equation}
\mathsf{det}\Big(I+\frac{1-\alpha_{t}}{\alpha_{t}-\overline{\alpha}_{t}}(I-J_{t}(x_{t}))\Big)=\Big(\frac{1-\overline{\alpha}_{t}}{\alpha_{t}-\overline{\alpha}_{t}}\Big)^{d}\mathsf{det}\Big(I-\frac{1-\alpha_{t}}{1-\overline{\alpha}_{t}}J_{t}(x_{t})\Big),\label{eq:det-relation}
\end{equation}
we arrive at
\begin{align}
 & \mathsf{det}\Big(I-\frac{1-\alpha_{t}}{1-\overline{\alpha}_{t}}J_{t}(x_{t})\Big)^{-1}p_{X_{t}}(x_{t})\nonumber \\
 & \qquad=\big(2\pi(\alpha_{t}-\overline{\alpha}_{t})\big)^{-d/2}\int_{x_{0}}p_{X_{0}}(x_{0})\exp\Big(-\frac{(1-\overline{\alpha}_{t})\|u_{t}-\sqrt{\overline{\alpha}_{t}}x_{0}\|^{2}}{2(\alpha_{t}-\overline{\alpha}_{t})^{2}}\Big)\mathrm{d}x_{0}\nonumber \\
 & \qquad\qquad\cdot\exp\Big(\xi_{t}(x_{t})+O\Big(\Big(\frac{1-\alpha_{t}}{1-\overline{\alpha}_{t}}\Big)^{2}\big(\vert\mathsf{Tr}(I-J_{t}(x_{t}))\vert+\|I-J_{t}(x_{t})\|_{\mathrm{F}}^{2}\big)\Big)\Big),\label{eq:proof-det-1-low-d}
\end{align}
which gives the desired decomposition (\ref{eq:lemma-det-1-low-d}).

In order to establish (\ref{eq:lemma-det-2-low-d}), we need the following
lemma.

\begin{lemma} \label{lemma:small-prob-low-d}Suppose that $\theta\gg C_{\mathsf{cover}}$
and $T\gg c_{1}C_{1}\log T$. Then we have\begin{subequations}\label{eq:proof-lemma-2-6-low-d}
\begin{align}
\int_{x_{0}}\int_{x_{t}\notin\mathcal{E}_{t,1}}p_{X_{0}}(x_{0})\Big(\frac{1-\overline{\alpha}_{t}}{2\pi(\alpha_{t}-\overline{\alpha}_{t})^{2}}\Big)^{d/2}\exp\Big(-\frac{(1-\overline{\alpha}_{t})\|u_{t}-\sqrt{\overline{\alpha}_{t}}x_{0}\|_{2}^{2}}{2(\alpha_{t}-\overline{\alpha}_{t})^{2}}\Big)\mathrm{d}x_{0}\mathrm{d}u_{t} & \leq T^{-4}\label{eq:proof-lemma-2-6A-low-d}
\end{align}
and 
\begin{equation}
\int_{x_{t}\in\mathcal{E}_{t,1}^{\mathrm{c}}}p_{X_{t}}(x_{t})\mathrm{d}x_{t}\leq T^{-4}.\label{eq:proof-lemma-2-6B-low-d}
\end{equation}
\end{subequations}

\end{lemma} \begin{proof}See Appendix~\ref{subsec:proof-lemma-small-prob-low-d}.\end{proof}

Then we have 
\begin{align*}
1 & \overset{\text{(i)}}{\geq}\int_{x_{t}\in\mathcal{E}_{t,1}}\int_{x_{0}}\Big(\frac{1-\overline{\alpha}_{t}}{2\pi(\alpha_{t}-\overline{\alpha}_{t})^{2}}\Big)^{d/2}p_{X_{0}}(x_{0})\exp\Big(-\frac{(1-\overline{\alpha}_{t})\|u_{t}-\sqrt{\overline{\alpha}_{t}}x_{0}\|_{2}^{2}}{2(\alpha_{t}-\overline{\alpha}_{t})^{2}}\Big)\mathrm{d}x_{0}\mathrm{d}u_{t}\\
 & \overset{\text{(ii)}}{=}\big(\frac{1-\overline{\alpha}_{t}}{\alpha_{t}-\overline{\alpha}_{t}}\big)^{d/2}\int_{x_{t}\in\mathcal{E}_{t,1}}\det\Big(I-\frac{1-\alpha_{t}}{1-\overline{\alpha}_{t}}J_{t}(x_{t})\Big)^{-1}p_{X_{t}}(x_{t})\\
 & \qquad\qquad\qquad\qquad\qquad\qquad\cdot\exp\bigg(-\xi_{t}(x_{t})+O\Big(\Big(\frac{1-\alpha_{t}}{1-\overline{\alpha}_{t}}\Big)^{2}\big(\vert\mathsf{Tr}(I-J_{t}(x_{t}))\vert+\|I-J_{t}(x_{t})\|_{\mathrm{F}}^{2}\big)\Big)\bigg)\mathrm{d}u_{t}\\
 & \overset{\text{(iii)}}{=}\big(\frac{1-\overline{\alpha}_{t}}{\alpha_{t}-\overline{\alpha}_{t}}\big)^{d/2}\int_{x_{t}\in\mathcal{E}_{t,1}}p_{X_{t}}(x_{t})\exp\bigg(-\xi_{t}(x_{t})+O\Big(\Big(\frac{1-\alpha_{t}}{1-\overline{\alpha}_{t}}\Big)^{2}\big(\vert\mathsf{Tr}(I-J_{t}(x_{t}))\vert+\|I-J_{t}(x_{t})\|_{\mathrm{F}}^{2}\big)\Big)\bigg)\mathrm{d}x_{t}\\
 & \overset{\text{(iv)}}{\geq}\int_{x_{t}\in\mathcal{E}_{t,1}}\bigg(1-\xi_{t}(x_{t})+O\Big(\Big(\frac{1-\alpha_{t}}{1-\overline{\alpha}_{t}}\Big)^{2}\big(\vert\mathsf{Tr}(I-J_{t}(x_{t}))\vert+\|I-J_{t}(x_{t})\|_{\mathrm{F}}^{2}\big)\Big)\bigg)p_{X_{t}}(x_{t})\mathrm{d}x_{t}.
\end{align*}
Here step (i) follows from (\ref{eq:proof-lemma-2-6A-low-d}); step
(ii) utilizes (\ref{eq:proof-det-1-low-d}); step (iii) holds since
$u_{t}=x_{t}+(1-\alpha_{t})s_{t}^{\star}(x_{t})$, namely 
\[
\mathrm{d}u_{t}=\mathsf{det}\Big(I-\frac{1-\alpha_{t}}{1-\overline{\alpha}_{t}}J_{t}(x_{t})\Big)\mathrm{d}x_{t};
\]
while step (iv) follows from the facts that $1>\alpha_{t}$ and $e^{x}\geq1+x$
for any $x\in\mathbb{R}$. Recall that $\xi_{t}(x_{t})\leq0$ for
any $x_{t}\in\mathcal{E}_{t,1}$. By rearranging terms, we have
\begin{align*}
 & \int_{x_{t}\in\mathcal{E}_{t,1}}|\xi_{t}(x_{t})|p_{X_{t}}(x_{t})\mathrm{d}x_{t}\\
 & \qquad\leq\int_{x_{t}\in\mathcal{E}_{t,1}^{\mathrm{c}}}p_{X_{t}}(x_{t})\mathrm{d}x_{t}+C_{3}\Big(\frac{1-\alpha_{t}}{1-\overline{\alpha}_{t}}\Big)^{2}\int_{x_{t}\in\mathcal{E}_{t,1}}\big(\vert\mathsf{Tr}(I-J_{t}(x_{t}))\vert+\|I-J_{t}(x_{t})\|_{\mathrm{F}}^{2}\big)p_{X_{t}}(x_{t})\mathrm{d}x_{t}\\
 & \qquad\leq C_{3}\Big(\frac{1-\alpha_{t}}{1-\overline{\alpha}_{t}}\Big)^{2}\int_{x_{t}\in\mathcal{E}_{t,1}}\big(\vert\mathsf{Tr}(I-J_{t}(x_{t}))\vert+\|I-J_{t}(x_{t})\|_{\mathrm{F}}^{2}\big)p_{X_{t}}(x_{t})\mathrm{d}x_{t}+T^{-4}
\end{align*}
for some universal constant $C_{3}>0$, where the last step follows
from (\ref{eq:proof-lemma-2-6B-low-d}).

\subsection{Proof of Lemma~\ref{lemma:u-x-low-d}\protect\label{sec:proof-lemma-u-x-low-d}}

To begin with, we record the following two results from \cite{li2024adapting}.
For any $x_{t}\in\mathcal{E}_{t,1}$, we have \begin{subequations}

\begin{equation}
\int_{x_{0}}p_{X_{0}|X_{t}}(x_{0}\mymid x_{t})x_{0}\mathrm{d}x_{0}=\overline{x}_{0}+\delta\quad\text{where}\quad\overline{x}_{0}\in\bigcup_{i\in\mathcal{I}(x_{t};C_{1}\theta)}\mathcal{B}_{i}\label{eq:widehat-x0-decom}
\end{equation}
and 
\begin{equation}
\Vert\delta\Vert_{2}\leq\sqrt{\frac{1-\overline{\alpha}_{t}}{\overline{\alpha}_{t}}}\exp\left(-\frac{1}{32}C_{1}\theta k\log T\right).\label{eq:delta-bound}
\end{equation}
\end{subequations} In addition, for any $x,x'\in\mathcal{X}_{t}(x_{t})$,
we have
\begin{equation}
\overline{\alpha}_{t}\Vert x-x'\Vert_{2}^{2}\leq9C_{1}\theta k\left(1-\overline{\alpha}_{t}\right)\log T\label{eq:Xt-dist}
\end{equation}
and 
\begin{equation}
\left|\omega^{\top}\left(x-x'\right)\right|\leq\sqrt{\theta k\log T}\Vert x-x'\Vert_{2}+\big(4\sqrt{d}+4\sqrt{\theta k\log T}\big)\varepsilon\label{eq:omega-inner}
\end{equation}
See \cite[Equations (A.4), (A.5) and (A.27)]{li2024adapting} for
the proof.

Recall the definition of $\zeta_{t}(x_{t},x_{0})$ in (\ref{eq:zeta-defn-low-d}),
which can be written as
\[
\zeta_{t}(x_{t},x_{0})=\frac{(1-\alpha_{t})(1+\alpha_{t}-2\overline{\alpha}_{t})}{2(\alpha_{t}-\overline{\alpha}_{t})^{2}(1-\overline{\alpha}_{t})}\theta_{1}(x_{t},x_{0})+\frac{1-\alpha_{t}}{(\alpha_{t}-\overline{\alpha}_{t})^{2}}\theta_{2}(x_{t},x_{0}),
\]
where
\begin{align*}
\theta_{1}(x_{t},x_{0}) & =\|x_{t}-\sqrt{\overline{\alpha}_{t}}x_{0}\|_{2}^{2}-\int_{x_{0}}p_{X_{0}|X_{t}}(x_{0}\mymid x_{t})\|x_{t}-\sqrt{\overline{\alpha}_{t}}x_{0}\|_{2}^{2}\mathrm{d}x_{0},\\
\theta_{2}(x_{t},x_{0}) & =\sqrt{\overline{\alpha}_{t}}\big[\int_{x_{0}}p_{X_{0}|X_{t}}(x_{0}\mymid x_{t})\big(x_{t}-\sqrt{\overline{\alpha}_{t}}x_{0}\big)\mathrm{d}x_{0}\big]^{\top}\big(x_{0}-\int_{x_{0}}p_{X_{0}|X_{t}}(x_{0}\mymid x_{t})x_{0}\mathrm{d}x_{0}\big).
\end{align*}
For any $x_{t}\in\mathcal{E}_{t,1}$, recall the decomposition $x_{t}=\sqrt{\overline{\alpha}_{t}}x_{0}(x_{t})+\sqrt{1-\overline{\alpha}_{t}}\omega$
in (\ref{eq:xt-decom}), we have
\begin{align*}
\theta_{1}(x_{t},x_{0}) & =\|x_{t}-\sqrt{\overline{\alpha}_{t}}x_{0}(x_{t})+\sqrt{\overline{\alpha}_{t}}x_{0}(x_{t})-\sqrt{\overline{\alpha}_{t}}x_{0}\|_{2}^{2}\\
 & \qquad\qquad-\int_{x_{0}}p_{X_{0}|X_{t}}(x_{0}\mymid x_{t})\|x_{t}-\sqrt{\overline{\alpha}_{t}}x_{0}(x_{t})+\sqrt{\overline{\alpha}_{t}}x_{0}(x_{t})-\sqrt{\overline{\alpha}_{t}}x_{0}\|_{2}^{2}\mathrm{d}x_{0}\\
 & =\overline{\alpha}_{t}\Big(\|x_{0}-x_{0}(x_{t})\|_{2}^{2}-\int_{x_{0}}p_{X_{0}|X_{t}}(x_{0}\mymid x_{t})\|x_{0}-x_{0}(x_{t})\|_{2}^{2}\mathrm{d}x_{0}\Big)\\
 & \qquad\qquad-2\sqrt{\overline{\alpha}_{t}(1-\overline{\alpha}_{t})}\Big[\omega^{\top}\big(x_{0}-x_{0}(x_{t})\big)-\int_{x_{0}}p_{X_{0}|X_{t}}(x_{0}\mymid x_{t})\omega^{\top}\big(x_{0}-x_{0}(x_{t})\big)\mathrm{d}x_{0}\Big].
\end{align*}
In view of (\ref{eq:omega-inner}), we have
\[
\big|\omega^{\top}\big(x_{0}-x_{0}(x_{t})\big)\big|\leq\sqrt{\theta k\log T}\Vert x_{0}-x_{0}(x_{t})\Vert_{2}+4\varepsilon(\sqrt{d}+\sqrt{\theta k\log T}).
\]
We also learn from (\ref{eq:xi-bound}) and (\ref{eq:zeta-bound})
in the proof of Lemma~\ref{lem:Jt} that
\[
\frac{\overline{\alpha}_{t}}{1-\overline{\alpha}_{t}}\int_{x_{0}}p_{X_{0}|X_{t}}(x_{0}\mymid x_{t})\Vert x_{0}(x_{t})-x_{0}\Vert_{2}^{2}\mathrm{d}x_{0}\leq3C_{1}\theta k\log T
\]
and
\[
\sqrt{\frac{\overline{\alpha}_{t}}{1-\overline{\alpha}_{t}}}\int p_{X_{0}|X_{t}}(x_{0}\mymid x_{t})\big|\omega^{\top}\big(x_{0}(x_{t})-x_{0}\big)\big|\mathrm{d}x_{0}\leq2\sqrt{C_{1}}\theta k\log T.
\]
Taking the above bounds collectively yields 
\begin{align*}
-\theta_{1}(x_{t},x_{0}) & \leq7C_{1}(1-\overline{\alpha}_{t})\theta k\log T+2\sqrt{\overline{\alpha}_{t}(1-\overline{\alpha}_{t})}\sqrt{\theta k\log T}\Vert x_{0}-x_{0}(x_{t})\Vert_{2}
\end{align*}
provided that $\varepsilon>0$ is sufficiently small (see (\ref{eq:eps-condition}))
and $C_{1}>0$ is sufficiently large. Regarding $\theta_{2}(x_{t},x_{0})$,
we first use the decomposition (\ref{eq:widehat-x0-decom}) to achieve
\begin{align*}
\theta_{2}(x_{t},x_{0}) & =\sqrt{\overline{\alpha}_{t}}\big(x_{t}-\sqrt{\overline{\alpha}_{t}}\overline{x}_{0}-\sqrt{\overline{\alpha}_{t}}\delta\big)^{\top}\big(x_{0}-\overline{x}_{0}-\delta\big)\\
 & =\sqrt{\overline{\alpha}_{t}}\big(\sqrt{\overline{\alpha}_{t}}x_{0}(x_{t})+\sqrt{1-\overline{\alpha}_{t}}\omega-\sqrt{\overline{\alpha}_{t}}\overline{x}_{0}-\sqrt{\overline{\alpha}_{t}}\delta\big)^{\top}\big(x_{0}-x_{0}(x_{t})+x_{0}(x_{t})-\overline{x}_{0}-\delta\big)\\
 & =\overline{\alpha}_{t}\big(x_{0}(x_{t})-\overline{x}_{0}-\delta\big)^{\top}\big(x_{0}-x_{0}(x_{t})\big)+\sqrt{\overline{\alpha}_{t}(1-\overline{\alpha}_{t})}\omega^{\top}\big(x_{0}-\overline{x}_{0}-\delta\big)+\overline{\alpha}_{t}\Vert x_{0}(x_{t})-\overline{x}_{0}-\delta\Vert_{2}^{2}.
\end{align*}
Hence we have
\begin{align*}
-\theta_{2}(x_{t},x_{0}) & \overset{\text{(i)}}{\leq}\overline{\alpha}_{t}\Vert x_{0}(x_{t})-\overline{x}_{0}-\delta\Vert_{2}\Vert x_{0}-x_{0}(x_{t})\Vert_{2}+\sqrt{\overline{\alpha}_{t}(1-\overline{\alpha}_{t})}\big(\big|\omega^{\top}\big(x_{0}-\overline{x}_{0}\big)\big|+\Vert\omega\Vert_{2}\Vert\delta\Vert_{2}\big)\\
 & \overset{\text{(ii)}}{\leq}\overline{\alpha}_{t}(\Vert x_{0}(x_{t})-\overline{x}_{0}\Vert_{2}+\Vert\delta\Vert_{2})\Vert x_{0}-x_{0}(x_{t})\Vert_{2}\\
 & \qquad+\sqrt{\overline{\alpha}_{t}(1-\overline{\alpha}_{t})}\big(\sqrt{\theta k\log T}(\Vert x_{0}-x_{0}(x_{t})\Vert_{2}+\Vert x_{0}(x_{t})-\overline{x}_{0}\Vert_{2})+\Vert\omega\Vert_{2}\Vert\delta\Vert_{2}\big)\\
 & \overset{\text{(iii)}}{\leq}4\sqrt{C_{1}\overline{\alpha}_{t}\left(1-\overline{\alpha}_{t}\right)\theta k\log T}\Vert x_{0}-x_{0}(x_{t})\Vert_{2}+4\sqrt{C_{1}}(1-\overline{\alpha}_{t})\theta k\log T.
\end{align*}
Here step (i) utilizes the Cauchy-Schwarz inequality; step (ii) follows
from (\ref{eq:omega-inner}); step (iii) uses (\ref{eq:Xt-dist})
and (\ref{eq:delta-bound}), and holds provided that $C_{1}>0$ is
sufficiently large. Hence we have
\begin{align}
-\zeta_{t}(x_{t},x_{0}) & =-\frac{(1-\alpha_{t})(1+\alpha_{t}-2\overline{\alpha}_{t})}{2(\alpha_{t}-\overline{\alpha}_{t})^{2}(1-\overline{\alpha}_{t})}\theta_{1}(x_{t},x_{0})-\frac{1-\alpha_{t}}{(\alpha_{t}-\overline{\alpha}_{t})^{2}}\theta_{2}(x_{t},x_{0})\nonumber \\
 & \overset{\text{(a)}}{\leq}\frac{2(1-\alpha_{t})}{(1-\overline{\alpha}_{t})^{2}}\big(8C_{1}(1-\overline{\alpha}_{t})\theta k\log T+5\sqrt{C_{1}\overline{\alpha}_{t}\left(1-\overline{\alpha}_{t}\right)\theta k\log T}\Vert x_{0}-x_{0}(x_{t})\Vert_{2}\big)\nonumber \\
 & \overset{\text{(a)}}{\leq}66C_{1}\frac{1-\alpha_{t}}{1-\overline{\alpha}_{t}}\theta k\log T+\frac{1-\alpha_{t}}{2(1-\overline{\alpha}_{t})^{2}}\overline{\alpha}_{t}\Vert x_{0}-x_{0}(x_{t})\Vert_{2}^{2}.\label{eq:zeta-bound-low-d-1}
\end{align}
provided that $C_{1}>0$ is sufficiently large. Here step (a) follows
from consequences of Lemma~\ref{lemma:step-size}
\begin{align*}
\frac{(1-\alpha_{t})(1+\alpha_{t}-2\overline{\alpha}_{t})}{2(\alpha_{t}-\overline{\alpha}_{t})^{2}(1-\overline{\alpha}_{t})} & =\frac{1-\alpha_{t}}{(1-\overline{\alpha}_{t})^{2}}\Big(1+\frac{1-\alpha_{t}}{\alpha_{t}-\overline{\alpha}_{t}}\Big)\Big(1+\frac{1-\alpha_{t}}{2(\alpha_{t}-\overline{\alpha}_{t})}\Big)\leq\frac{2(1-\alpha_{t})}{(1-\overline{\alpha}_{t})^{2}}
\end{align*}
and
\[
\frac{1-\alpha_{t}}{(\alpha_{t}-\overline{\alpha}_{t})^{2}}=\frac{1-\alpha_{t}}{(1-\overline{\alpha}_{t})^{2}}\Big(1+\frac{1-\alpha_{t}}{\alpha_{t}-\overline{\alpha}_{t}}\Big)^{2}\leq\frac{2(1-\alpha_{t})}{(1-\overline{\alpha}_{t})^{2}}
\]
as long as $T$ is sufficiently large. Finally, notice that
\begin{align*}
\overline{\alpha}_{t}\Vert x_{0}-x_{0}(x_{t})\Vert_{2}^{2}-\Vert x_{t}-\sqrt{\overline{\alpha}_{t}}x_{0}\Vert_{2}^{2} & =\overline{\alpha}_{t}\Vert x_{0}-x_{0}(x_{t})\Vert_{2}^{2}-\Vert\sqrt{\overline{\alpha}_{t}}x_{0}(x_{t})+\sqrt{1-\overline{\alpha}_{t}}\omega-\sqrt{\overline{\alpha}_{t}}x_{0}\Vert_{2}^{2}\\
 & \leq-2\sqrt{\overline{\alpha}_{t}(1-\overline{\alpha}_{t})}\omega^{\top}(x_{0}(x_{t})-x_{0})\\
 & \overset{\text{(i)}}{\leq}2\sqrt{\overline{\alpha}_{t}(1-\overline{\alpha}_{t})\theta k\log T}\Vert x_{0}-x_{0}(x_{t})\Vert_{2}+1-\overline{\alpha}_{t}\\
 & \overset{\text{(ii)}}{\leq}\frac{1}{2}\overline{\alpha}_{t}\Vert x_{0}-x_{0}(x_{t})\Vert_{2}^{2}+3(1-\overline{\alpha}_{t})\theta k\log T
\end{align*}
where the last step follows from (\ref{eq:omega-inner}) and (\ref{eq:eps-condition}).
By rearranging terms we have
\begin{equation}
\overline{\alpha}_{t}\Vert x_{0}-x_{0}(x_{t})\Vert_{2}^{2}\leq2\Vert x_{t}-\sqrt{\overline{\alpha}_{t}}x_{0}\Vert_{2}^{2}+6(1-\overline{\alpha}_{t})\theta k\log T.\label{eq:zeta-bound-low-d-2}
\end{equation}
Taking (\ref{eq:zeta-bound-low-d-1}) and (\ref{eq:zeta-bound-low-d-2})
collectively yields
\begin{equation}
-\zeta_{t}(x_{t},x_{0})\leq69C_{1}\frac{1-\alpha_{t}}{1-\overline{\alpha}_{t}}\theta k\log T+\frac{1-\alpha_{t}}{(1-\overline{\alpha}_{t})^{2}}\Vert x_{t}-\sqrt{\overline{\alpha}_{t}}x_{0}\Vert_{2}^{2}\label{eq:zeta-bound-low-d}
\end{equation}
provided that $C_{2}\geq1$. Armed with this relation, we can follow
the same analysis in the proof of Lemma~\ref{lemma:u-x} to establish
the desired result under the condition $T\gg\theta k\log^{2}T$.

\subsection{Proof of Lemma~\ref{lemma:small-prob-low-d} \protect\label{subsec:proof-lemma-small-prob-low-d}}

\paragraph{Proof of \eqref{eq:proof-lemma-2-6B-low-d}.}

We have
\[
\int_{x_{t}\in\mathcal{E}_{t,1}^{\mathrm{c}}}p_{X_{t}}(x_{t})\mathrm{d}x_{t}=\mathbb{P}\left(X_{t}\notin\mathcal{E}_{t,1}\right)\leq\mathbb{P}\left(X_{0}\notin\cup_{i\in\mathcal{I}}\mathcal{B}_{i}\right)+\mathbb{P}\left(\overline{W}_{t}\notin\mathcal{G}\right),
\]
where we use the decomposition $X_{t}=\sqrt{\overline{\alpha}_{t}}X_{0}+\sqrt{1-\overline{\alpha}_{t}}\overline{W}_{t}$
for $\overline{W}_{t}\sim\mathcal{N}(0,I_{d})$. It is straightforward
to check that
\[
\mathbb{P}\left(X_{0}\notin\cup_{i\in\mathcal{I}}\mathcal{B}_{i}\right)\leq N_{\varepsilon}\exp\left(-\theta k\log T\right)\leq\exp\left(C_{\mathsf{cover}}k\log T-\theta k\log T\right)\leq\frac{1}{2}\exp\left(-\frac{\theta}{4}k\log T\right)
\]
provided that $\theta\gg C_{\mathsf{cover}}$. In addition, since
$\overline{W}_{t}\sim\mathcal{N}(0,I_{d})$, by the definition of
$\mathcal{G}$ we know that 
\begin{align*}
\mathbb{P}\left(\overline{W}_{t}\notin\mathcal{G}\right) & \leq\mathbb{P}\left(\Vert\overline{W}_{t}\Vert_{2}>\sqrt{d}+\sqrt{C_{1}k\log T}\right)+\sum_{i=1}^{N_{\varepsilon}}\sum_{j=1}^{N_{\varepsilon}}\mathbb{P}\left(\vert(x_{i}^{\star}-x_{j}^{\star})^{\top}\overline{W}_{t}\vert>\sqrt{\theta k\log T}\Vert x_{i}^{\star}-x_{j}^{\star}\Vert_{2}\right)\\
 & \overset{\text{(i)}}{\leq}\left(N_{\varepsilon}^{2}+1\right)\exp\left(-\frac{\theta}{2}k\log T\right)\leq\left(\exp\left(2C_{\mathsf{cover}}k\log T\right)+1\right)\exp\left(-\frac{\theta}{2}k\log T\right)\\
 & \overset{\text{(ii)}}{\leq}\frac{1}{2}\exp\left(-\frac{\theta}{4}k\log T\right).
\end{align*}
Here step (i) follows from concentration bounds for Gaussian and chi-square
variables (see Lemma~\ref{lemma:concentration}); while step (ii)
holds as long as $C_{1}\gg C_{\mathsf{cover}}$. Taking the above
bounds collectively yields
\begin{equation}
\int_{x_{t}\in\mathcal{E}_{t,1}^{\mathrm{c}}}p_{X_{t}}(x_{t})\mathrm{d}x_{t}\leq\exp\left(-\frac{\theta}{4}k\log T\right)\leq T^{-4}\label{eq:pXt-outside-typical}
\end{equation}
when $\theta>0$ is sufficiently large.

\paragraph{Proof of \eqref{eq:proof-lemma-2-6A-low-d}.}

For any $j\geq1$, define
\begin{align*}
\mathcal{I}_{j} & \coloneqq\left\{ 1\leq i\leq N_{\varepsilon}:\mathbb{P}(X_{0}\in\mathcal{B}_{i})\geq\exp(-2^{j-1}\theta k\log T)\right\} ,\\
\mathcal{G}_{j} & \coloneqq\big\{\omega\in\mathbb{R}^{d}:\Vert\omega\Vert_{2}\leq2\sqrt{d}+\sqrt{2^{j-1}\theta k\log T},\quad\text{and}\\
 & \qquad\qquad\qquad\vert(x_{i}^{\star}-x_{j}^{\star})^{\top}\omega\vert\leq\sqrt{2^{j-1}\theta k\log T}\Vert x_{i}^{\star}-x_{j}^{\star}\Vert_{2}\quad\text{for all}\quad1\leq i,j\leq N_{\varepsilon}\big\},
\end{align*}
and let
\[
\mathcal{L}_{t,j}:=\big\{\sqrt{\overline{\alpha}_{t}}x_{0}+\sqrt{1-\overline{\alpha}_{t}}\omega:x_{0}\in\cup_{i\in\mathcal{I}_{j}}\mathcal{B}_{i},\omega\in\mathcal{G}_{j}\big\}.
\]
We know that $\mathcal{L}_{t,1}\subseteq\mathcal{L}_{t,2}\subseteq\cdots$
and $\cup_{j=1}^{\infty}\mathcal{L}_{t,j}=\mathbb{R}^{d}$. Notice
that $\mathcal{E}_{t,1}=\mathcal{L}_{t,1}$. By defining $\mathcal{E}_{t,j}=\mathcal{L}_{t,j+1}\setminus\mathcal{L}_{t,j}$
for each $j\geq2$, we know that
\[
\bigcup_{j=1}^{\infty}\mathcal{E}_{t,j}=\mathbb{R}^{d}\qquad\text{where }\mathcal{E}_{t,1},\mathcal{E}_{t,2},\ldots\text{ are disjoint}.
\]
For any $x_{t}\in\mathcal{E}_{t,j}$, there exists an index $i(x_{t})\in\mathcal{I}_{j}$,
two points $x_{0}(x_{t})\in\mathcal{B}_{i(x_{t})}$ and $\omega\in\mathcal{G}_{j}$
such that $x_{t}=\sqrt{\overline{\alpha}_{t}}x_{0}(x_{t})+\sqrt{1-\overline{\alpha}_{t}}\omega$.
We learn from (\ref{eq:zeta-bound-low-d-1}) that, 
\begin{equation}
-\zeta_{t}(x_{t},x_{0})\leq66C_{1}\frac{1-\alpha_{t}}{1-\overline{\alpha}_{t}}2^{j}\theta k\log T+\frac{1-\alpha_{t}}{(1-\overline{\alpha}_{t})^{2}}\overline{\alpha}_{t}\Vert x_{0}-x_{0}(x_{t})\Vert_{2}^{2}.\label{eq:zeta-bound-j}
\end{equation}
This implies that for any $x_{t}\in\mathcal{E}_{t,j}$, we have
\begin{align}
 & -\frac{(1-\overline{\alpha}_{t})\|u_{t}-\sqrt{\overline{\alpha}_{t}}x_{0}\|^{2}}{2(\alpha_{t}-\overline{\alpha}_{t})^{2}}\label{eq:ut-decom-j}\\
 & \quad\overset{\text{(i)}}{=}-\frac{\|x_{t}-\sqrt{\overline{\alpha}_{t}}x_{0}\|_{2}^{2}}{2(1-\overline{\alpha}_{t})}-\frac{1-\alpha_{t}}{\alpha_{t}-\overline{\alpha}_{t}}\mathsf{Tr}\left(I-J_{t}(x_{t})\right)-\zeta_{t}(x_{t},x_{0})+O\Big(\Big(\frac{1-\alpha_{t}}{\alpha_{t}-\overline{\alpha}_{t}}\Big)^{2}\vert\mathsf{Tr}(I-J_{t}(x_{t}))\vert\Big).\nonumber \\
 & \quad\overset{\text{(ii)}}{\leq}-\frac{\|x_{t}-\sqrt{\overline{\alpha}_{t}}x_{0}\|_{2}^{2}}{2(1-\overline{\alpha}_{t})}+\frac{(1-\alpha_{t})\overline{\alpha}_{t}}{(1-\overline{\alpha}_{t})^{2}}\Vert x_{0}-x_{0}(x_{t})\Vert_{2}^{2}-\frac{1-\alpha_{t}}{\alpha_{t}-\overline{\alpha}_{t}}\mathsf{Tr}\left(I-J_{t}(x_{t})\right)+66C_{1}\frac{1-\alpha_{t}}{1-\overline{\alpha}_{t}}2^{j}\theta k\log T\nonumber \\
 & \quad\overset{\text{(iii)}}{\leq}-\frac{\|x_{t}-\sqrt{\overline{\alpha}_{t}}x_{0}\|_{2}^{2}}{2(1-\overline{\alpha}_{t})}+\frac{(1-\alpha_{t})\overline{\alpha}_{t}}{(1-\overline{\alpha}_{t})^{2}}\Vert x_{0}-x_{0}(x_{t})\Vert_{2}^{2}-\log\det\Big(I+\frac{1-\alpha_{t}}{\alpha_{t}-\overline{\alpha}_{t}}(I-J_{t}(x_{t}))\Big)+530c_{1}C_{1}\frac{\log^{2}T}{T}2^{j}\theta k\nonumber 
\end{align}
Here step (i) follows from (\ref{eq:zeta-decom-low-d}); step (ii)
follows from (\ref{eq:zeta-bound-j}) and Lemma~\ref{lem:Jt}, and
holds provided that $T$ is sufficiently large; step (iii) uses the
relation $\log(1+x)\leq x$ for any $x\geq0$ and $I-J_{t}(x_{t})\succeq0$.
Therefore for any $j\geq2$, we have
\begin{align}
I_{j} & \coloneqq\int_{x_{0}}\int_{x_{t}\in\mathcal{E}_{t,j}}p_{X_{0}}(x_{0})\Big(\frac{1-\overline{\alpha}_{t}}{2\pi(\alpha_{t}-\overline{\alpha}_{t})^{2}}\Big)^{d/2}\exp\Big(-\frac{(1-\overline{\alpha}_{t})\|u_{t}-\sqrt{\overline{\alpha}_{t}}x_{0}\|_{2}^{2}}{2(\alpha_{t}-\overline{\alpha}_{t})^{2}}\Big)\mathrm{d}x_{0}\mathrm{d}u_{t}\nonumber \\
 & \overset{\text{(a)}}{=}\int_{x_{0}}\int_{x_{t}\in\mathcal{E}_{t,j}}p_{X_{0}}(x_{0})\Big(\frac{1-\overline{\alpha}_{t}}{2\pi(\alpha_{t}-\overline{\alpha}_{t})^{2}}\Big)^{d/2}\exp\Big(-\frac{(1-\overline{\alpha}_{t})\|u_{t}-\sqrt{\overline{\alpha}_{t}}x_{0}\|_{2}^{2}}{2(\alpha_{t}-\overline{\alpha}_{t})^{2}}\Big)\det\Big(I-\frac{1-\alpha_{t}}{1-\overline{\alpha}_{t}}J_{t}(x_{t})\Big)\mathrm{d}x_{0}\mathrm{d}x_{t}\nonumber \\
 & \overset{\text{(b)}}{\leq}\int_{x_{t}\in\mathcal{E}_{t,j}}\int_{x_{0}}p_{X_{0}}(x_{0})\big(2\pi(1-\overline{\alpha}_{t})\big)^{-d/2}\exp\Big(-\frac{\|x_{t}-\sqrt{\overline{\alpha}_{t}}x_{0}\|_{2}^{2}}{2(1-\overline{\alpha}_{t})}+\frac{(1-\alpha_{t})\overline{\alpha}_{t}}{(1-\overline{\alpha}_{t})^{2}}\Vert x_{0}-x_{0}(x_{t})\Vert_{2}^{2}\Big)\mathrm{d}x_{0}\mathrm{d}x_{t}\nonumber \\
 & \qquad\cdot\exp\bigg(530c_{1}C_{1}\frac{\log^{2}T}{T}2^{j}\theta k\bigg).\label{eq:int-Ij-1}
\end{align}
Here step (a) follows from the relation $u_{t}=x_{t}+(1-\alpha_{t})s_{t}^{\star}(x_{t})$;
step (b) utilizes (\ref{eq:ut-decom-j}) and (\ref{eq:det-relation}).
Recall the defintion (\ref{eq:I-defn}) and let
\[
\mathcal{X}_{j}(x_{t})=\bigcup_{i\in\mathcal{I}(x_{t};C_{1}2^{j}\theta)}\mathcal{B}_{i}\qquad\text{and}\qquad\mathcal{Y}_{j}(x_{t})=\bigcup_{i\notin\mathcal{I}(x_{t};C_{1}2^{j}\theta)}\mathcal{B}_{i}.
\]
Then we have
\begin{align}
 & \int_{x_{0}\in\mathcal{X}_{j}(x_{t})}p_{X_{0}}(x_{0})\big(2\pi(1-\overline{\alpha}_{t})\big)^{-d/2}\exp\Big(-\frac{\|x_{t}-\sqrt{\overline{\alpha}_{t}}x_{0}\|_{2}^{2}}{2(1-\overline{\alpha}_{t})}+\frac{(1-\alpha_{t})\overline{\alpha}_{t}}{(1-\overline{\alpha}_{t})^{2}}\Vert x_{0}-x_{0}(x_{t})\Vert_{2}^{2}\Big)\mathrm{d}x_{0}\nonumber \\
 & \qquad\overset{\text{(i)}}{=}p_{X_{t}}(x_{t})\int_{x_{0}\in\mathcal{X}_{j}(x_{t})}p_{X_{0}|X_{t}}(x_{0}\mymid x_{t})\exp\Big(\frac{(1-\alpha_{t})\overline{\alpha}_{t}}{(1-\overline{\alpha}_{t})^{2}}\Vert x_{0}-x_{0}(x_{t})\Vert_{2}^{2}\Big)\mathrm{d}x_{0}\nonumber \\
 & \qquad\overset{\text{(ii)}}{\leq}p_{X_{t}}(x_{t})\int_{x_{0}\in\mathcal{X}_{j}(x_{t})}p_{X_{0}|X_{t}}(x_{0}\mymid x_{t})\exp\Big(4\frac{1-\alpha_{t}}{1-\overline{\alpha}_{t}}C_{1}2^{j}\theta k\log T\Big)\mathrm{d}x_{0}\nonumber \\
 & \qquad\overset{\text{(ii)}}{\leq}\exp\Big(32c_{1}C_{1}\frac{\log^{2}T}{T}2^{j}\theta k\Big)p_{X_{t}}(x_{t}).\label{eq:int-Ij-2}
\end{align}
Here step (i) uses the following relation
\begin{equation}
p_{X_{t}}(x_{t})=\int_{x_{0}}p_{X_{0}}(x_{0})\big(2\pi(1-\overline{\alpha}_{t})\big)^{-d/2}\exp\Big(-\frac{\|x_{t}-\sqrt{\overline{\alpha}_{t}}x_{0}\|_{2}^{2}}{2(1-\overline{\alpha}_{t})}\Big)\mathrm{d}x_{0};\label{eq:p_Xt}
\end{equation}
step (ii) follows from a direct consequence of $x_{0}\in\mathcal{X}_{j}(x_{t})$
and (\ref{eq:x0(xt)-x0}):
\[
\Vert x_{0}-x_{0}(x_{t})\Vert_{2}\leq\sqrt{\frac{1-\overline{\alpha}_{t}}{\overline{\alpha}_{t}}C_{1}2^{j}\theta k\log T}+2\varepsilon\leq2\sqrt{\frac{1-\overline{\alpha}_{t}}{\overline{\alpha}_{t}}C_{1}2^{j}\theta k\log T}.
\]
In addition, we also have
\begin{align}
 & \int_{x_{0}\in\mathcal{Y}_{j}(x_{t})}p_{X_{0}}(x_{0})\big(2\pi(1-\overline{\alpha}_{t})\big)^{-d/2}\exp\Big(-\frac{\|x_{t}-\sqrt{\overline{\alpha}_{t}}x_{0}\|_{2}^{2}}{2(1-\overline{\alpha}_{t})}+\frac{(1-\alpha_{t})\overline{\alpha}_{t}}{(1-\overline{\alpha}_{t})^{2}}\Vert x_{0}-x_{0}(x_{t})\Vert_{2}^{2}\Big)\mathrm{d}x_{0}\nonumber \\
 & \qquad\overset{\text{(a)}}{=}p_{X_{t}}(x_{t})\int_{x_{0}\in\mathcal{Y}_{j}(x_{t})}p_{X_{0}|X_{t}}(x_{0}\mymid x_{t})\exp\Big(\frac{(1-\alpha_{t})\overline{\alpha}_{t}}{(1-\overline{\alpha}_{t})^{2}}\Vert x_{0}-x_{0}(x_{t})\Vert_{2}^{2}\Big)\mathrm{d}x_{0}\nonumber \\
 & \qquad\leq p_{X_{t}}(x_{t})\sum_{i\notin\mathcal{I}(x_{t};C_{1}2^{j}\theta)}\mathbb{P}\left(X_{0}\in\mathcal{B}_{i}\mymid X_{t}=x_{t}\right)\exp\Big(\sup_{x_{0}\in\mathcal{B}_{i}}\frac{(1-\alpha_{t})\overline{\alpha}_{t}}{(1-\overline{\alpha}_{t})^{2}}\Vert x_{0}-x_{0}(x_{t})\Vert_{2}^{2}\Big)\nonumber \\
 & \qquad\overset{\text{(b)}}{\leq}p_{X_{t}}(x_{t})\sum_{i\notin\mathcal{I}(x_{t};C_{1}2^{j}\theta)}\exp\Big(-\frac{\overline{\alpha}_{t}}{32\left(1-\overline{\alpha}_{t}\right)}\Vert x_{i(x_{t})}^{\star}-x_{i}^{\star}\Vert_{2}^{2}\Big)\mathbb{P}\left(X_{0}\in\mathcal{B}_{i}\right)\nonumber \\
 & \qquad\overset{\text{(c)}}{\leq}\exp\Big(-\frac{1}{32}C_{1}2^{j}\theta k\log T\Big)p_{X_{t}}(x_{t}).\label{eq:int-Ij-3}
\end{align}
Here step (a) uses the (\ref{eq:p_Xt}); step (b) follows from Lemma~\ref{lem:cond-low-dim}
and the following relation:
\begin{align*}
\sup_{x_{0}\in\mathcal{B}_{i}}\frac{(1-\alpha_{t})\overline{\alpha}_{t}}{(1-\overline{\alpha}_{t})^{2}}\Vert x_{0}-x_{0}(x_{t})\Vert_{2}^{2} & \overset{\text{(i)}}{\leq}\frac{(1-\alpha_{t})\overline{\alpha}_{t}}{(1-\overline{\alpha}_{t})^{2}}(\Vert x_{i(x_{t})}^{\star}-x_{i}^{\star}\Vert_{2}+2\varepsilon)^{2}\overset{\text{(ii)}}{\leq}\frac{\overline{\alpha}_{t}}{32\left(1-\overline{\alpha}_{t}\right)}\Vert x_{i(x_{t})}^{\star}-x_{i}^{\star}\Vert_{2}^{2},
\end{align*}
where step (i) uses the relation (\ref{eq:x0(xt)-x0}) and step (ii)
follows from Lemma~\ref{lemma:step-size} and (\ref{eq:eps-condition}),
and holds provided that $T$ is sufficiently large; step (c) follows
from the definition of $\mathcal{I}(x_{t};C_{1}2^{j}\theta)$. Taking
(\ref{eq:int-Ij-1}), (\ref{eq:int-Ij-2}) and (\ref{eq:int-Ij-3})
collectively leads to 
\begin{align*}
I_{j} & =\exp\bigg(530c_{1}C_{1}\frac{\log^{2}T}{T}2^{j}\theta k\bigg)\int_{x_{t}\in\mathcal{E}_{t,j}}\bigg(\int_{x_{0}\in\mathcal{X}_{j}(x_{t})}+\int_{x_{0}\in\mathcal{Y}_{j}(x_{t})}\bigg)p_{X_{0}}(x_{0})\big(2\pi(1-\overline{\alpha}_{t})\big)^{-d/2}\\
 & \qquad\qquad\cdot\exp\Big(-\frac{\|x_{t}-\sqrt{\overline{\alpha}_{t}}x_{0}\|_{2}^{2}}{2(1-\overline{\alpha}_{t})}+\frac{(1-\alpha_{t})\overline{\alpha}_{t}}{(1-\overline{\alpha}_{t})^{2}}\Vert x_{0}-x_{0}(x_{t})\Vert_{2}^{2}\Big)\mathrm{d}x_{0}\mathrm{d}x_{t}\\
 & \leq\exp\Big(562c_{1}C_{1}\frac{\log^{2}T}{T}2^{j}\theta k\Big)\int_{x_{t}\in\mathcal{E}_{t,j}}p_{X_{t}}(x_{t})\mathrm{d}x_{t}\\
 & \overset{\text{(a)}}{\leq}\exp\left(562c_{1}C_{1}\frac{\log^{2}T}{T}2^{j}\theta k-\frac{1}{4}2^{j-1}\theta k\log T\right)\overset{\text{(b)}}{\leq}\exp\left(-\frac{1}{8}2^{j-1}\theta k\log T\right).
\end{align*}
Here step (a) follows from the relation (\ref{eq:pXt-outside-typical})
that we have already proved (by replacing $\theta$ with $2^{j-1}\theta$,
since $\mathcal{E}_{t,j}\in\mathcal{L}_{t,j}^{\mathrm{c}}$); step
(b) holds provided that $T\gg c_{1}C_{1}\log T$. Hence we have
\begin{align*}
 & \int_{x_{0}}\int_{x_{t}\notin\mathcal{E}_{t,1}}p_{X_{0}}(x_{0})\Big(\frac{1-\overline{\alpha}_{t}}{2\pi(\alpha_{t}-\overline{\alpha}_{t})^{2}}\Big)^{d/2}\exp\Big(-\frac{(1-\overline{\alpha}_{t})\|u_{t}-\sqrt{\overline{\alpha}_{t}}x_{0}\|_{2}^{2}}{2(\alpha_{t}-\overline{\alpha}_{t})^{2}}\Big)\mathrm{d}x_{0}\mathrm{d}u_{t}\\
 & \qquad\overset{\text{(i)}}{=}\sum_{j=2}^{\infty}\int_{x_{0}}\int_{x_{t}\in\mathcal{E}_{t,j}}p_{X_{0}}(x_{0})\Big(\frac{1-\overline{\alpha}_{t}}{2\pi(\alpha_{t}-\overline{\alpha}_{t})^{2}}\Big)^{d/2}\exp\Big(-\frac{(1-\overline{\alpha}_{t})\|u_{t}-\sqrt{\overline{\alpha}_{t}}x_{0}\|_{2}^{2}}{2(\alpha_{t}-\overline{\alpha}_{t})^{2}}\Big)\mathrm{d}x_{0}\mathrm{d}u_{t}\\
 & \qquad=\sum_{j=2}^{\infty}I_{j}\leq\sum_{j=2}^{\infty}\exp\left(-\frac{1}{8}2^{j-1}\theta k\log T\right)\leq T^{-4}
\end{align*}
provided that $\theta>0$ is sufficiently large.

%% file: appendix_auxiliary.tex
\section{Technical lemmas}

In this section, we gather a couple of useful technical lemmas.

\begin{lemma} \label{lemma:step-size}When $T$ is sufficiently large,
for $1\leq t\leq T$, we have

\[
\alpha_{t}\geq1-\frac{c_{1}\log T}{T}\geq\frac{1}{2}.
\]
For $2\leq t\leq T$, we have 
\begin{align*}
\frac{1-\alpha_{t}}{1-\overline{\alpha}_{t}} & \leq\frac{1-\alpha_{t}}{\alpha_{t}-\overline{\alpha}_{t}}\leq\frac{8c_{1}\log T}{T}.
\end{align*}
In addition, we have 
\[
\overline{\alpha}_{T}\leq T^{-c_{1}/2}.
\]

\end{lemma}\begin{proof} See \citet[Appendix A.2]{li2023towards}.\end{proof}

\begin{lemma}\label{lemma:concentration}For $Z\sim\mathcal{N}(0,1)$
and any $t\geq1$, we know that 
\[
\mathbb{P}\left(\left|Z\right|\geq t\right)\leq e^{-t^{2}/2},\qquad\forall\,t\geq1.
\]
In addition, for a chi-square random variable $Y\sim\chi^{2}(d)$,
we have 
\[
\mathbb{P}(\sqrt{Y}\geq\sqrt{d}+t)\leq e^{-t^{2}/2},\qquad\forall\,t\geq1.
\]
\end{lemma}\begin{proof}See \citet[Proposition 2.1.2]{vershynin2018high}
and \citet[Section 4.1]{laurent2000adaptive}.\end{proof}

\begin{lemma}\label{lemma:initialization-error}Suppose that Assumption~\ref{assumption:moment}
holds, and that $T$ and $c_{2}$ are sufficiently large. Then we
have 
\[
\mathsf{TV}\big(p_{X_{T}}\Vert p_{Y_{T}}\big)\leq T^{-99}.
\]
\end{lemma}

\begin{proof}

Define a random variable $X_{0}^{-}\coloneqq X_{0}\ind\{\Vert X_{0}\Vert_{2}\leq T^{c_{M}+100}\}$
by truncating $X_{0}$. Let 
\[
X_{T}^{-}=\sqrt{\overline{\alpha}_{T}}X_{0}^{-}+\sqrt{1-\overline{\alpha}_{T}}Z,
\]
where $Z\sim\mathcal{N}(0,I_{d})$ is independent of $X_{0}^{-}$.
Notice that $X_{0}^{-}$ has bounded support, which allows us to invoke
\cite[Lemma 3]{li2023towards} to achieve 
\begin{equation}
\mathsf{TV}(p_{\overline{X}_{T}},p_{Y_{T}})=O(T^{-100}),\label{eq:TV-T-1}
\end{equation}
provided that $c_{2}$ and $T$ are sufficiently large. In addition,
we have 
\begin{align}
\mathsf{TV}(p_{\overline{X}_{T}},p_{X_{T}}) & =\frac{1}{2}\int\vert p_{\overline{X}_{T}}(x)-p_{X_{T}}(x)\vert\mathrm{d}x\nonumber \\
 & =\frac{1}{2}\int_{x}\Big|\int_{x_{0}}\big(p_{\overline{X}_{0}}(x_{0})-p_{X_{0}}(x_{0})\big)\big(2\pi(1-\overline{\alpha}_{T})\big)^{-d/2}\exp\Big(-\frac{\Vert x-\sqrt{\overline{\alpha}_{T}}x_{0}\Vert_{2}^{2}}{2(1-\overline{\alpha}_{T})}\Big)\mathrm{d}x_{0}\Big|\mathrm{d}x\nonumber \\
 & \leq\frac{1}{2}\int_{x}\int_{x_{0}}\big|p_{\overline{X}_{0}}(x_{0})-p_{X_{0}}(x_{0})\big|\big(2\pi(1-\overline{\alpha}_{T})\big)^{-d/2}\exp\Big(-\frac{\Vert x-\sqrt{\overline{\alpha}_{T}}x_{0}\Vert_{2}^{2}}{2(1-\overline{\alpha}_{T})}\Big)\mathrm{d}x_{0}\mathrm{d}x\nonumber \\
 & \overset{\text{(i)}}{=}\frac{1}{2}\int_{x_{0}}\big|p_{\overline{X}_{0}}(x_{0})-p_{X_{0}}(x_{0})\big|\mathrm{d}x_{0}=\mathsf{TV}(p_{\overline{X}_{0}},p_{X_{0}})=\mathbb{P}\big(\Vert X_{0}\Vert_{2}>T^{c_{M}+100}\big)\nonumber \\
 & \overset{\text{(ii)}}{\leq}\frac{\mathbb{E}[\Vert X_{0}\Vert_{2}]}{T^{c_{M}+100}}=T^{-100}.\label{eq:TV-T-2}
\end{align}
Here step (i) invokes Tonelli's theorem, while step (ii) follows from
Markov's inequality. Taking (\ref{eq:TV-T-1}) and (\ref{eq:TV-T-2})
collectively yields the desired result, provided that $T$ is sufficiently
large.

\end{proof}

\begin{lemma} \label{lemma:jacob-sum}Suppose that Assumption~\ref{assumption:moment}
holds, and that $T\gg d\log T$. Then we have 
\begin{equation}
\sum_{t=2}^{T}\frac{1-\alpha_{t}}{1-\overline{\alpha}_{t}}\mathsf{Tr}\big(\mathbb{E}\big[\big(\Sigma_{\overline{\alpha}_{t}}(X_{t})\big)^{2}\big]\big)\leq C_{J}d\log T\label{eq:jacob-sum}
\end{equation}
for some universal constant $C_{J}>0$. Here the matrix function $\Sigma_{\overline{\alpha}_{t}}(\cdot)$
is defined as 
\begin{equation}
\Sigma_{\overline{\alpha}_{t}}(x)\coloneqq\mathsf{Cov}\big(Z\mymid\sqrt{\overline{\alpha}_{t}}X_{0}+\sqrt{1-\overline{\alpha}_{t}}Z=x\big),\label{eq:cov-defn}
\end{equation}
where $Z\sim\mathcal{N}(0,I_{d})$ is independent of $X_{0}$. \end{lemma}
\begin{proof} The result \eqref{eq:jacob-sum} was established in
\citet[Lemma 2]{li2024sharp} under the stronger assumption that 
\begin{equation}
\mathbb{P}(\|X_{0}\|_{2}<T^{c_{R}})=1\label{eq:bounded-support}
\end{equation}
for some universal constant $c_{R}>0$. The assumption \eqref{eq:bounded-support}
is used to prove part (a) of their Lemma~2, which states that for
any $\overline{\alpha}',\overline{\alpha}\in[\overline{\alpha}_{t},\overline{\alpha}_{t-1}]$
with $1\le t\le T$, one has 
\begin{align*}
\mathbb{E}\Big[\Big(\Sigma_{\overline{\alpha}'}\big(\sqrt{\overline{\alpha}'}X_{0}+\sqrt{1-\overline{\alpha}'}Z\big)\Big)^{2}\Big] & \preceq c_{1}'\mathbb{E}\Big[\Big(\Sigma_{\overline{\alpha}}\big(\sqrt{\overline{\alpha}}X_{0}+\sqrt{1-\overline{\alpha}}Z\big)\Big)^{2}\Big]+c_{1}'\exp(-c_{2}'d\log T)I_{d}.
\end{align*}
for some universal constants $c_{1}',c_{2}'>0$. Through a similar
truncation argument as in the proof of Lemma~\ref{lemma:initialization-error},
we can show that 
\begin{align*}
\mathbb{E}\Big[\Big(\Sigma_{\overline{\alpha}'}\big(\sqrt{\overline{\alpha}'}X_{0}+\sqrt{1-\overline{\alpha}'}Z\big)\Big)^{2}\Big] & \preceq c_{1}'\mathbb{E}\Big[\Big(\Sigma_{\overline{\alpha}}\big(\sqrt{\overline{\alpha}}X_{0}+\sqrt{1-\overline{\alpha}}Z\big)\Big)^{2}\Big]+c_{1}'T^{-100}I_{d}.
\end{align*}
Armed with this result, we can use the same analysis for proving part
(b) of \citet[Lemma 2]{li2024sharp} to establish \eqref{eq:jacob-sum}
under our Assumption~\ref{assumption:moment}. The details are omitted
here for simplicity.

\end{proof}

\begin{lemma} \label{lemma:jacob-sum-low-d}Let $k$ be the intrinsic
dimension (cf.~Definition~\ref{defn:intrinsic}) of the support
of $p_{\mathsf{data}}$, and suppose that $T\gg k\log T$. Then we
have 
\begin{equation}
\sum_{t=2}^{T}\frac{1-\alpha_{t}}{1-\overline{\alpha}_{t}}\mathsf{Tr}\big(\mathbb{E}\big[\big(\Sigma_{\overline{\alpha}_{t}}(X_{t})\big)^{2}\big]\big)\leq C_{J}k\log T\label{eq:jacob-sum-low-d}
\end{equation}
for some universal constant $C_{J}>0$. Here the matrix function $\Sigma_{\overline{\alpha}_{t}}(\cdot)$
is defined in (\ref{eq:cov-defn}). \end{lemma}

\begin{proof}

This lemma can be proved by modifying the first part of the proof
of \citet[Lemma 2]{li2024sharp}, and we describe these modification
as follows. For any $\overline{\alpha},\overline{\alpha}'\in(0,1)$,
define
\[
X_{\overline{\alpha}}=\sqrt{\overline{\alpha}}X_{0}+\sqrt{1-\overline{\alpha}}Z\qquad\text{and}\qquad X_{\overline{\alpha}'}=\sqrt{\overline{\alpha}'}X_{0}+\sqrt{1-\overline{\alpha}'}Z,
\]
where $X_{0}\sim p_{\mathsf{data}}$ and $Z\sim\mathcal{N}(0,I_{d})$
are independently random variables. \citet[Lemma~2, Part~(a)]{li2024sharp}
demonstrated that as long as $T\gg d\log T$, for any 
\[
\frac{|\overline{\alpha}'-\overline{\alpha}|}{\overline{\alpha}(1-\overline{\alpha})}=O\Big(\frac{1}{d\log T}\Big)
\]
 and any pair $(x,x')$ where $x$ is in a certain typical set (see
Equation (79) therein) and $x=\sqrt{\overline{\alpha}/\overline{\alpha}'}x'$,
it holds that $p_{X_{\overline{\alpha}'}}(x')\asymp p_{X_{\overline{\alpha}}}(x)$;
see Equation (81) therein. However here we only assume that $T\gg k\log T$,
and we want such a result to hold for any
\begin{equation}
\frac{|\overline{\alpha}'-\overline{\alpha}|}{\overline{\alpha}(1-\overline{\alpha})}=O\Big(\frac{1}{k\log T}\Big)\label{eq:alpha-bar-alpha-bar-prime-relation}
\end{equation}
in order to improve the dimension factor $d$ in \citet[Lemma~2, Part~(b)]{li2024sharp}
to the intrinsic dimension $k$. To this end, we instead consider
any pair $(x,x')$ where 
\begin{equation}
x=h(x')\coloneqq\sqrt{\overline{\alpha}/\overline{\alpha}'}x'+\big(\sqrt{\overline{\alpha}/\overline{\alpha}'}(1-\overline{\alpha}')-\sqrt{(1-\overline{\alpha})(1-\overline{\alpha}')}\big)s_{\overline{\alpha}'}^{\star}(x'),\label{eq:x-x'-relation}
\end{equation}
Here $s_{\overline{\alpha}'}^{\star}(\cdot)$ is the score function
of $X_{\overline{\alpha}'}$, namely
\[
s_{\overline{\alpha}'}^{\star}(x')=-\frac{1}{1-\overline{\alpha}_{t}}\int p_{X_{0}|X_{\overline{\alpha}'}}(x_{0}\mymid x')\big(x'-\sqrt{\overline{\alpha}_{t}}x_{0}\big)\mathrm{d}x_{0}.
\]
Let $\mathcal{E}_{\overline{\alpha},1}$ be the typical set of $X_{\overline{\alpha}}$
defined as replacing the $\overline{\alpha}_{t}$ in (\ref{eq:E-t-1-low-d})
with $\overline{\alpha}$. Following similar analysis as in Lemmas~\ref{lemma:det-low-d}~and~\ref{lemma:u-x-low-d},
we can show that
\[
p_{X_{\overline{\alpha}'}}(x')\mathrm{d}x'\asymp p_{X_{\overline{\alpha}}}(x)\mathrm{d}x,\quad\text{i.e.,}\quad p_{X_{\overline{\alpha}'}}(x')\asymp p_{X_{\overline{\alpha}}}(x)|\det J_{h}(x')|,
\]
holds for any $x\in\mathcal{E}_{\overline{\alpha}}$, where $J_{h}$
is the Jacobian matrix of $h$ (see (\ref{eq:x-x'-relation})). Equipped
with this relation, we can follow the steps in the proof of \citet[Lemma 2, Part (a)]{li2024sharp}
to show that
\[
p_{X_{0}|X_{\overline{\alpha}'}}(x_{0}\mymid x')\asymp p_{X_{0}|X_{\overline{\alpha}}}(x_{0}\mymid x),
\]
which corresponds to Equation (82) therein, and this further leads
to 
\begin{align*}
\mathbb{E}\Big[\Big(\Sigma_{\overline{\alpha}'}\big(\sqrt{\overline{\alpha}'}X_{0}+\sqrt{1-\overline{\alpha}'}Z\big)\Big)^{2}\Big] & \preceq C_{3}^{\prime2}\mathbb{E}\Big[\Big(\Sigma_{\overline{\alpha}}\big(\sqrt{\overline{\alpha}}X_{0}+\sqrt{1-\overline{\alpha}}Z\big)\Big)^{2}\Big]+C_{8}^{\prime}\exp(-C_{9}^{\prime}k\log T)I_{d}
\end{align*}
holds for all $\overline{\alpha},\overline{\alpha}'\in(0,1)$ satisfying
(\ref{eq:alpha-bar-alpha-bar-prime-relation}). Using the above result,
we can follow the same proof as in \citet[Lemma 2, Part (b)]{li2024sharp}
to establish the desired result. The detailed proof is omitted here
for brevity. \end{proof}